\documentclass[11pt]{article}

\usepackage[preprint]{acl}

\usepackage{times}
\usepackage{latexsym}

\usepackage[T1]{fontenc}

\usepackage[utf8]{inputenc}

\usepackage{microtype}

\usepackage{inconsolata}

\usepackage{graphicx}
\usepackage{amsmath}
\allowdisplaybreaks

\title{\textsc{UAVgent}: A Language-Guided Distributed Control Framework}

\author{Ziyi Zhang$^*$ \and Xiyu Deng$^*$ \and Guannan Qu \and Yorie Nakahira \\
  Department of Electrical and Computer Engineering\\
       Carnegie Mellon University\\
  5000 Forbes Avenue, Pittsburgh, PA 15213 \\
  \texttt{ziyizhan,xiyud,gqu,ynakahira@andrew.cmu.edu}}

\usepackage{etoolbox, amsmath,amsfonts,amsbsy,amsgen,amscd,mathrsfs,amssymb,amsthm,bm, mathtools}
\usepackage{titlesec}
\usepackage{dsfont}
\usepackage{algorithm}
\usepackage{algorithmic}
\usepackage[most]{tcolorbox}
\tcbuselibrary{breakable}
\usepackage{enumitem}
\usepackage{caption,wrapfig}
\usepackage{subcaption}
\hypersetup{citecolor=[RGB]{0,0,255}}
\usepackage{booktabs} 
\usepackage{cite}
\usepackage{footnotehyper}
\makesavenoteenv{table}
\makesavenoteenv{tabular}

\usepackage{natbib}
\usepackage{algorithm}
\usepackage{algorithmic}

\usepackage{cleveref}

\allowdisplaybreaks

\newcommand{\cmd}{\mathrm{cmd}}

\let\epsilon\varepsilon

\newtheorem{theorem}{Theorem}

\newtheorem{corollary}{Corollary}
\newtheorem{remark}{Remark}
\newtheorem{assumption}{Assumption}

\newboolean{showcomment}
\setboolean{showcomment}{true}
\newcommand{\todo}[1]{\ifthenelse{\boolean{showcomment}}{{\textcolor{red}{\bf Todo:  #1}}}{}}
\newcommand{\guannan}[1]{\ifthenelse{\boolean{showcomment}}{{\textcolor{red}{\bf [Guannan:  #1]}}}{}}

\usepackage{tikz}
\usetikzlibrary{matrix,decorations.pathreplacing,calc}

\usepackage{natbib}

\usepackage{multirow}
\usepackage{algorithm,algorithmic}

\pgfkeys{tikz/mymatrixenv/.style={decoration=brace,every left delimiter/.style={xshift=3pt},every right delimiter/.style={xshift=-3pt}}}
\pgfkeys{tikz/mymatrix/.style={matrix of math nodes,left delimiter=[,right delimiter={]},inner sep=2pt,column sep=1em,row sep=0.5em,nodes={inner sep=0pt}}}
\pgfkeys{tikz/mymatrixbrace/.style={decorate,thick}}

\begin{document}
\maketitle
\def\thefootnote{*}\footnotetext{These authors contributed equally to this work.}
\begin{abstract}
We study language-in-the-loop control for multi-drone systems that execute evolving, high-level missions while retaining formal robustness guarantees at the physical layer. We propose a three-layer architecture in which (i) a human operator issues natural-language instructions, (ii) an LLM-based supervisor periodically interprets, verifies, and corrects the commanded task in the context of the latest state and target estimates, and (iii) a distributed inner-loop controller tracks the resulting reference using only local relative information. We derive a theoretical guarantee that characterizes tracking performance under bounded disturbances and piecewise-smooth
references with discrete jumps induced by LLM updates. Overall, our results illustrate how centralized language-based task reasoning can be combined with distributed feedback control to achieve complex behaviors with provable robustness and stability.
\end{abstract}

\section{Introduction}
Multi‑robot systems are increasingly deployed for surveillance, inspection, environmental monitoring, and emergency response~\citep{Queralta20}. Yet specifying and adjusting multi‑agent behavior in dynamic environments can be both challenging and tedious, even for an expert user, as it requires not only relevant domain knowledge to determine the optimal action but also rapid adaptation to an evolving environment. The above problem magnifies when the number of drones increases, as designing control policies for each drone and having them collaborate in a time-varying environment requires designing and modifying many parameters that may be prohibitive for any single individual. 

The development of natural language interfaces offers an approach to lower this barrier by allowing non‑experts to express goals and directly map those objectives into low-level actions, such as movement and formation control~\citep{Liu24,FlockGPT24,venkatesh25}. However, this mapping remains brittle, especially in time-varying environments, because a one-time command might not account for future system changes; and multi‑agent coordination amplifies small specification errors into large collective failures~\citep{YoungScardoviLeonard2010,ShiJohansson2013}.

Large language models (LLMs) have recently improved the fidelity with which free‑form text can be grounded in structured plans and codes~\citep{pmlr-v162-huang22a,pmlr-v205-ichter23a}. However, most language‑to‑robot pipelines are compile‑once: a prompt is translated into a plan or policy and then executed until task completion or failure~\citep{Liang2022CodeAsPolicies}. This open‑loop usage is problematic for operation in a time-varying environment. In such settings, what matters is not only whether the initial plan is reasonable, but also whether the system can continuously incorporate new information to plan for the next action. A human-in-the-loop system offers a natural approach to solving this problem, as a human would be issuing new instructions amid changing objectives or environments. However, in a large multi-agent control system, it may be difficult for a human to pay attention to the entire system all the time. To solve this problem, we propose a mid-layer LLM supervisor as an "assistant" to the human user, which verifies that current execution matches the human’s intent and repair deviations without sacrificing the stability guarantees provided by traditional control. Therefore, the human user would only need to issue a command when a major change occurs, leaving minor adjustments and adaptation to environmental changes to the LLM supervisor.

\paragraph{Contribution} We introduce a hierarchical, human‑in‑the‑loop control framework in which an LLM-agent supervises execution, translating natural language goals to parameters for each agent, while low-level multi-agent control ensures stability and robustness. Our framework designs a centralized LLM-agent that supervises tasks proposed by a human operator and adapts execution to a changing environment, providing smooth, continuous execution without human intervention. By using an LLM supervisor, we ensure that the commands given to the drones are always aligned with the user's intention, and by using conventional multi-agent control, we ensure that the drones faithfully adapt in a dynamic environment with minimal human intervention. In addition, we offer a theoretical guarantee that ensures convergence of the control algorithm and provides an upper bound on the controller error between LLM verifications. 

\section{Related Works}
\label{sec:related_work}
\paragraph{Classical Multi-agent Control.} Multi-agent control has a long and well-established literature~\citep{Luo10,Olfati-Saber06}. A common paradigm is to employ a
centralized planner that generates intermediate goal representations or control inputs, which are then executed by the team~\citep{Hsieh06,Cheah09}.
While central planning can be effective for coordinating complex objectives, it typically depends on a command-and-update
structure and frequent communication to disseminate policies to individual agents. This reliance can become a bottleneck in
real-time settings where the environment and mission objectives evolve over time. On the other hand, some studies have explored decentralized planning in which each agent's policy relies solely on its local observations, thereby mitigating the need for real-time communication and providing robustness guarantees~\citep{Zhu19,Zhao16}. However, those studies often overlook the importance of centralized planning and struggle to achieve complex objectives. Compared to those works, we use an LLM for central planning and periodic auto-correction, but rely on distributed multi-agent control to achieve local robust guarantees. 
\paragraph{LLM-guided Swarm Control.}
Recent work has begun to use LLMs as high-level interfaces for specifying and adapting collective behaviors in robot swarms.
A first line of research treats language as a convenient task specification layer and learns a low-level multi-robot policy that realizes the requested pattern or motion.
For instance, \citet{Liu24} studies language-guided pattern formation using multi-agent reinforcement learning, where natural language prompts describe target spatial arrangements that are then achieved by decentralized robot policies.
Moreover, \citet{FlockGPT24} explores linguistic orchestration of UAV flocking, using an LLM to translate descriptive commands (e.g., formation type and motion style) into flock-level control objectives.
Going beyond a single centralized LLM, \citet{strobel24} propose \emph{LLM2Swarm}, in which LLM components support responsive reasoning, planning, and collaboration for swarms, highlighting the potential of multi-agent LLM architectures to mediate coordination.
Zero-shot approaches have also emerged: \citet{venkatesh25} demonstrates context-aware pattern formation from high-level user intent without task-specific fine-tuning, including setups where formations may be specified via richer modalities.
Despite rapid progress, most existing language-to-swarm pipelines remain largely ''compile-once'' at the high level: a language module produces a formation/plan, which is then executed with limited formal linkage between language-level edits and closed-loop stability.
This gap is particularly acute in dynamic multi-target scenarios, where ambiguous instructions, drifting execution, or transient disturbances can cause cascading coordination errors.
Our work is positioned to address this gap by coupling language-in-the-loop supervision (periodic intent verification and constrained goal edits) with a robust distributed controller that provides explicit stability/robustness guarantees between LLM updates.

\section{Algorithm Design}\label{sec:algorithm}
We propose the framework \textsc{UAVgent}, a three-layer hierarchy that couples
(i) natural-language commands,
(ii) LLM-based runtime verification/correction, and
(iii) a distributed inner-loop controller with robustness guarantees.
The layers operate on separated time scales: the user issues sporadic commands (outer layer),
the LLM performs $K$ periodic supervision at times $\{t_k\}_{k\in \{0,\dots,K-1\}}$ with $t_{k}-t_{k-1} =: \Delta_t$ for all $k \in \{1,\dots,K-1\}$ (middle layer),
and the multi-agent controller runs continuously (inner layer).

In the environment, a human user commands $N_a$ drones in $\mathbb R^d$ to achieve objectives that may involve
$N_b$ independent moving targets (e.g., ``Follow the red car'').

\paragraph{Notation.}
Let $p_a(t)\in\mathbb R^{dN_a}$ denote the stacked \emph{drone} positions and
$p_b(t)\in\mathbb R^{dN_b}$ denote the stacked \emph{target} positions.
Let
\[
p(t):=\begin{bmatrix}p_a(t)\\p_b(t)\end{bmatrix}\in\mathbb R^{dN},\qquad N:=N_a+N_b.
\]
We use $t_k^-$ and $t_k^+$ to denote the instants immediately before and after a supervision time $t_k$.

\paragraph{Radius-induced interaction graph.}
The neighbor graph is induced purely by geometry:
a drone considers another drone/target to be a neighbor if and only if it lies within the observation radius $r$.
Formally, define the (time-varying) edge set
\begin{equation*}
    \begin{split}
        \mathcal E(t):=&\Big\{\{i,j\}:\ i\in\mathcal V_a,\ j\in(\mathcal V_a\cup\mathcal V_b),
        \\
        &\ j\neq i,\ \|p_i(t)-p_j(t)\|\le r\Big\},
    \end{split}
\end{equation*}
i.e., edges are exactly those pairs within range of some drone sensor. Edges can disappear when nodes separate
and can reappear when nodes move back within range. The LLM never outputs edges and never enforces connectivity.
\begin{figure}
    \centering
    \includegraphics[width=\linewidth]{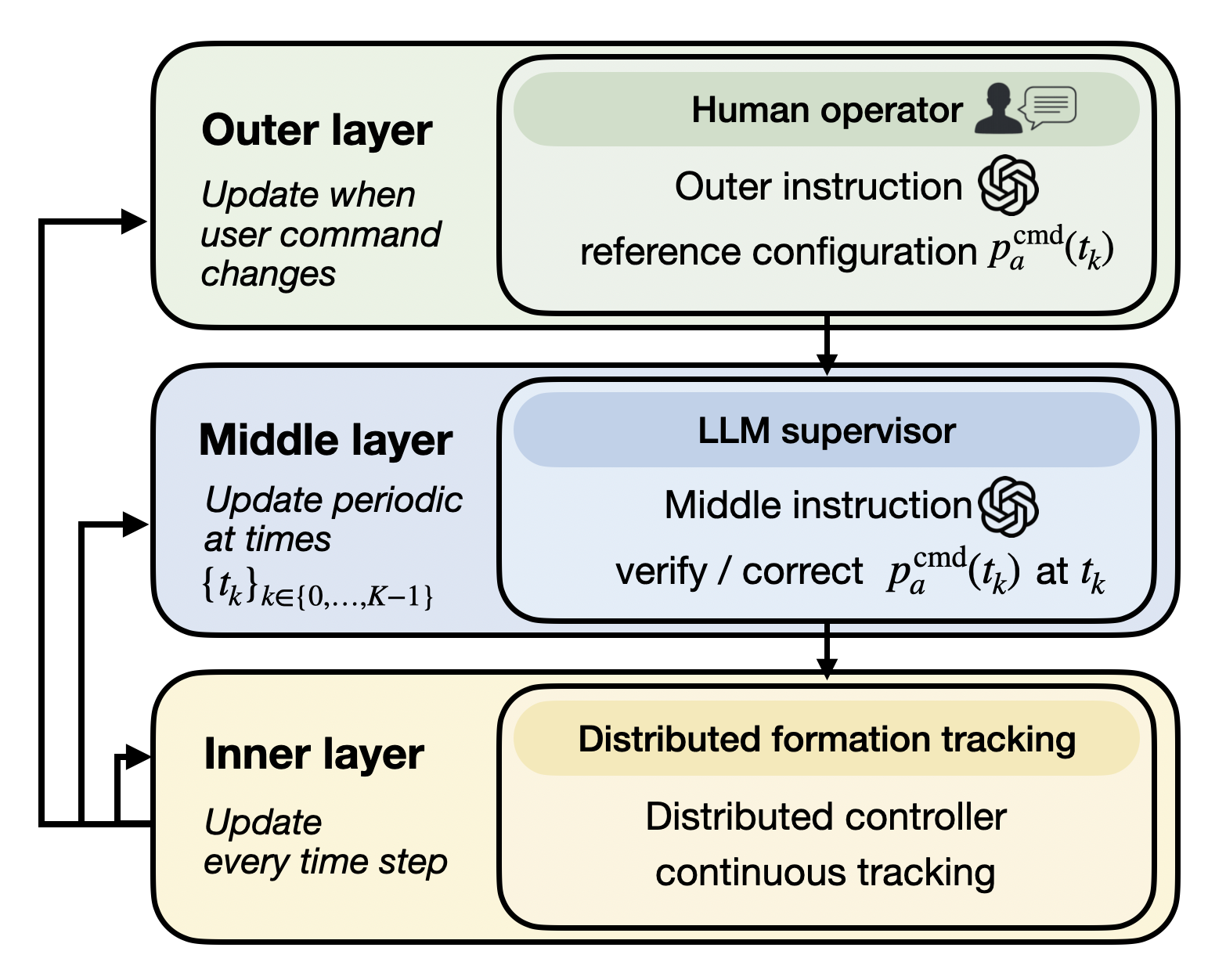}
    \caption{\textsc{UAVgent} hierarchical structure}
    \label{fig:flowchart}
\end{figure}

\subsection{Outer layer: user instruction via LLM}\label{sec:outer_layer_revised2}
At any supervision instants (a subset of $\{t_k\}$), the user provides a natural-language instruction
(e.g., formation type/scale, anchor choice, which target(s) to track).
The LLM translates the instruction into:
\begin{enumerate}
    \item a \emph{nominal} drone reference configuration
    \begin{equation}
        p_a^{\mathrm{cmd}}(t_k)\in\mathbb R^{dN_a}, \label{eq:p_cmd_revised2}
    \end{equation}
    \item an intent summary $\cmd$ stored in memory, including a \emph{reference type}
    \[
        \rho\in\{\textsf{stationary},\ \textsf{track},\ \textsf{search}\}.
    \]
\end{enumerate}
In \textsf{stationary} reference type, the intent is specified in the world frame (fixed anchors/coordinates).
In \textsf{track} reference type, the intent is specified relative to one or more targets.
In \textsf{search} reference type, the intent is at first specified in the world frame within a certain range defined by the user. When any drone finds the target, they are within the observation radius of each other, the search is completed, and all drones form a concentrated formation around the target. 

\subsection{Middle layer: LLM supervision}\label{sec:middle_layer_revised2}
At \emph{every} time $t_k$, the LLM supervisor receives (i) the latest intent $\cmd$ with reference type $\rho$,
and (ii) the state $p(t_k^-)$.
It performs auto-verification:
\begin{enumerate}
    \item If execution is consistent with $\cmd$, it produces no additional input.
    \item If inconsistencies are detected, it outputs an updated \emph{enforced} drone reference
    \begin{equation}
        p_a^r(t_k^+)=V\!\left(\cmd,p(t_k^-)\right), \label{eq:verify_map}
    \end{equation}
    where $V(\cdot)$ is the LLM verification/correction operator.
\end{enumerate}

\paragraph{Reference interpolation.}
On each interval $[t_k,t_{k+1})$, $p_a^r(t)$ is held constant or smoothly interpolated.
We assume $p_a^r$ is piecewise $C^1$ on each such interval and may jump at $t_k$.

\paragraph{Tracking re-grounding for moving targets.}
If $\rho=\textsf{track}$, the outer layer stores a target-referenced intent rather than a fixed world-frame coordinate. In target-referenced intent, if the intent can no longer be fulfilled due to target movement (e.g. tracking a group of target starting to diverge), the system re-grounds the target-referenced intent using the current target estimate
(e.g., via identity association and nearest-target heuristics with cooldown). 
At each check time $t_k$, the LLM verification \eqref{eq:verify_map}
further updates $p_a^r(t_k^+)$ if needed in accordance with the user intent. This prevents stale coordinates from being reissued when targets move between checks.

\paragraph{Range-limited target search.}
When $\rho=\textsf{search}$, the outer layer maintains a description of the sought target together with a bounded region in
which the target is expected to lie within. The LLM supervisor then assigns each drone a search waypoint within this region. After a
drone arrives at its assigned waypoint, it performs a local scan within its sensing (observation) radius; if no detection is
reported, the supervisor reassigns a new waypoint and the process repeats. If a detection occurs, the supervisor broadcasts
the estimated target location, directs all drones to rendezvous at the target, and transitions the swarm to the
user-specified formation.

\subsection{Inner layer: distributed formation tracking on a radius-induced graph}\label{sec:inner_layer_revised2}
Drones obey disturbed single-integrator dynamics
\begin{equation}
\dot p_a(t)=u_a(t)+d_a(t), \label{eq:dynamics_drones_revised2}
\end{equation}
while targets evolve exogenously
\begin{equation}
\dot p_b(t)=v_b(t), \label{eq:dynamics_targets_revised2}
\end{equation}
where $d_a(t)$ is an additive disturbance and $v_b(t)$ is an unknown target velocity.

At time $t$, each drone $i$ uses the instantaneous neighbor set
\[
\mathcal N_i(t):=\{j:\ \{i,j\}\in\mathcal E(t)\}.
\]
The inner loop applies the distributed edge-error feedback
\begin{equation}
u_i(t)=\sum_{j\in\mathcal N_i(t)}\Big((p_j(t)-p_i(t))-(p_j^r(t)-p_i^r(t))\Big), \label{eq:controller_revised2}
\end{equation}
where $p_i^r(t)$ is the enforced drone reference for drone $i$, and for any target $j\in\mathcal V_b$ we set
$p_j^r(t):=p_j(t)$ as targets are not commanded. Thus, target neighbors act as moving anchors, while drone neighbors enforce
relative formation constraints.

\section{Preliminaries}\label{sec:prelim}
This section fixes graph/coordinate notation and derives the closed-loop edge-error dynamics used in
\Cref{sec:guarantee}. The key modeling changes are:
(i) the neighbor graph is induced by the observation radius $r$, and
(ii) targets are not actuated.

\subsection{Radius-induced graph and incidence matrices}
\label{sec:radius_graph}
Let $\mathcal V=\mathcal V_a\cup\mathcal V_b$ be the node set (drones and targets), $|\mathcal V_a|=N_a$, $|\mathcal V_b|=N_b$.
At time $t$, define the undirected neighbor graph
\begin{align*}
    G(t):=&\big(\mathcal V,\mathcal E(t)\big),
    \\
    \mathcal E(t):=&\Big\{\{i,j\}:\ i\in\mathcal V_a,\ j\in(\mathcal V_a\cup\mathcal V_b),\ j\neq i,
    \\
    & \|p_i(t)-p_j(t)\|\le r\Big\}.
\end{align*}
Choose an arbitrary orientation of the edges at each time and let $E(t)\in\mathbb R^{N\times m(t)}$ denote the oriented incidence matrix,
with $m(t):=|\mathcal E(t)|$.
For analytical purposes, between two supervision instances, we assume the edge set is constant on $[t_k,t_{k+1})$ and write
\[
E_k:=E(t)\ \ \text{for}\ \ t\in[t_k,t_{k+1}),\qquad m_k:=m(t).
\]
If the induced graph is disconnected, the analysis below applies on each connected component.

Define the node Laplacian on $[t_k,t_{k+1})$ as
\begin{equation}
L_{n,k}:=E_kE_k^\top. \label{eq:Ln_k_revised2}
\end{equation}
Let the drone-selection matrix $R_a := \begin{bmatrix}
    I_{N_a} & 0
\end{bmatrix} \in \mathbb{R}^{N_a \times N}$ and
\[
S_a:=\mathrm{diag}(\underbrace{1,\dots,1}_{N_a},\underbrace{0,\dots,0}_{N_b}) = R_a^\top R_a,
\]
and define the \emph{drone-row} incidence matrix
\[
E_{a,k}:=R_aE_k\in\mathbb R^{N_a\times m_k}.
\]

Similarly, let the target-selection matrix $R_b := \begin{bmatrix}
    0 & I_{N_b}
\end{bmatrix} \in \mathbb{R}^{N_b \times N}$ and define the \emph{target-row} incidence matrix
\[
E_{b,k}:=R_bE_k\in\mathbb R^{N_b\times m_k}.
\]

\subsection{Edge coordinates and error}
Stack positions as $p(t)=[p_1(t)^\top\ \cdots\ p_N(t)^\top]^\top\in\mathbb R^{dN}$.
Define the stacked edge-relative vector on $[t_k,t_{k+1})$ as
\begin{equation}
z(t):=(E_k^\top\otimes I_d)\,p(t)\in\mathbb R^{dm_k}. \label{eq:z_def_revised2}
\end{equation}

The enforced reference is defined for drones only; targets are not commanded. Accordingly, define the node reference
\[
p^r(t):=\begin{bmatrix}p_a^r(t)\\p_b(t)\end{bmatrix},
\]
and the induced edge-relative reference
\begin{equation}
z^r(t):=(E_k^\top\otimes I_d)\,p^r(t). \label{eq:z_ref_revised2}
\end{equation}
Define the edge-formation error
\begin{equation}
e(t):=z^r(t)-z(t). \label{eq:edge_error_revised2}
\end{equation}
Because $p_b^r(t)=p_b(t)$, the target components cancel and $e(t)$ depends only on drone position error:
\begin{equation}
\label{eq:e_range}
e(t)=(E_{a,k}^\top\otimes I_d)\,\big(p_a^r(t)-p_a(t)\big)\in \mathrm{range}(E_{a,k}^\top\otimes I_d). 
\end{equation}

\subsection{Closed-loop edge-error dynamics}\label{sec:error_dynamics_revised2}
Let drone and target dynamics be \eqref{eq:dynamics_drones_revised2}--\eqref{eq:dynamics_targets_revised2} and define the stacked exogenous input
\[
d(t):=\begin{bmatrix}d_a(t)\\v_b(t)\end{bmatrix}.
\]
The distributed control law \eqref{eq:controller_revised2} can be written in stacked form on $[t_k,t_{k+1})$ as
\begin{equation}
u_a(t)=(E_{a,k}\otimes I_d)\,e(t), \label{eq:stacked_controller_revised2}
\end{equation}
i.e., the node control is computed from incident edge errors and applied only to drones.

On any interval where $p_a^r$ is $C^1$, differentiating \eqref{eq:edge_error_revised2} and using the dynamics yields
\begin{equation}
\dot e(t)=-(L_{e,k}^a\otimes I_d)\,e(t)+w(t), \label{eq:error_dyn}
\end{equation}
where the \emph{actuated edge Laplacian} is
\begin{equation}
L_{e,k}^a:=E_{a,k}^\top E_{a,k}=E_k^\top S_a E_k \succeq 0, \label{eq:Lea_revised2}
\end{equation}
and the additive input is
\begin{equation}
\label{eq:w_def}
w(t):=\dot z^r(t)-(E_k^\top\otimes I_d)\,d(t). 
\end{equation}
Using $p^r=[p_a^r;p_b]$ and $d=[d_a;v_b]$, the $v_b$ terms cancel inside \eqref{eq:w_def}, giving the equivalent expression
\begin{equation}
w(t)=(E_{a,k}^\top\otimes I_d)\,\big(\dot p_a^r(t)-d_a(t)\big)\in \mathrm{range}(E_{a,k}^\top\otimes I_d). \label{eq:w_simplified}
\end{equation}

\paragraph{Jumps at supervision times.}
At $t_k$, $p(t)$ is continuous but $p_a^r$ may jump. If the neighbor graph does not change at $t_k$,
then $e(t_k^+)=e(t_k^-)+\Delta z_k^r$ with $\Delta z_k^r:=z^r(t_k^+)-z^r(t_k^-)$.
If the neighbor graph changes due to the radius rule, then $E_k$ changes and $e(t_k^+)$ is reinitialized.

\section{Theoretical Guarantee}
\label{sec:guarantee}
In this section, we prove the robustness within an arbitrary interval $[t_k,t_{k+1})$. We first introduce a few assumptions for the derivation of a theoretical guarantee. We point out that those assumptions are for the convenience of analysis and are not strictly required by the algorithm, as shown in \Cref{sec:simulation}.

\begin{assumption}
    \label{assumption:general}
    On some fixed $[t_k, t_{k+1})$,
    \begin{enumerate}
        \item[(A1)] $E_k$ is constant, and the enforced edge-relative reference is held constant, i.e. $\dot z^r(t)=0.$
        \item[(A2)] the target motion is \emph{formation-feasible} in the sense that \eqref{eq:compatibility}
        admits at least one solution $\dot p_a^r(\cdot)$ on $[t_k,t_{k+1})$, and that the reference generator
        selects a feasible solution with bounded speed
        $\|\dot p_a^r(t)\|\le \bar v_r$
        for some $\bar v_r<\infty$.
    
        \item[(A3)] The drone disturbance and target velocity are bounded as
        $\|d_a(t)\|\le \bar d_a,\  \|v_b(t)\|\le \bar v_b.$
    \end{enumerate}
\end{assumption}
We note that since $z^r(t)=(E_k^\top\otimes I_d)p^r(t)$ with $p^r(t)=[p_a^r(t);p_b(t)]$, \Cref{assumption:general} (A1) is equivalent to the
        \emph{compatibility condition}
        \begin{equation}
        \label{eq:compatibility}
        \begin{split}
            0=&\dot z^r(t)=(E_k^\top\otimes I_d)
            \begin{bmatrix}\dot p_a^r(t)\\ v_b(t)\end{bmatrix}
            \\
            =&(E_{a,k}^\top\otimes I_d)\dot p_a^r(t)+(E_{b,k}^\top\otimes I_d)v_b(t).
        \end{split}
        \end{equation}
Moreover, given \Cref{assumption:general}, we have that the stacked exogenous input $d(t)=[d_a(t);v_b(t)]$ satisfies
        \[
        \|d(t)\|\le \bar d:=\bar d_a+\bar v_b,
        \]
        and the node-level effective input satisfies
        \[
        \|\dot p_a^r(t)-d_a(t)\|\le \bar \nu := \bar v_r+\bar d_a.
        \]
In most time intervals, the above assumption is natural: when drones maintain their formations, their relative positions should remain relatively constant, so their neighbor relations do not change either. The above assumption would be broken if the movement of targets renders the formation impossible to maintain. As we show in \Cref{sec:sim_police_chase}, the proposed algorithm still handles the scenario with ease. Moreover, as the disturbances in real-world scenarios are usually caused by wind or current, which can be assumed to be upper bounded. For both stationary and tracking tasks, when formation is required, the reference edge-relative positions of the drones should stay constant. 

\begin{theorem}[Exponential ISS of edge-formation error]\label{thm:ISS}
Fix $k$ and consider an interval $[t_k,t_{k+1})$ where \Cref{assumption:general} holds.
Further assume the corresponding graph is connected (or restrict to a connected component).
Let $e(t)$ satisfy the closed-loop edge-error dynamics \eqref{eq:error_dyn} with additive input $w(t)$ given by
\eqref{eq:w_def}.

Let $\lambda_k>0$ denote its smallest strictly positive eigenvalue of the edge Laplacian:
\begin{equation}
\lambda_k \;:=\;\lambda_{\min}^+\!\left(L_{e,k}^a\right)
          \;=\;\lambda_{\min}^+\!\left(E_{a,k}E_{a,k}^\top\right).
\label{eq:lambda_k_def}
\end{equation}
Then, for any $t\in[t_k,t_{k+1})$,
\begin{equation}
\|e(t)\|\le e^{-\lambda_k(t-t_k)}\|e(t_k^+)\|+\int_{t_k}^{t}e^{-\lambda_k(t-s)}\|w(s)\|\,ds. \label{eq:ISS_bound}
\end{equation}
In particular, if $\|w(t)\|\le \bar w$ on $[t_k,t_{k+1})$ then
\begin{equation}
\|e(t)\|\;\le\;e^{-\lambda_k(t-t_k)}\|e(t_k^+)\|
+\frac{\bar w}{\lambda_k}\Big(1-e^{-\lambda_k(t-t_k)}\Big). \label{eq:ISS_uniform}
\end{equation}
\end{theorem}

The proof of \Cref{thm:ISS} is deferred to \Cref{appendix:proof_main}. We point out that at an LLM supervision instant $t_k$, if the neighbor graph does not change (so $E_k$ is unchanged across $t_k$) and
$z^r$ jumps by $\Delta z^r_k:=z^r(t_k^+)-z^r(t_k^-)$, then
\begin{equation}
\label{eq:jump}
\begin{split}
    &e(t_k^+)=e(t_k^-)+\Delta z^r_k,\\
    &\|e(t_k^+)\|\le \|e(t_k^-)\|+\|\Delta z^r_k\|.
\end{split}
\end{equation}
If instead the radius-induced graph changes at $t_k$, then $e(t_k^+)$ is reinitialized in the new edge coordinates as stated in
\Cref{sec:error_dynamics_revised2} and serves as the initial condition for \eqref{eq:ISS_bound} on the next interval.

\subsection{Checking-period guideline}\label{sec:checking_period_guideline}

Define the drone reference-tracking error
\[
\tilde p_a(t):=p_a^r(t)-p_a(t)\in\mathbb R^{dN_a},\qquad t\in[t_k,t_{k+1}).
\]
Using the drone dynamics \eqref{eq:dynamics_drones_revised2}, the stacked controller \eqref{eq:stacked_controller_revised2},
and the identity $e(t)=(E_{a,k}^\top\otimes I_d)\tilde p_a(t)$ from \eqref{eq:e_range}, we obtain
\begin{equation}
\label{eq:node_error_dynamics}
    \begin{split}
        \dot{\tilde p}_a(t)
        &=\dot p_a^r(t)-\dot p_a(t)
        = \dot p_a^r(t)-u_a(t)-d_a(t) \nonumber\\
        &= -\big(E_{a,k}\otimes I_d\big)e(t) + \big(\dot p_a^r(t)-d_a(t)\big)
        \\
        &= -\big(E_{a,k}E_{a,k}^\top\otimes I_d\big)\tilde p_a(t) + \nu(t),
    \end{split}
\end{equation}
where $\nu(t):=\dot p_a^r(t)-d_a(t).$

If the connected component under consideration contains at least one target (equivalently, the restriction is \emph{anchored} so that
$E_{a,k}E_{a,k}^\top\succ 0$ on that component), then
the semigroup $e^{-(E_{a,k}E_{a,k}^\top\otimes I_d)t}$ contracts at rate at least $\lambda_k$ on $\mathbb R^{dN_a}$.
Applying variation-of-constants to \eqref{eq:node_error_dynamics} yields, for any $t\in[t_k,t_{k+1})$,
\begin{equation}
\begin{split}
    \|\tilde p_a(t)\|
    \le& e^{-\lambda_k(t-t_k)}\|\tilde p_a(t_k^+)\|
    +\int_{t_k}^{t}e^{-\lambda_k(t-s)}\|\nu(s)\|ds.
\end{split}
\label{eq:node_ISS_bound}
\end{equation}

In particular, if $\|\nu(t)\|\le \bar \nu$ on $[t_k,t_{k+1})$, then

\begin{equation}
\|\tilde p_a(t)\|
\le e^{-\lambda_k(t-t_k)}\|\tilde p_a(t_k^+)\|
+\frac{\bar \nu}{\lambda_k}\Big(1-e^{-\lambda_k(t-t_k)}\Big),
\label{eq:node_ISS_uniform}
\end{equation}

\begin{corollary}[LLM relative accuracy for edge-level tolerance]
\label{cor:LLM_accuracy}
Consider interval $[t_k,t_{k+1})$ satisfying \Cref{assumption:general}. Denote
$\bar w := \|E_k^\top\otimes I_d\|\,\bar d$ so that $\|w(t)\|\le \bar w$.
Further assume that the LLM grounding error in edge coordinates is uniformly bounded by $\varepsilon_z\ge 0$:
\[
\sup_{t\in[t_k,t_{k+1})}\|z^r(t)-z^\star(t)\|\le \varepsilon_z.
\]
Fix a desired edge tolerance $\delta_z>0$ at time $t_{k+1}^-$:
$\|z(t_{k+1}^-)-z^\star(t_{k+1}^-)\|\le \delta_z$.
If
\[
\eta^{(z)}_{\max}(\delta_z,\Delta_t,\varepsilon_z)
:= e^{\lambda_k\Delta_t}(\delta_z-\varepsilon_z) - \frac{\bar w}{\lambda_k}\bigl(e^{\lambda_k\Delta_t}-1\bigr)\ge 0,
\]
and the post-check edge error satisfies $\|e(t_k^+)\|\le \eta^{(z)}_{\max}(\delta_z,\Delta_t,\varepsilon_z)$,
then $\|z(t_{k+1}^-)-z^\star(t_{k+1}^-)\|\le \delta_z$.
\end{corollary}

\begin{remark}[Feasibility and interpretation]\label{rem:node_feasibility}
The threshold $\eta_k^{\max}(\delta_z,\Delta_t,\varepsilon_z)$ can be negative if $\delta$ is too small relative to
(i) the LLM grounding error $\varepsilon_z$, (ii) the effective input level $\bar w$, and (iii) the check interval $\Delta_t$.
In particular, even if the controller could instantaneously enforce $\tilde p_a(t_k^+)=0$, the bound
\eqref{eq:node_ISS_uniform} implies the unavoidable contribution
$\frac{\bar w}{\lambda_k}(1-e^{-\lambda_k\Delta_t})$, so a necessary condition for feasibility is
\[
\varepsilon_z+\frac{\bar w}{\lambda_k}\big(1-e^{-\lambda_k\Delta_t}\big)\le \delta_z.
\]
\end{remark}

The above results extend to a longer time horizon as long as the user intent stays consistent. We defer the detailed theorem and discussions to \Cref{sec:horizon_bounds}.

\section{Simulation}
\label{sec:simulation}
In this section, we evaluate the proposed framework in two representative case studies. The results demonstrate how high-level natural-language commands are translated into coordinated multi-drone behaviors, and how LLM-based supervision enables robust adaptation to dynamic environments, as detailed below. 
\subsection{Police--chasing scenario}
\label{sec:sim_police_chase}
We simulate 3 police-assisted pursuit scenarios, each with a different map, user command, number of drones, or configuration of suspect cars. The scenarios are represented in a three-dimensional environment ($d=3$), where a human operator issues
high-level natural-language commands and the proposed three-layer architecture coordinates a team of drones to track three
suspect vehicles (targets). Drone positions satisfy $p_i(t)\in\mathbb R^3$ and are initialized randomly, as in \Cref{fig:case1a}; the suspect cars evolve on the ground
plane, while drones maintain a commanded operating altitude band. 

\begin{figure*}[!ht]
\begin{subfigure}{.32\linewidth}
    \includegraphics[width=\linewidth]{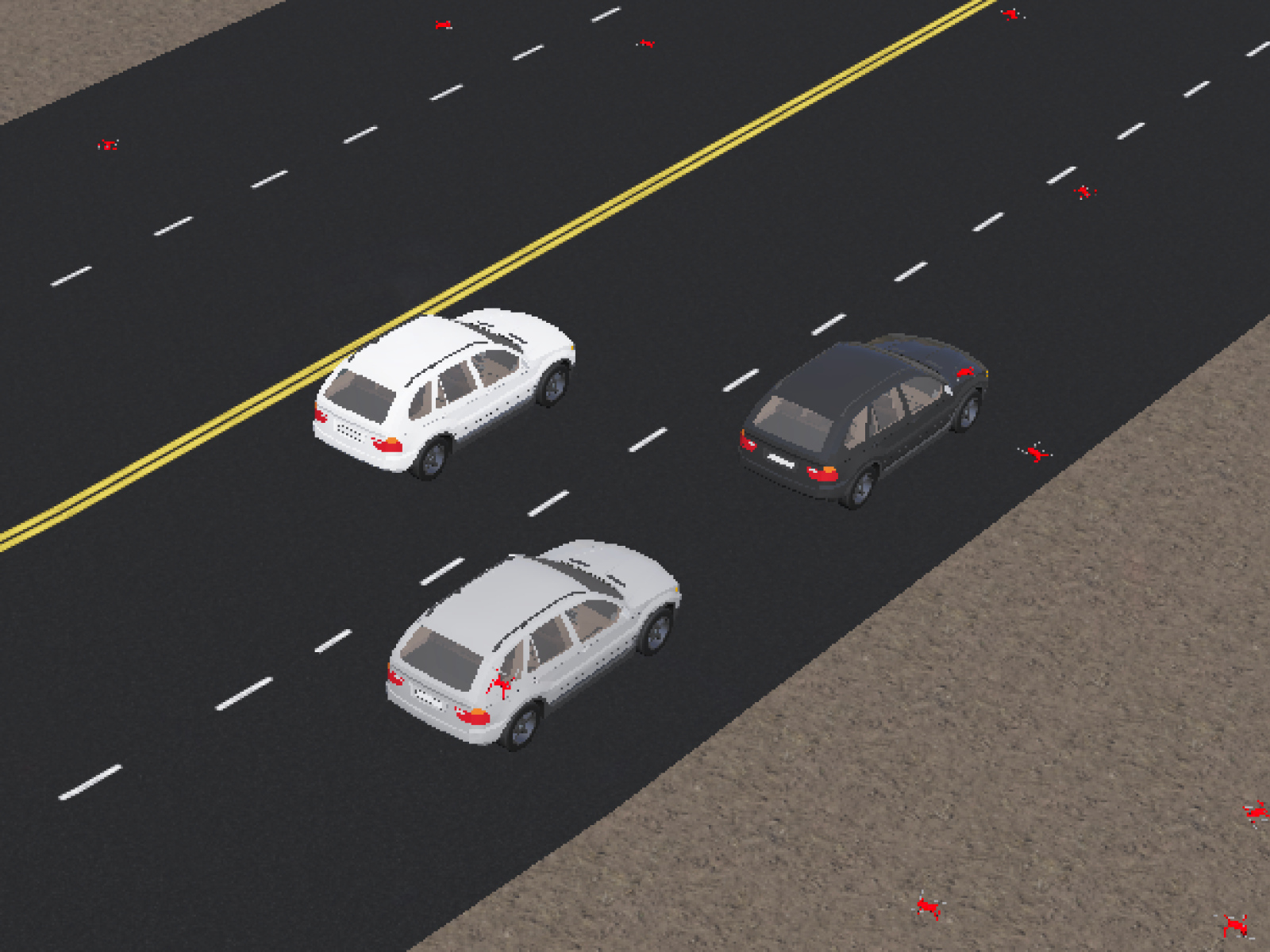}
\caption{Scatter initialization}
\label{fig:case1a}
\end{subfigure}
\hfill
\begin{subfigure}{.32\linewidth}
    \includegraphics[width=\linewidth]{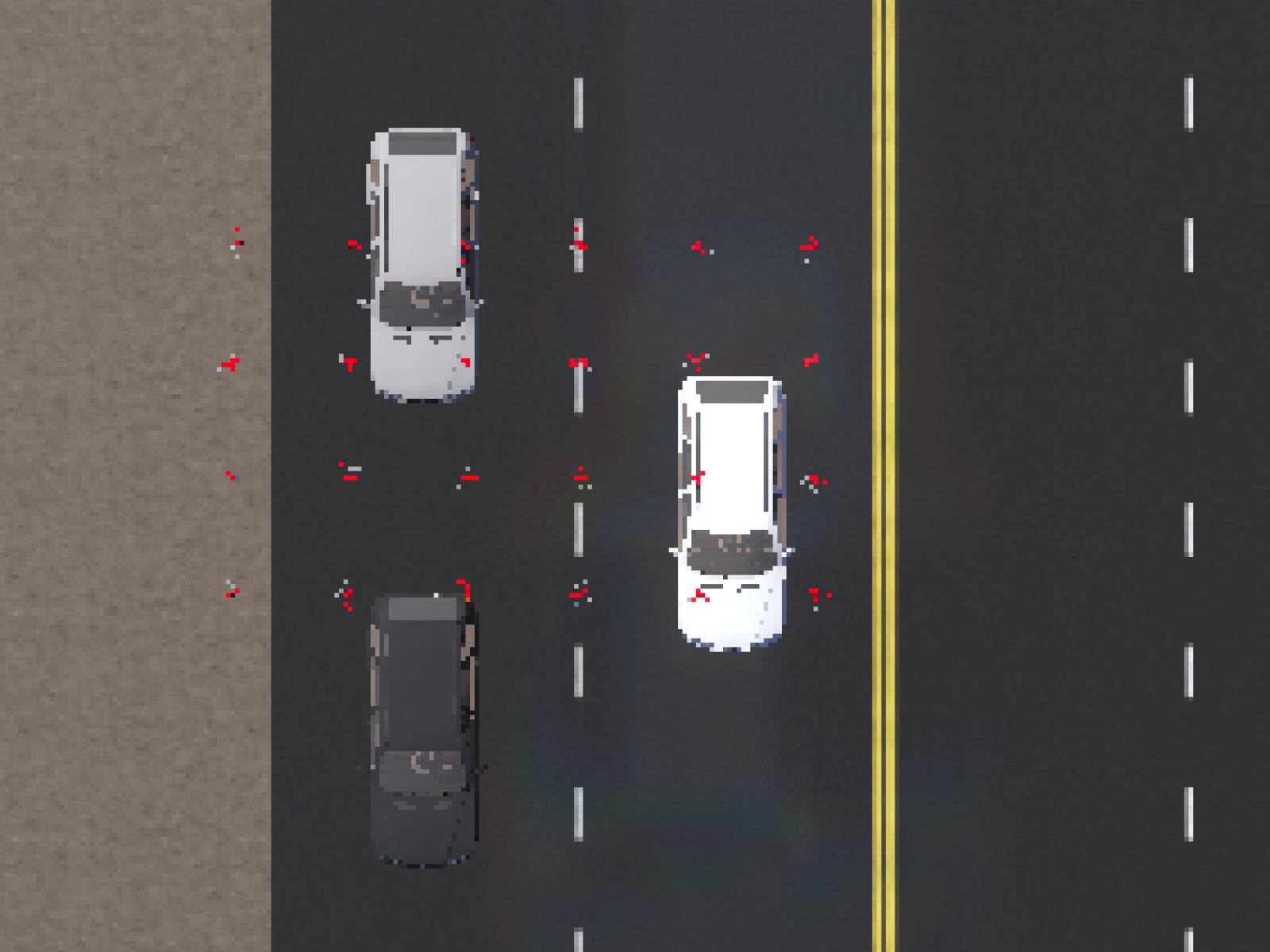}
\caption{Group tracking formation}
\label{fig:case1b}
\end{subfigure}
\hfill
\begin{subfigure}{.32\linewidth}
    \includegraphics[width=\linewidth]{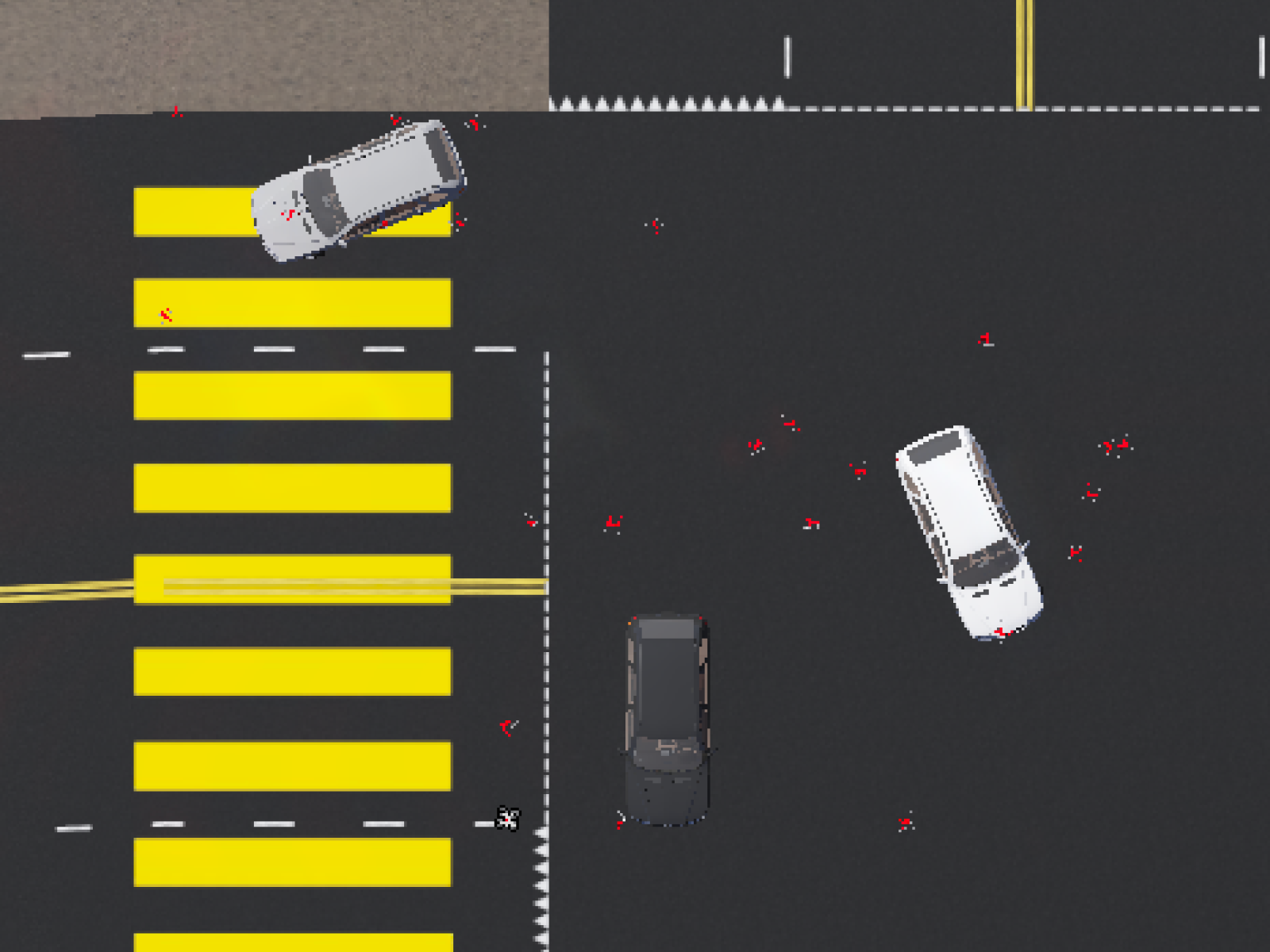}
    \caption{Drone formation after targets split}
\label{fig:case1c}
\end{subfigure}

\begin{subfigure}{.32\linewidth}
\includegraphics[width=\linewidth]{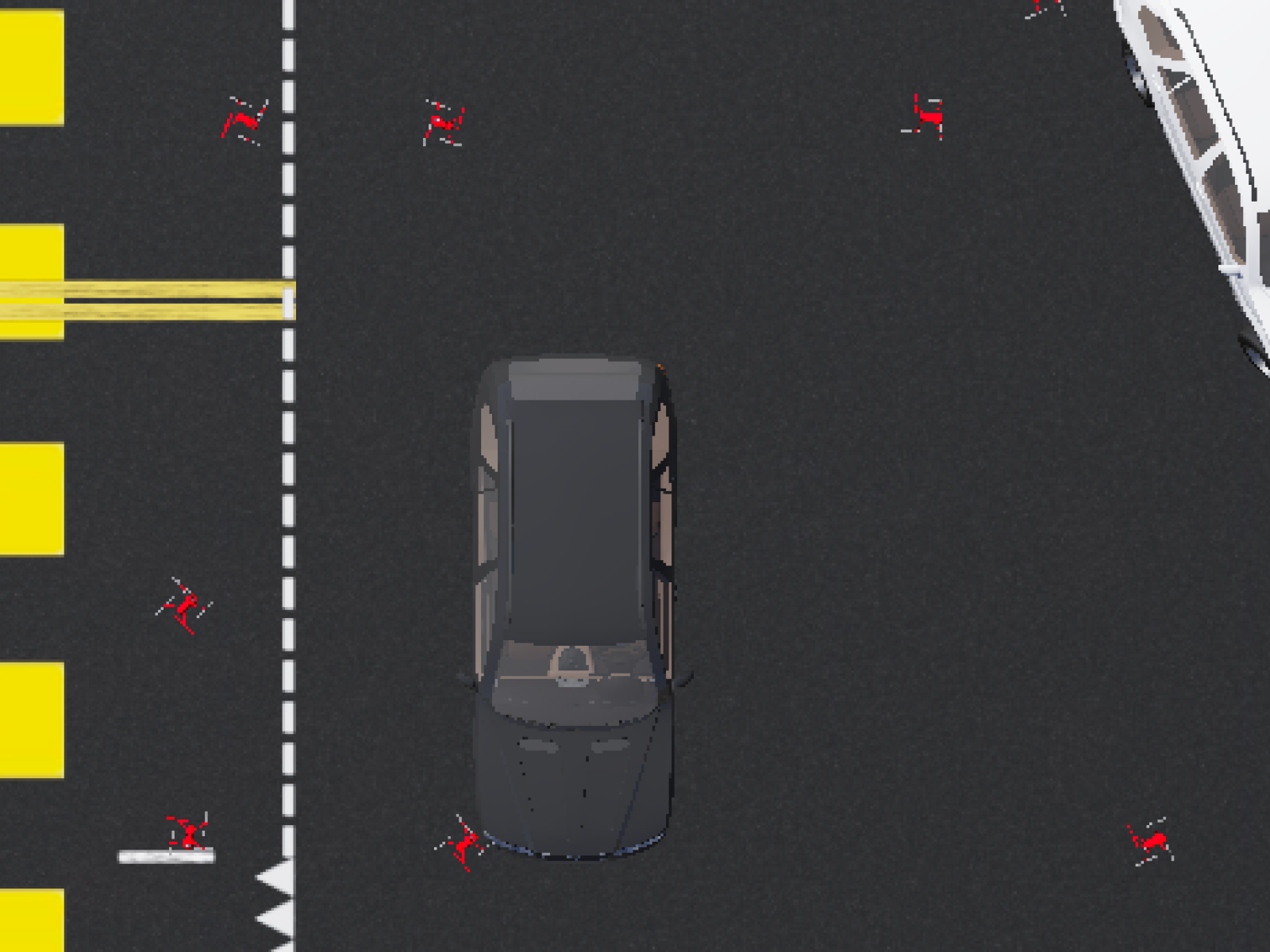}
\caption{Car 1 after targets split}
\label{fig:case1d}
\end{subfigure}
\hfill
\begin{subfigure}{.32\linewidth}
   \includegraphics[width=\linewidth]{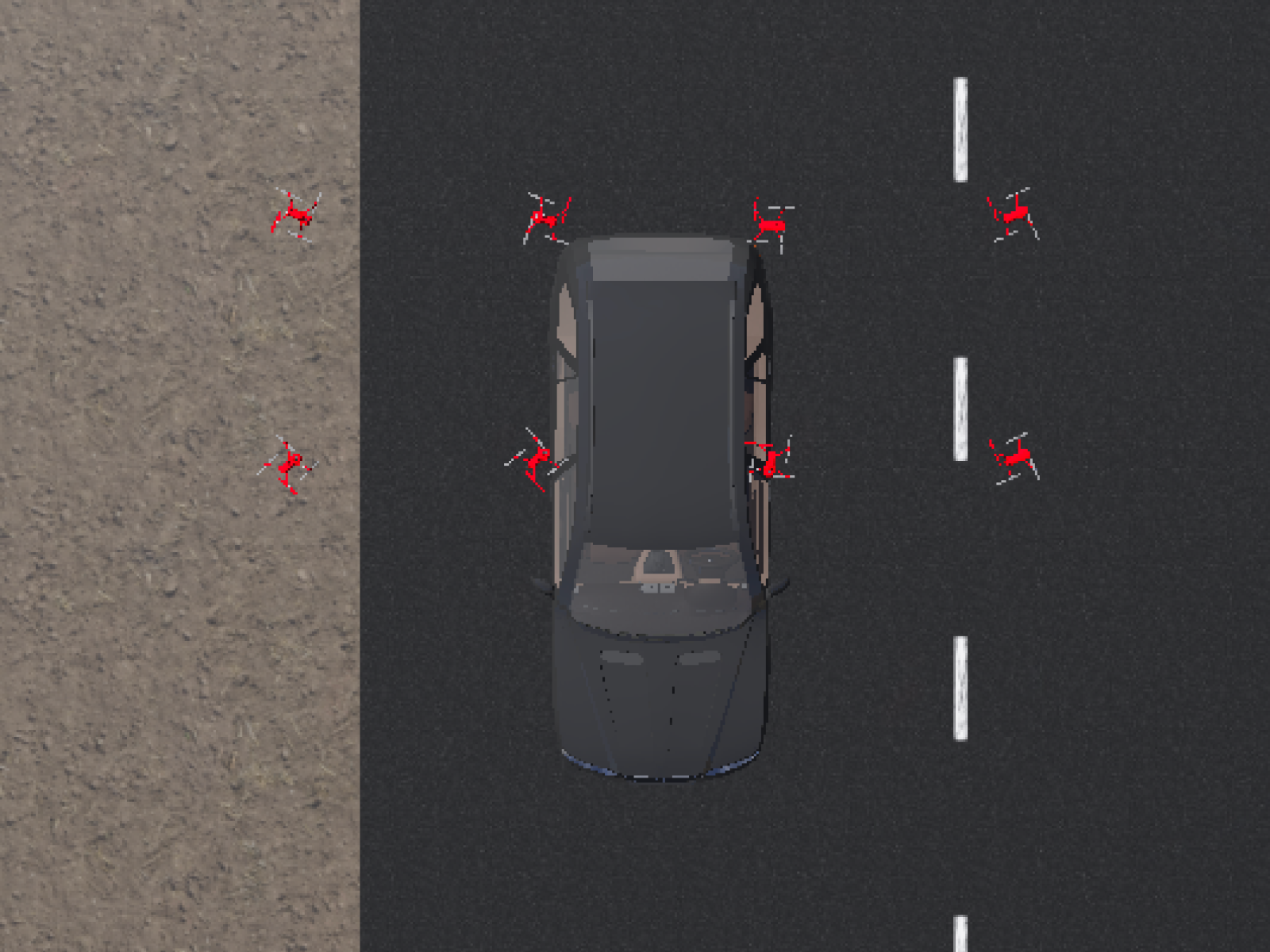}
\caption{Corrected formation on car 1}
\label{fig:case1g}
\end{subfigure}
 \hfill
 \begin{subfigure}{.32\linewidth}
    \includegraphics[width=\linewidth]{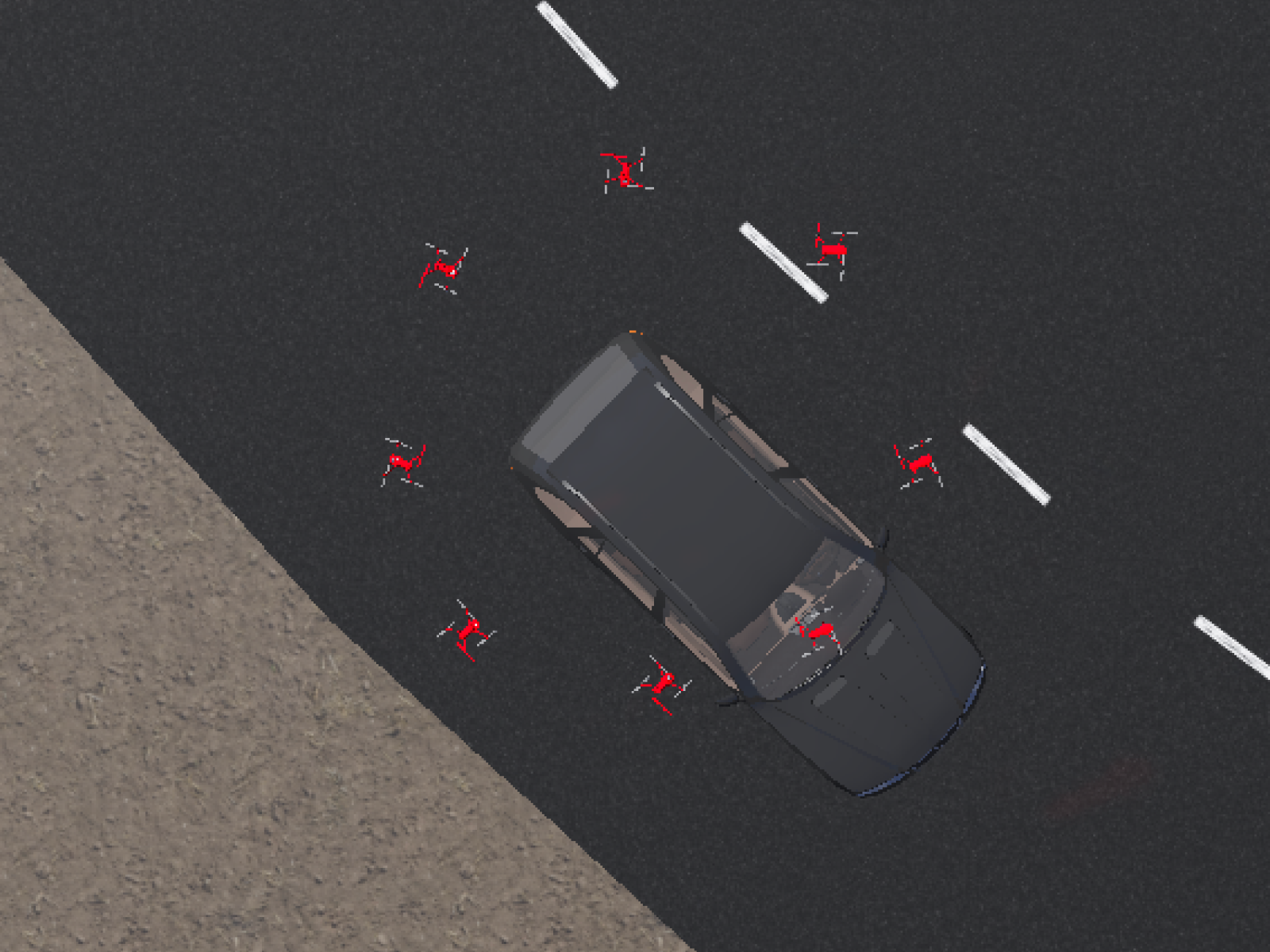}
\caption{Reformation on car 1}
 \label{fig:case1j}
 \end{subfigure}


\caption{In~\Cref{fig:case1a}, 24 drones are initially scattered around the target vehicles. The user commands the drones to form a grid and track the vehicles (\Cref{fig:case1b}). As the vehicles diverge, the inner-layer controller splits the swarm to track each target, temporarily degrading the formation (\Cref{fig:case1c}). The LLM supervisor then recenters and rebalances the sub-swarms (\Cref{fig:case1g}). Finally, the user requests reformation into a circle, square, and cross (\Cref{fig:case1j}). Additional details are shown in~\Cref{fig:app_figure}.}
\label{fig:figure}
\end{figure*}

\paragraph{Phase I: convoy tracking in grid formation.}
Initially, the officer commands the swarm to \emph{follow the group of cars in a [grid] formation}. The system instantiates the specified formation (grid) and tracks the suspects as a convoy while the cars remain sufficiently close, see \Cref{fig:case1b}.

\paragraph{Phase II: suspects split and local reassignment.}
After the suspects detect surveillance, the three cars split and move in different directions. Once the inter-target
separation exceeds a preset threshold, the inner layer performs a fast, local reassignment based on target-proximity, producing three pursuit sub-swarms. During this transient, the formation quality degrades:
agents migrate between sub-swarms, the formation geometry (grid) is temporarily distorted, and the three groups can become imbalanced in size, as in \Cref{fig:case1d} and \Cref{fig:app_figure}.

\paragraph{Phase III: mid-layer verification and regrouping.}
At the subsequent verification instances, the mid-layer LLM supervisor detects distorted grids and uneven group sizes, then applies automatic correction: it re-grounds the command using the \emph{current} target estimates, rebalances the drone--target assignments to
yield roughly equal subgroup sizes, and regenerates clean per-target formation references (grid). The inner layer then recenters each
sub-swarm to an organized grid while continuing to track the assigned car as in \Cref{fig:case1g} and \Cref{fig:app_figure}.

\paragraph{Phase IV: human formation switching for target disambiguation.}
Finally, the officer issues a command to visually distinguish suspects: \emph{one group forms a [circle],
one forms a [square], and one forms a [cross]}. The supervisor updates the per-target formation templates accordingly, and the
inner layer drives each subgroup to its new formation while preserving target tracking, as shown in \Cref{fig:case1j} and \Cref{fig:app_figure}.

\paragraph{Results and Discussions}
In the police-chasing scenarios, we focus on testing the target tracking ability of the proposed framework and how the LLM supervision helps in correcting formation-command mismatch caused by the change of system dynamics when the suspects split. We see that when the targets initially split, the inner layer distributed control helps to roughly distribute drones to each target (\Cref{fig:case1b}), while the mid-layer LLM auto-verification gives a relative position assignment to each drone more aligned with human-intent (\Cref{fig:case1g} and \Cref{fig:app_figure}). In the entire process, there is no involvement of human user giving command for adaptation to environmental changes, and the command they give demonstrates new user intentions, which are then incorporated into LLM supervision in the future. Overall, the LLM supervision performed as intended and greatly reduced the human involvement in the process. 

\begin{figure*}[!ht]
\begin{subfigure}{.32\linewidth}
    \includegraphics[width=\linewidth]{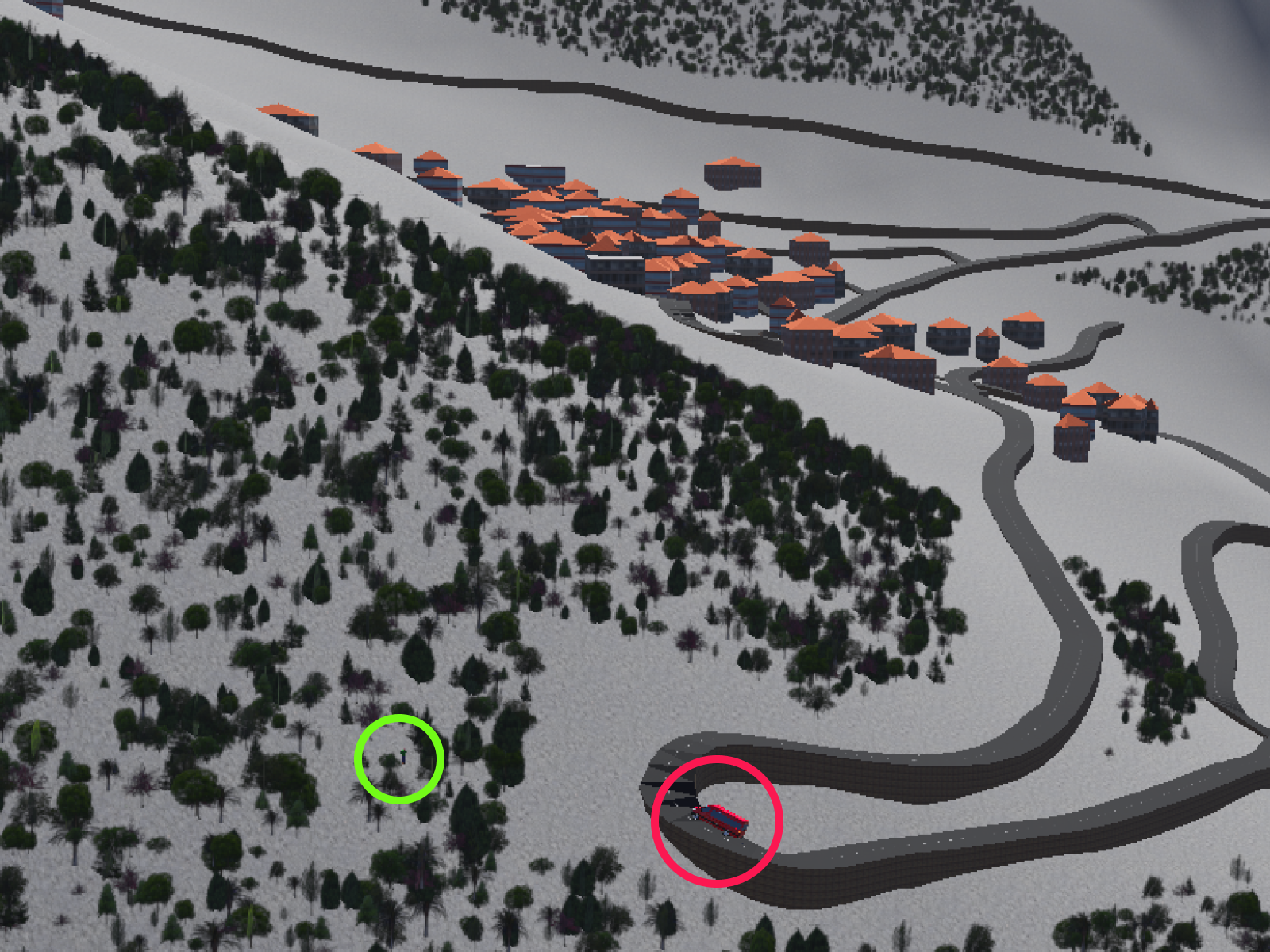}
\caption{Scenario illustration}
\label{fig:case2a}
\end{subfigure}
\hfill
\begin{subfigure}{.32\linewidth}
    \includegraphics[width=\linewidth]{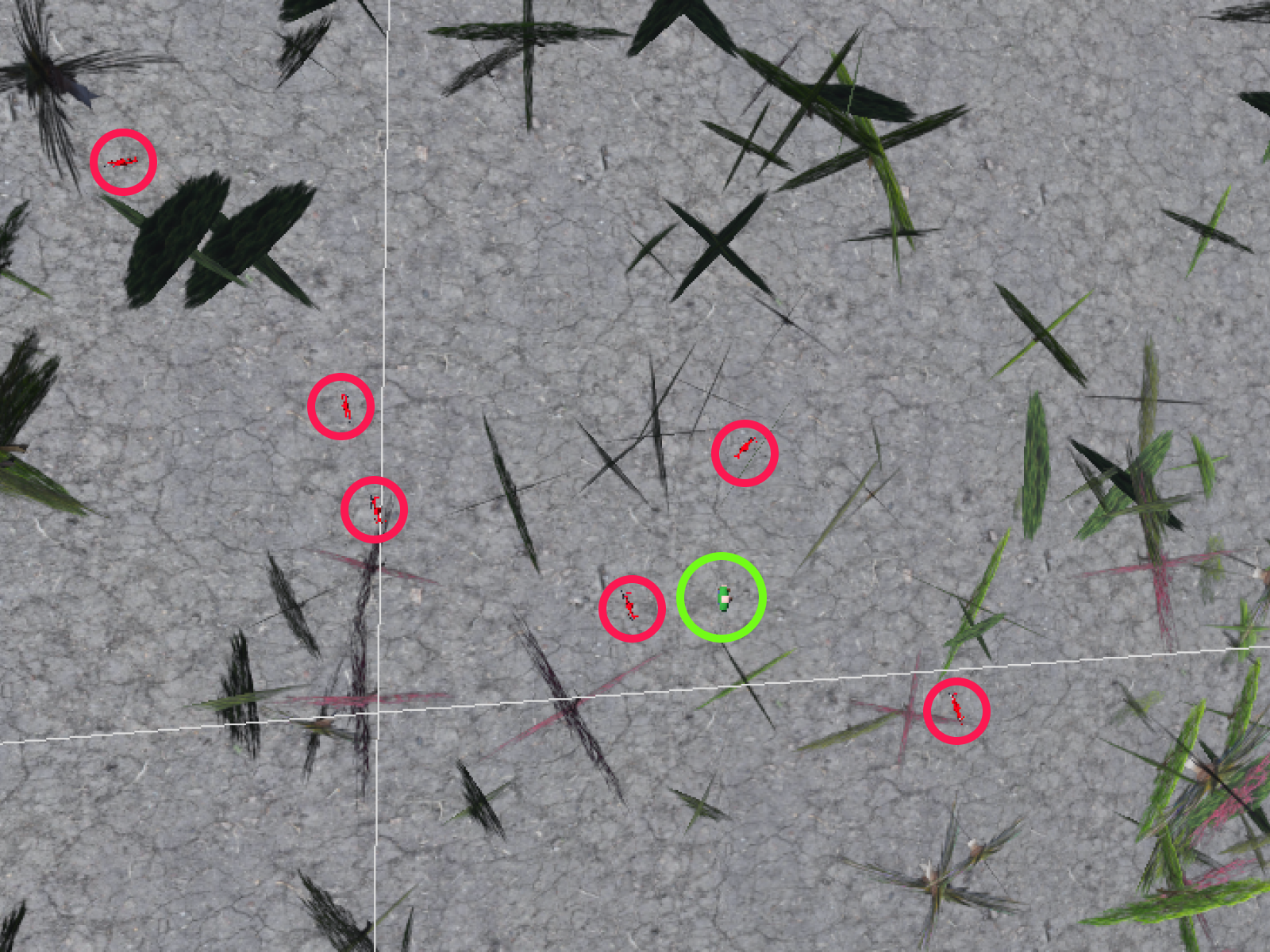}
\caption{Exploration of the region}
\label{fig:case2b}
\end{subfigure}
\hfill
\begin{subfigure}{.32\linewidth}
    \includegraphics[width=\linewidth]{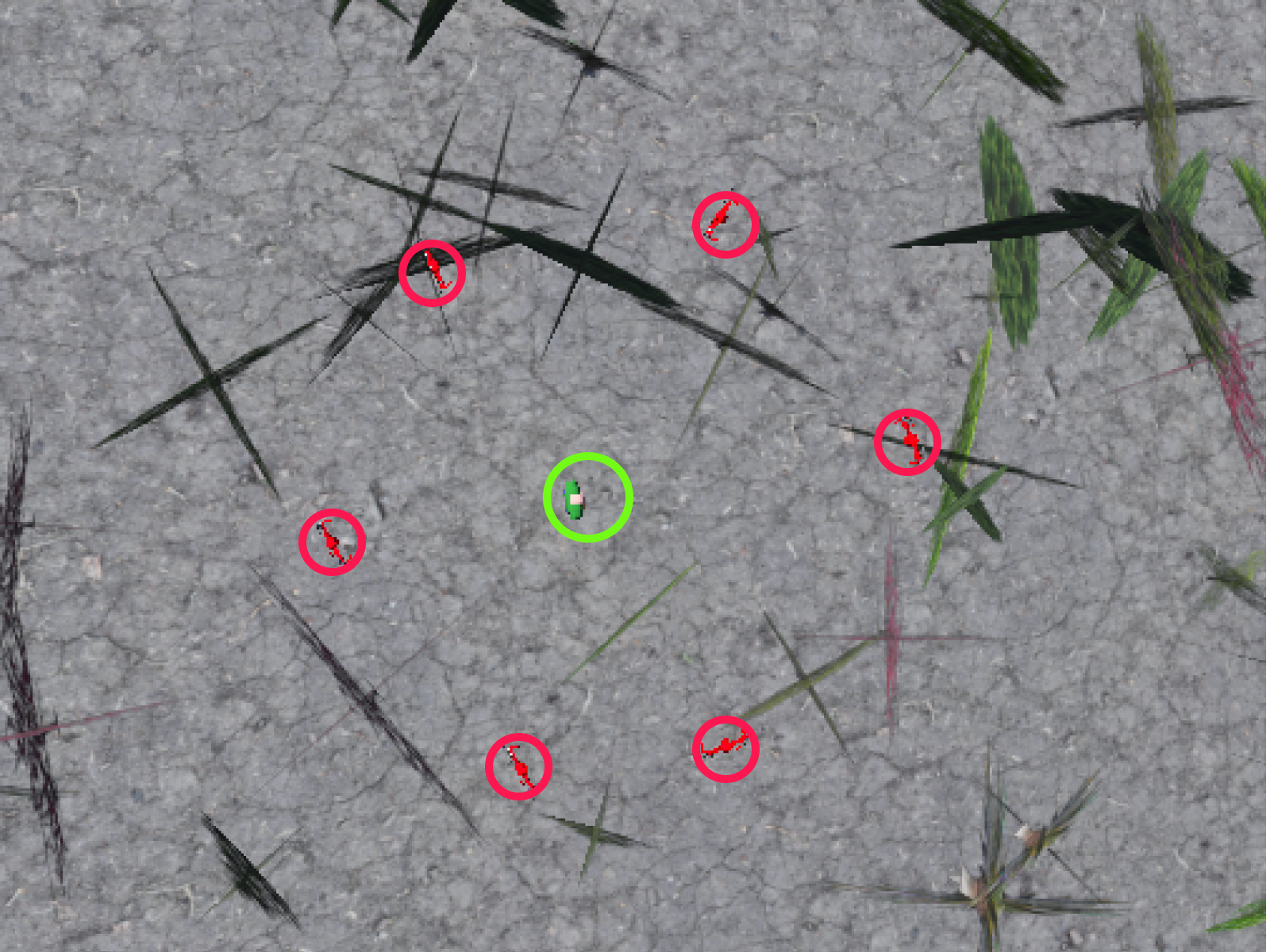}
    \caption{Encirclement pattern}
\label{fig:case2c}
\end{subfigure}
\caption{In~\Cref{fig:case2a}, we show a mountain search-and-rescue scenario with a rescue vehicle (red circle) and a missing person (green circle). After the user commands the drones to explore, they move to assigned waypoints while searching the region (\Cref{fig:case2b}). Upon detection, the drones form an encirclement around the target (\Cref{fig:case2c}).}
\label{fig:case2}
\end{figure*}

\subsection{LLM-guided search and rescue in a forest environment}
\label{sec:sim_sar}
We consider 3 search-and-rescue scenarios in which a single missing person is believed to be within a bounded forest region under severe environmental conditions.
An officer issues high-level instructions through the LLM interface, and the drones are assigned random coordinates in the forest to search for the missing person.

\paragraph{Phase I: randomized waypoint assignment for area search.}
At the start of the mission, all drones are located at the start point, as in \Cref{fig:case2a}. The LLM supervisor generates candidate search waypoints within the user-specified search region. Each drone is assigned a waypoint to inspect,
yielding a spatially distributed search pattern without requiring the officer to micromanage coverage.

\paragraph{Phase II: event-driven re-tasking upon negative observations.}
When a drone reaches its assigned waypoint without detecting the missing person anywhere along its path, it reports a
\emph{clear} status to the supervisor. The LLM then issues a new waypoint assignment to that drone,
thereby maintaining continuous exploration of the region. This process repeats until a detection event occurs. In this phase,
the swarm behavior is dominated by successive waypoint updates, while the inner-loop controller ensures stable motion between
updates and smooth transitions when new waypoints are issued, as shown in \Cref{fig:case2b}.

\paragraph{Phase III: detection-triggered rally and encirclement.}
If any drone detects the missing person, it immediately reports the detection and location.
The LLM supervisor switches the mission objective from \emph{search} to \emph{assist} by directing all drones to converge to
the detected target location and adopt an encirclement pattern. Specifically, the supervisor issues a circle formation
reference centered at the detected position to mark the location for
ground responders and maintain line-of-sight coverage. The resulting
circle formation provides a persistent visual and sensing cue to guide the rescue effort, as shown in \Cref{fig:case2c}.

\paragraph{Reported outputs.}
We summarize this simulation by reporting (i) time-to-detection, (ii) number of waypoint reassignments prior to detection,
(iii) search efficiency indicators (e.g., fraction of candidate waypoints cleared versus elapsed time), and (iv) formation
quality during the encirclement phase (e.g., circle-radius error and inter-drone spacing regularity).
Representative 3D snapshots illustrate the transition from randomized search to coordinated target encirclement.

\clearpage
\section{Limitations}
This work demonstrates a framework of using LLM to assist with real-life UAV distributed control. It did not use the newly developed visual-language-action model or multimodal policies, some of which have been introduced in \Cref{sec:related_work}. The authors anticipate that using those tools would further enhance the framework's functionality and bring it one step closer to real-world deployment. Furthermore, we introduce only three human intentions, but the proposed framework is flexible enough to incorporate additional ones as the user determines necessary.

\bibliography{custom}
\clearpage
\appendix

\section{Additional Figures and Simulations}

In this section, we provide some additional figures for the existing simulation, together with additional scenarios under the same general settings.

\textbf{Police–chasing scenario 1.}
In~\Cref{fig:app_figure}, we provide a detailed visualization of the drone layouts over the target vehicles throughout the scenario. The suspect vehicles initially travel together along the road and then diverge in different directions at the intersection. The drone layouts are automatically updated by the mid-layer LLM supervisor in response to these changes. After the operator issues additional commands for the drones to adopt different formations (circle, square, and cross), the outer layer processes the new instructions and the middle layer reformulates the corresponding formation references accordingly. The inner layer continues to operate at each time step to ensure that the desired formation is achieved. 
\begin{figure*}
\begin{subfigure}{.32\linewidth}
    \includegraphics[width=\linewidth]{figures/fig1_init.png}
\caption{Scatter initialization}
\label{fig:app_case1_1a}
\end{subfigure}
\hfill
\begin{subfigure}{.32\linewidth}
    \includegraphics[width=\linewidth]{figures/fig2_merge.png}
\caption{Group tracking formation}
\label{fig:app_case1_1b}
\end{subfigure}
\hfill
\begin{subfigure}{.32\linewidth}
    \includegraphics[width=\linewidth]{figures/fig3_chaos.png}
    \caption{Drone formation after target split}
\label{fig:app_case1_1c}
\end{subfigure}

\begin{subfigure}{.32\linewidth}
    \includegraphics[width=\linewidth]{figures/fig7_chaos1.png}
\caption{Car 1 after target split}
\label{fig:app_case1_1d}
\end{subfigure}
\hfill
\begin{subfigure}{.32\linewidth}
    \includegraphics[width=\linewidth]{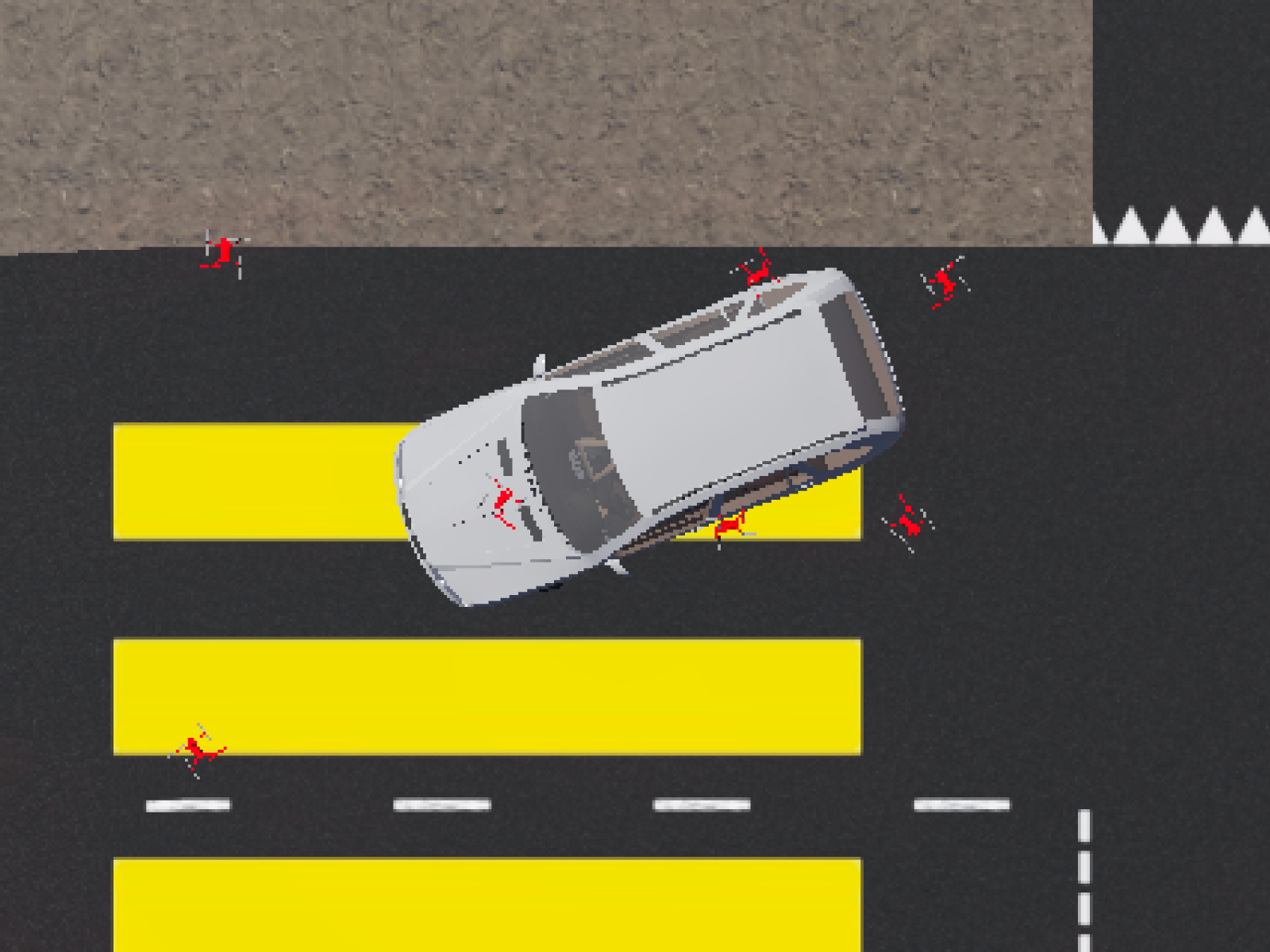}
\caption{Car 2 after target split}
\label{fig:app_case1_1e}
\end{subfigure}
\hfill
\begin{subfigure}{.32\linewidth}
    \includegraphics[width=\linewidth]{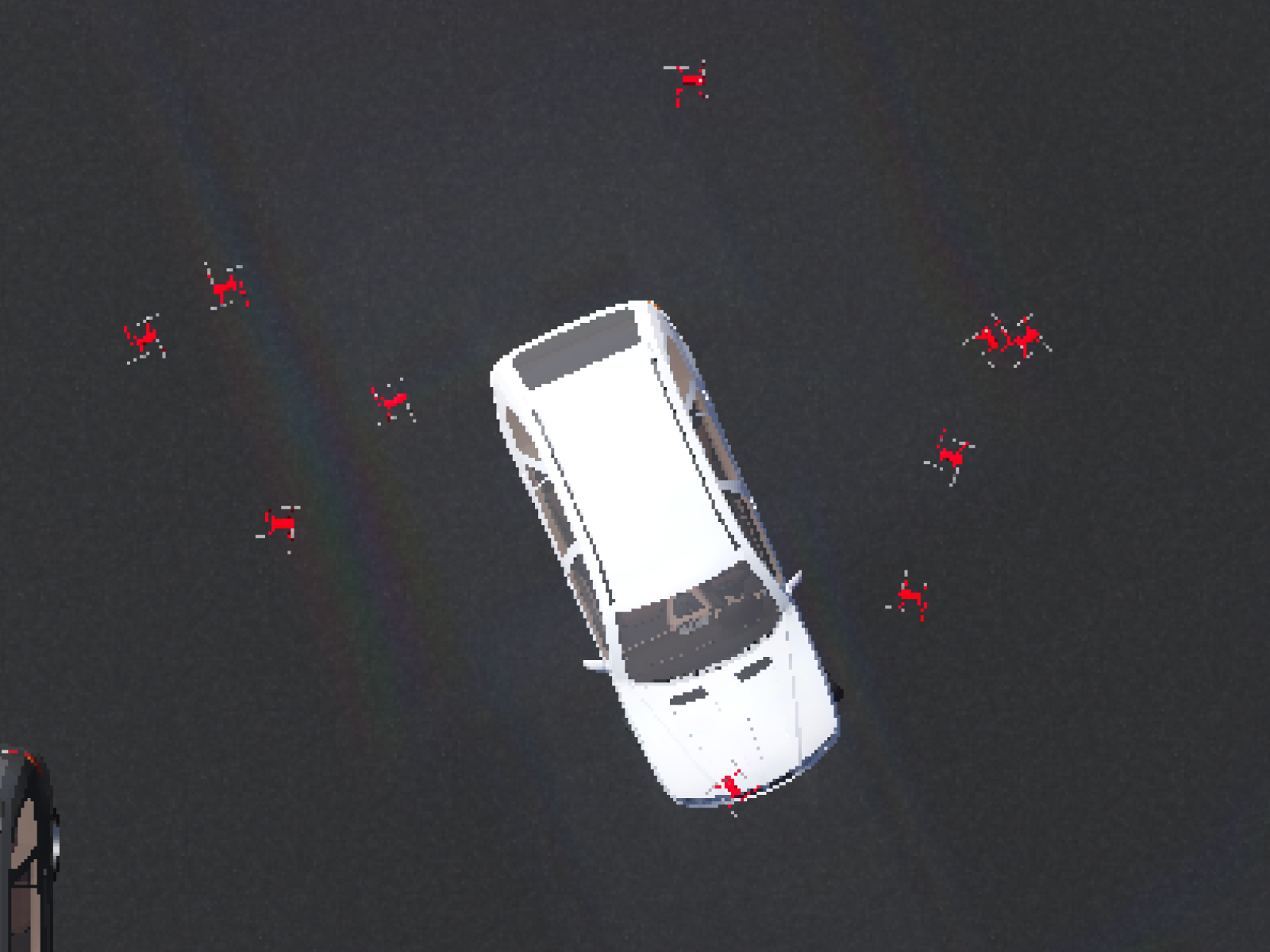}
    \caption{Car 3 after target split}
\label{fig:app_case1_1f}
\end{subfigure}

\begin{subfigure}{.32\linewidth}
    \includegraphics[width=\linewidth]{figures/fig4_car1.png}
\caption{Corrected formation on car 1}
\label{fig:app_case1_1g}
\end{subfigure}
\hfill
\begin{subfigure}{.32\linewidth}
    \includegraphics[width=\linewidth]{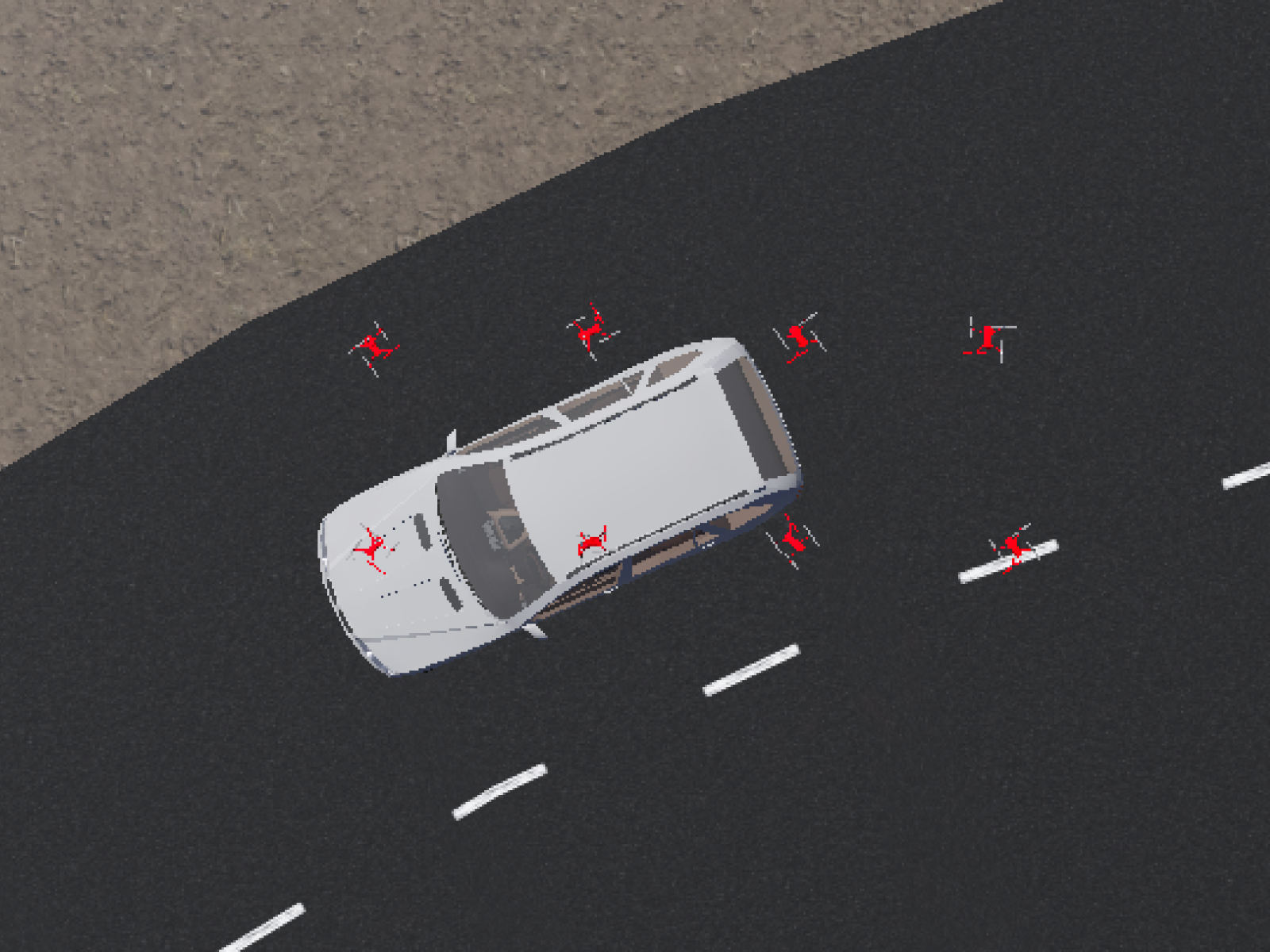}
\caption{Corrected formation on car 2}
\label{fig:app_case1_1h}
\end{subfigure}
\hfill
\begin{subfigure}{.32\linewidth}

\includegraphics[width=\linewidth]{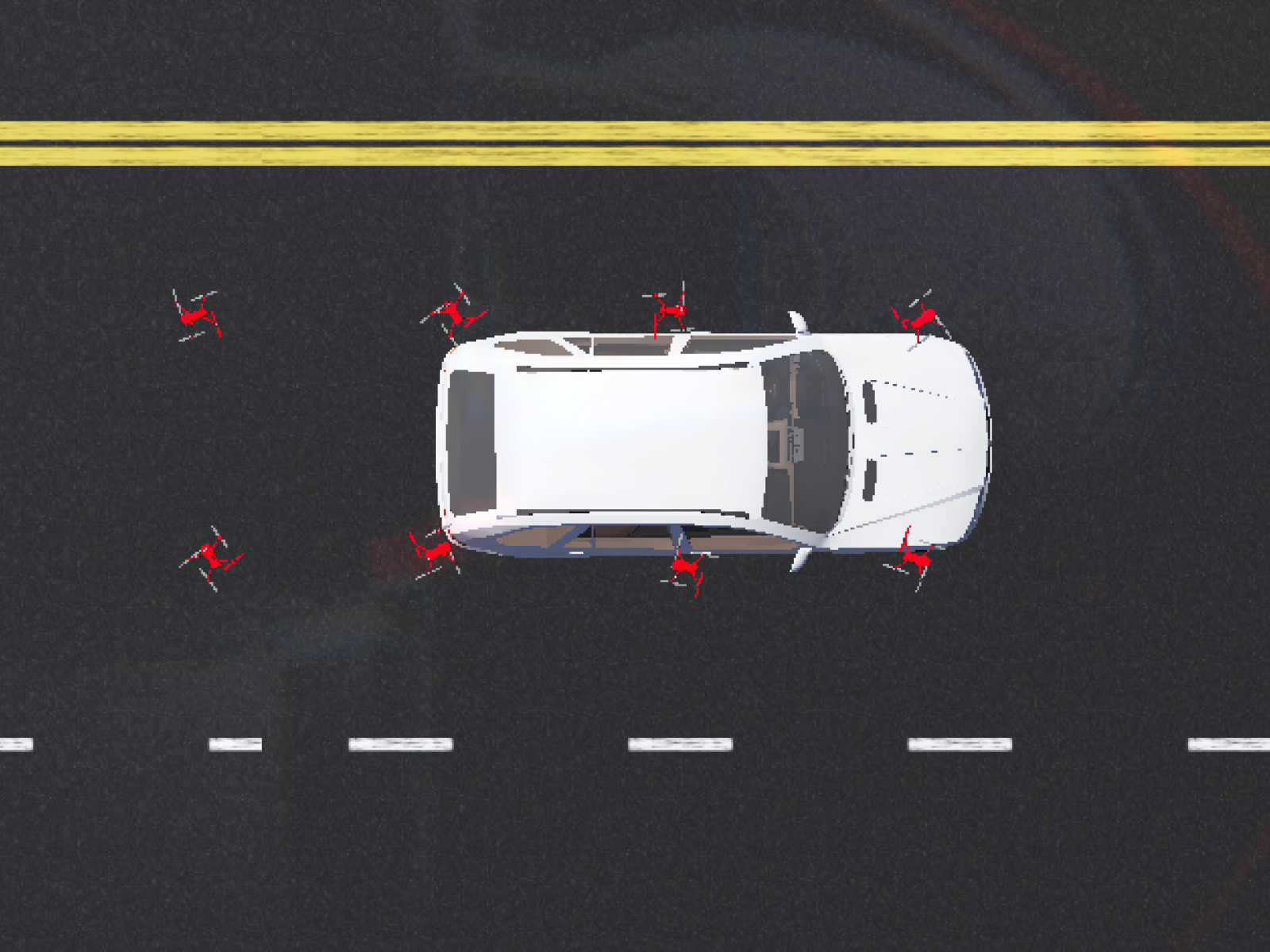}
\caption{Corrected formation on car 3}
\label{fig:app_case1_1i}
\end{subfigure}

\begin{subfigure}{.32\linewidth}
    \includegraphics[width=\linewidth]{figures/fig10_reform1.png}
\caption{Reformation on car 1}
\label{fig:app_case1_1j}
\end{subfigure}
\hfill
\begin{subfigure}{.32\linewidth}
    \includegraphics[width=\linewidth]{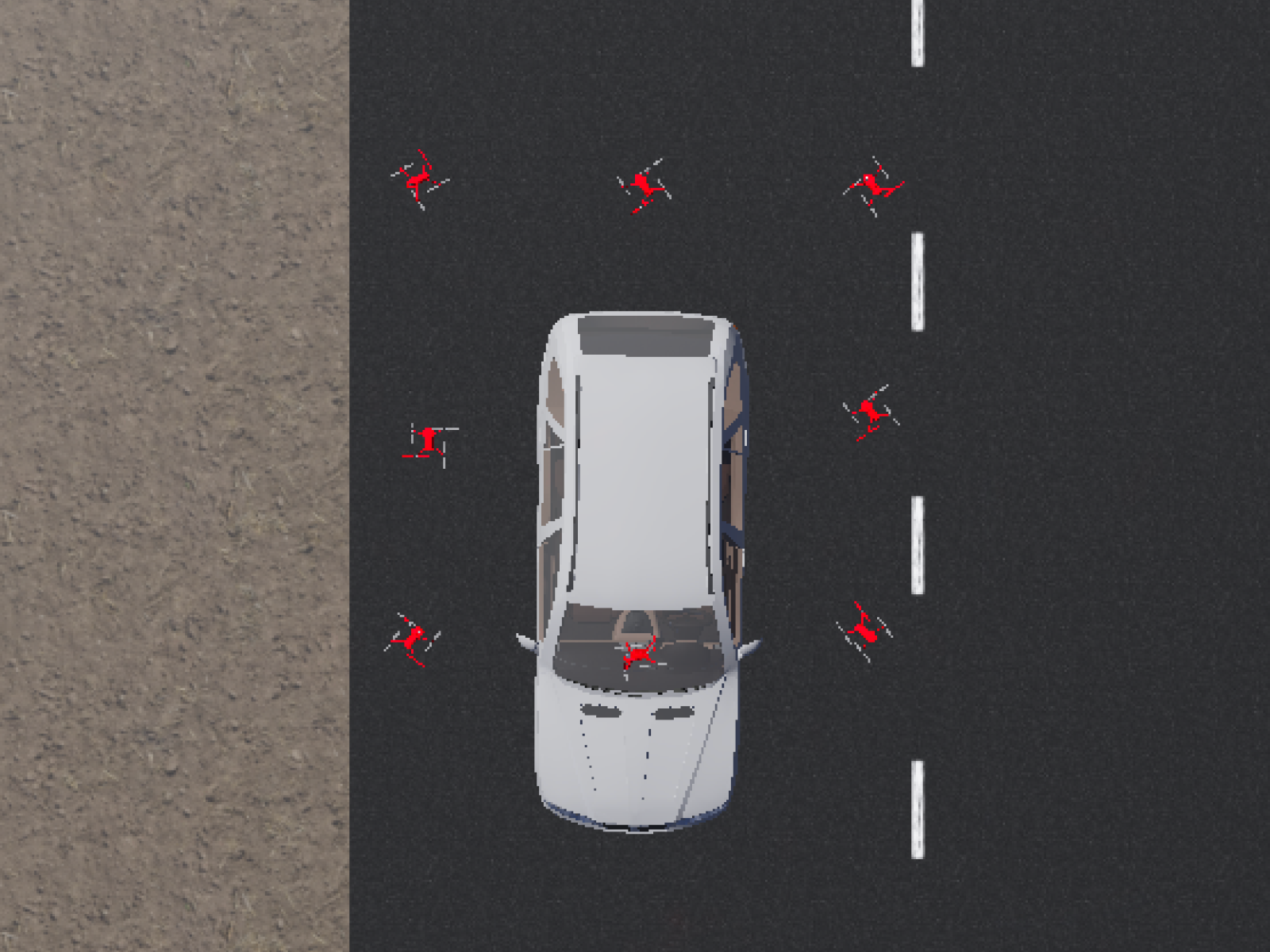}
\caption{Reformation on car 2}
\label{fig:app_case1_1k}
\end{subfigure}
\hfill
\begin{subfigure}{.32\linewidth}

\includegraphics[width=\linewidth]{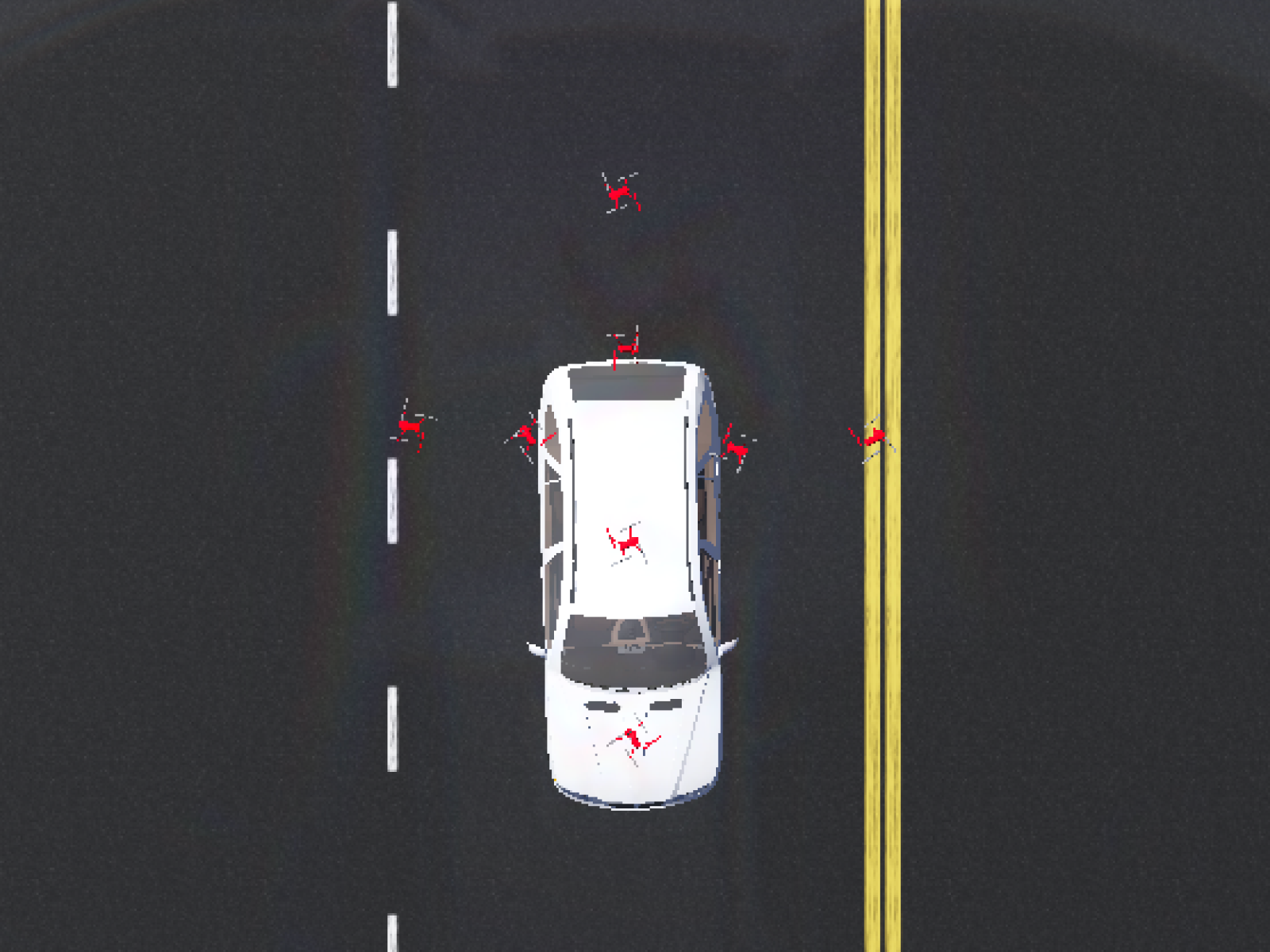}
\caption{Reformation on car 3}
\label{fig:app_case1_1l}
\end{subfigure}
\caption{In \Cref{fig:app_case1_1a}, we have 24 drones scattered around the target cars. The user then commands the drones to form a grid around the cars with LLM, as shown in \Cref{fig:app_case1_1b}. As the cars move toward the road intersection, they diverge in different directions. The inner-layer control algorithm has the group of drones diverted to track different cars, as shown in \Cref{fig:app_case1_1c}. However, the low-level control heuristic does not ensure the drones maintain the formation exactly during splitting. Lastly, the LLM auto-verification recenters each drone on the target and rebalances the group for each target, as shown in \Cref{fig:app_case1_1g}, \Cref{fig:app_case1_1h}, and \Cref{fig:app_case1_1i}. The user then issues additional commands to reform the groups into a circle, square, and cross, as shown in \Cref{fig:app_case1_1j}, \Cref{fig:app_case1_1k}, and \Cref{fig:app_case1_1l}.}
\label{fig:app_figure}
\end{figure*}

\textbf{Police–chasing scenario 2.}
In~\Cref{fig:app_figure2}, the operator issues a command as \emph{"Three suspects are attempting to escape. Form a circle like a coordinated police dragnet and track all three targets."} 
The suspect vehicles initially travel together along the road and then diverge in different directions at the intersection. As the targets separate, the system automatically adapts the swarm configuration to maintain coordinated tracking of all three vehicles. Compared to grid-based convoy tracking, enforcing a circular encirclement with a larger swarm (24 drones in this example) introduces additional geometric and assignment complexity. Nevertheless, the proposed architecture handles this case in the same unified manner by grounding the user instruction into formation references, dynamically reassigning drones to targets, and continuously updating the control inputs through the inner layer.
After the drones form three encirclement groups around the three vehicles, the operator updates the instruction \emph{"Stop tracking all the targets."}. The outer and middle layers process this new command, update the swarm configuration, and deactivate the tracking objective. This scenario demonstrates that the proposed framework can robustly handle both the splitting and reconvergence of targets, enabling the swarm to dynamically merge and reform joint behaviors as the task structure evolves.
\begin{figure*}
\begin{subfigure}{.32\linewidth}
    \includegraphics[width=\linewidth]{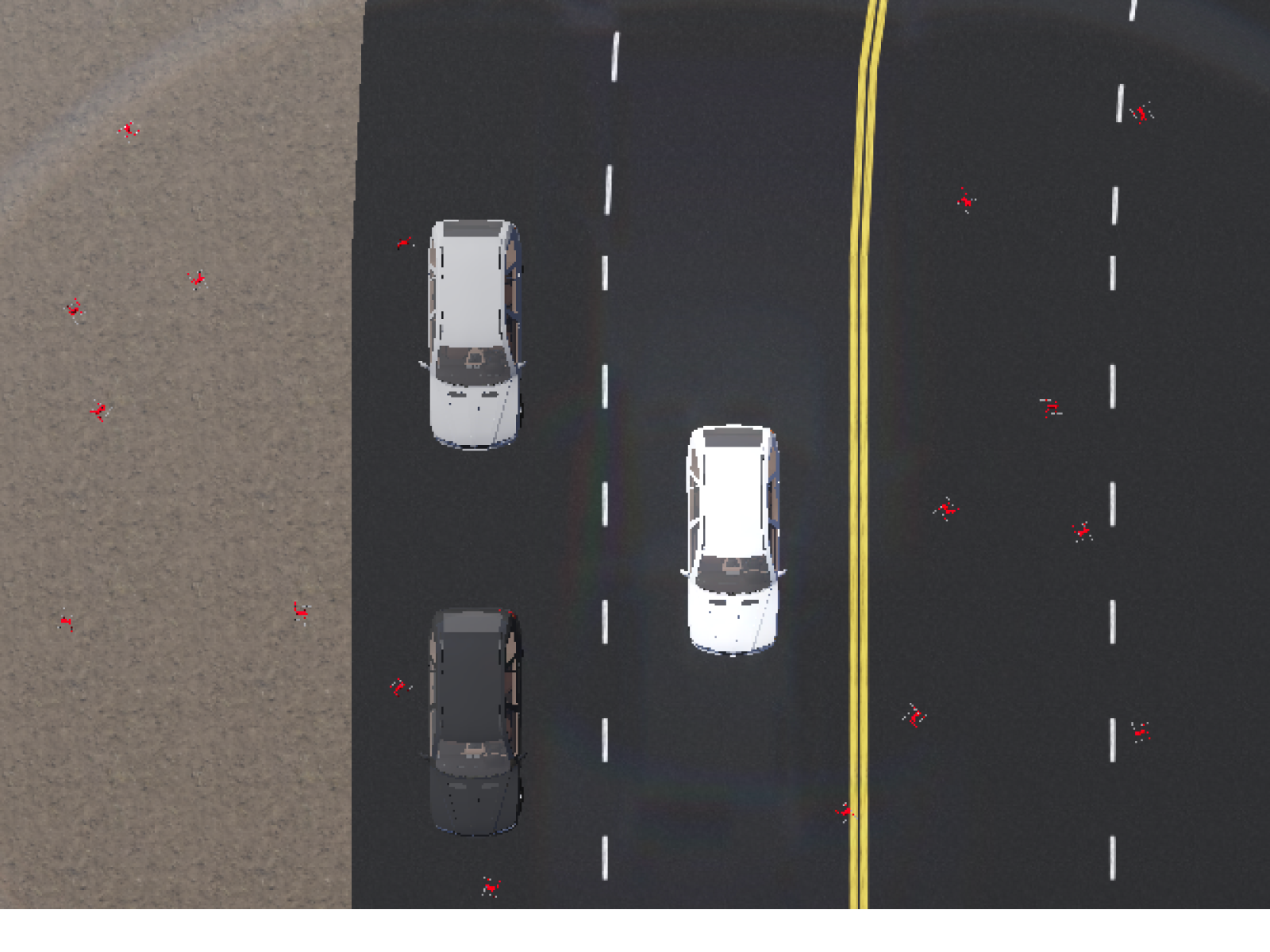}
\caption{Scatter initialization}
\label{fig:app_case1_2a}
\end{subfigure}
\hfill
\begin{subfigure}{.32\linewidth}
    \includegraphics[width=\linewidth]{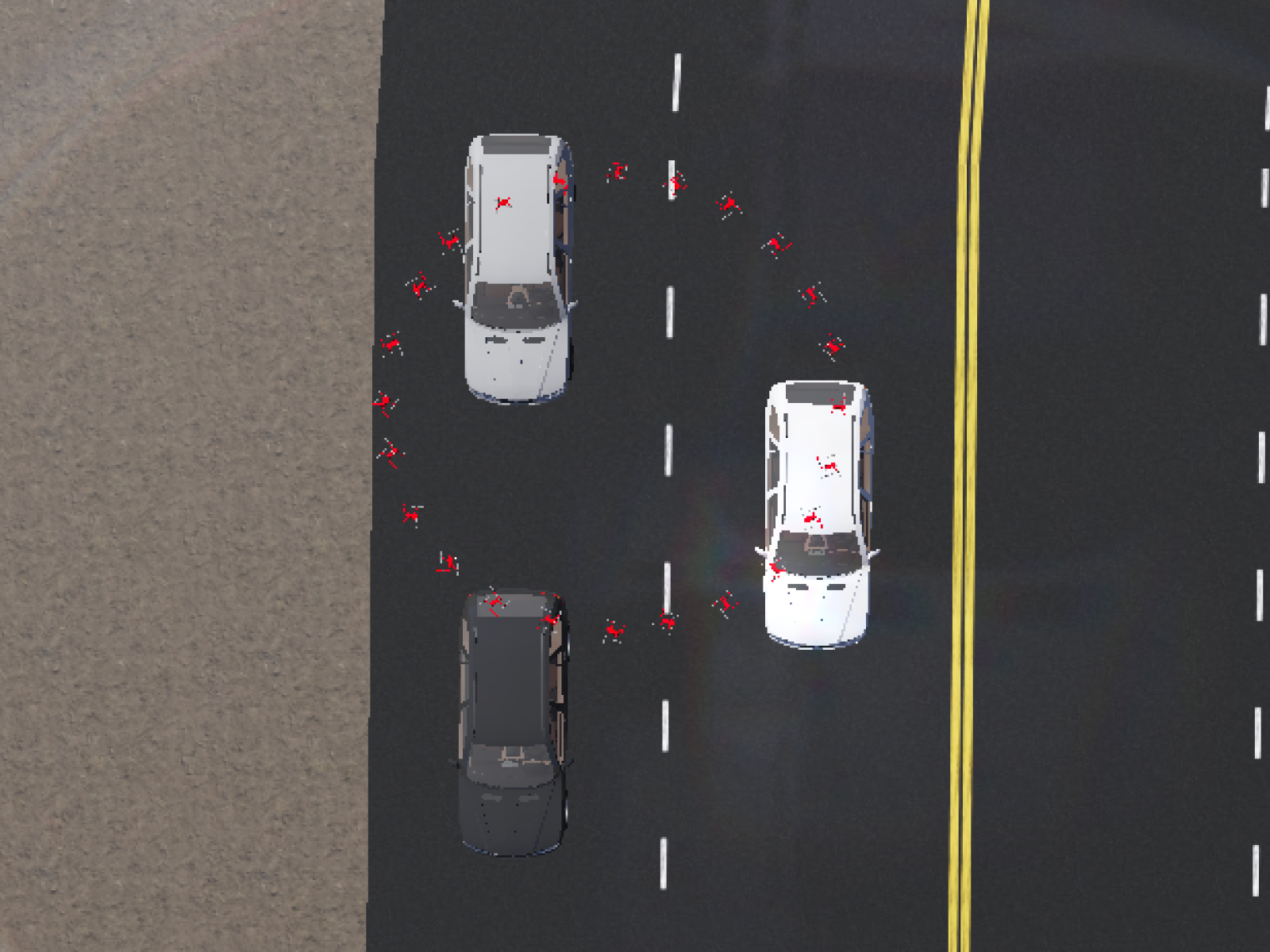}
\caption{Group tracking formation}
\label{fig:app_case1_2b}
\end{subfigure}
\hfill
\begin{subfigure}{.32\linewidth}
    \includegraphics[width=\linewidth]{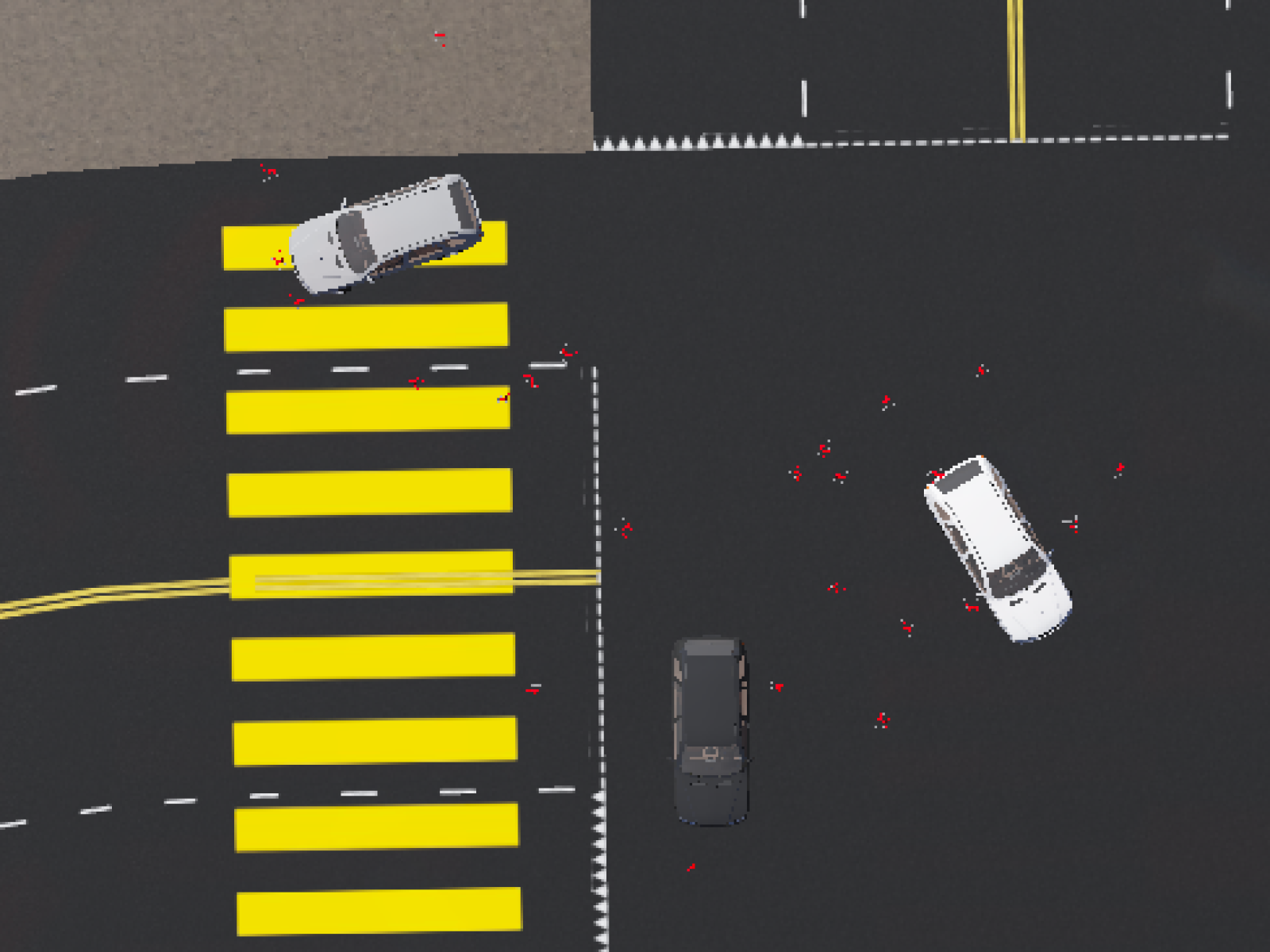}
    \caption{Drone formation after target split}
\label{fig:app_case1_2c}
\end{subfigure}

\begin{subfigure}{.32\linewidth}
    \includegraphics[width=\linewidth]{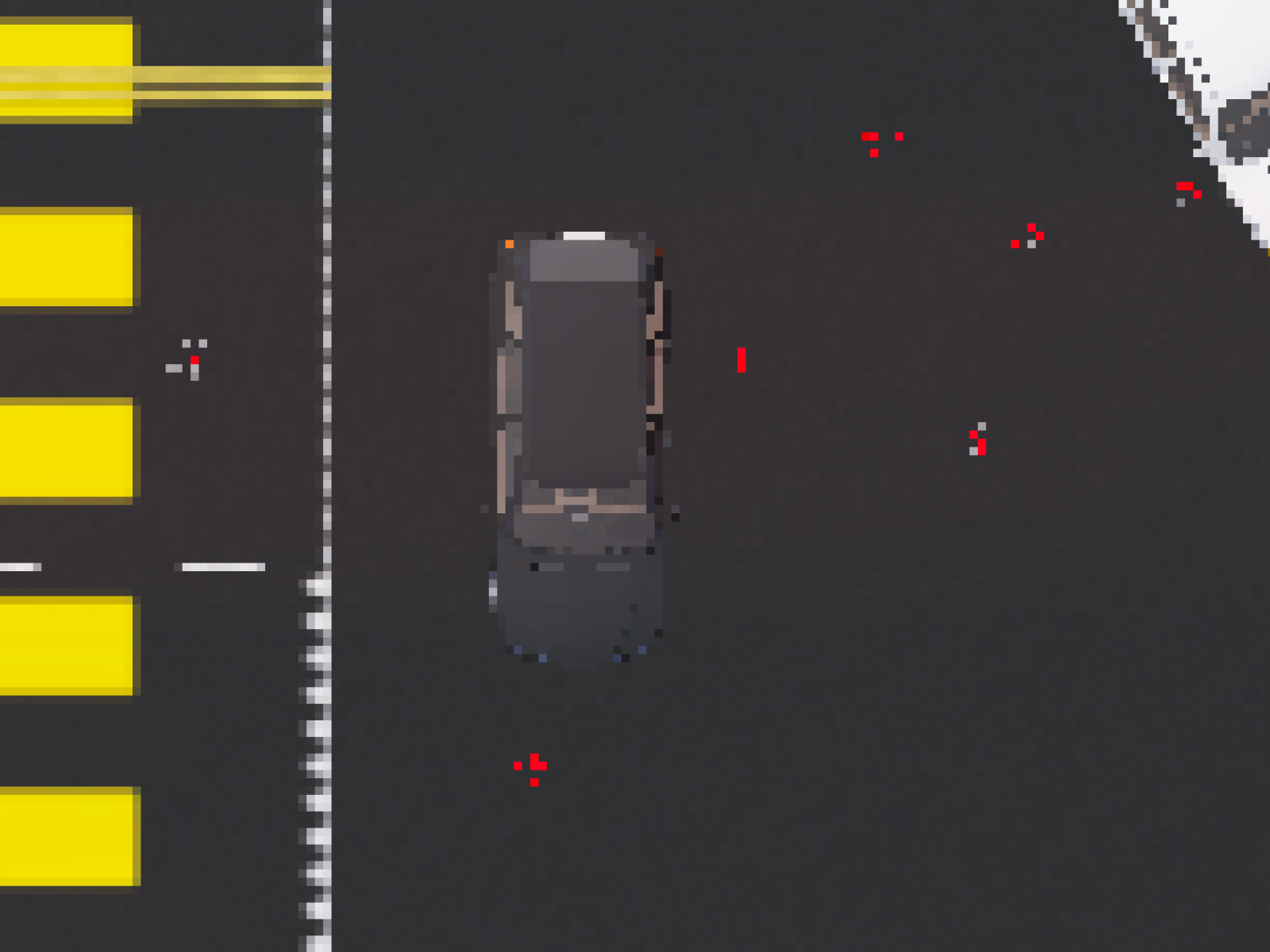}
\caption{Car 1 after target split}
\label{fig:app_case1_2d}
\end{subfigure}
\hfill
\begin{subfigure}{.32\linewidth}
    \includegraphics[width=\linewidth]{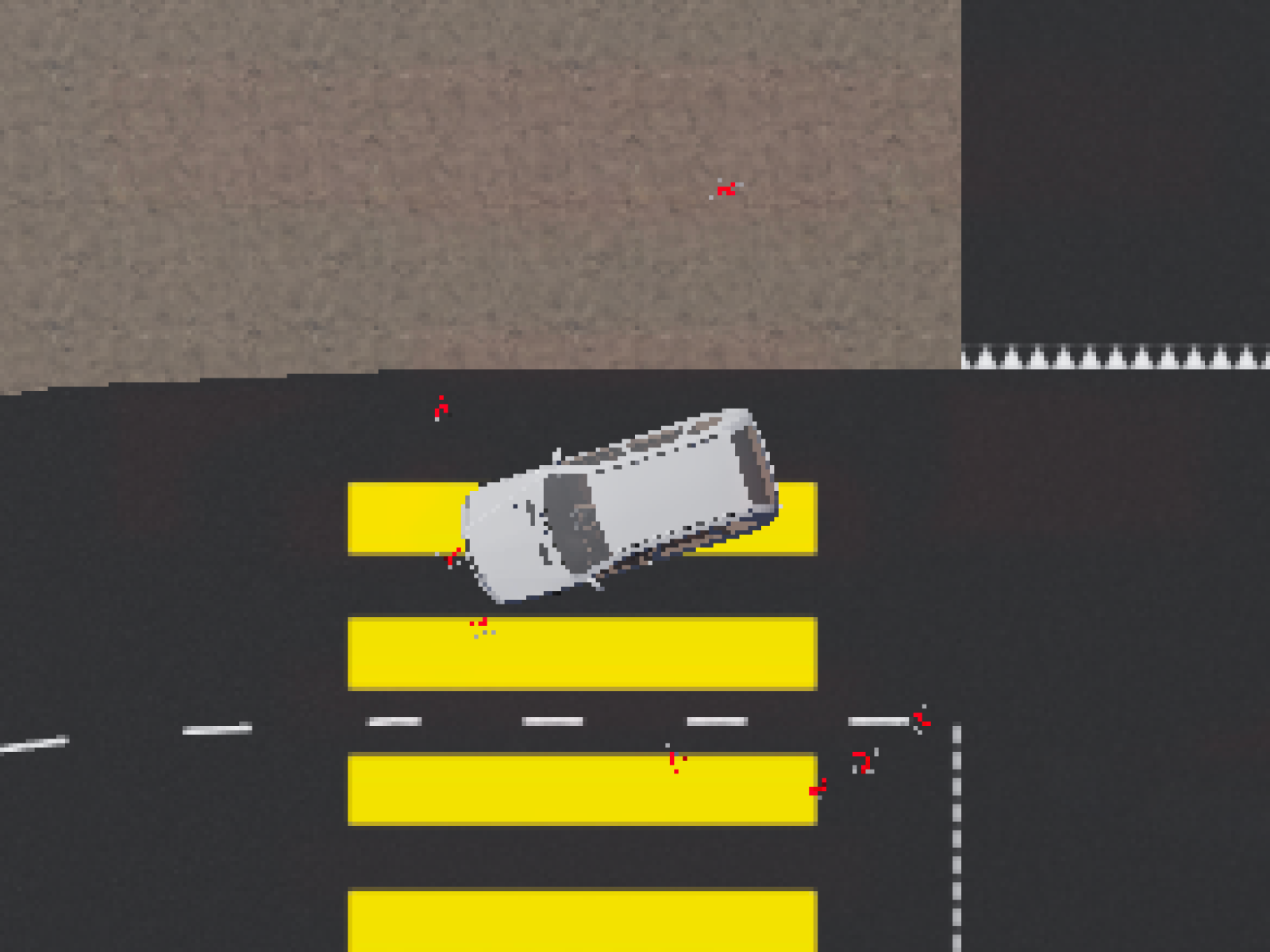}
\caption{Car 2 after target split}
\label{fig:app_case1_2e}
\end{subfigure}
\hfill
\begin{subfigure}{.32\linewidth}
    \includegraphics[width=\linewidth]{figures/c1s2_1.png}
    \caption{Car 3 after target split}
\label{fig:app_case1_2f}
\end{subfigure}

\begin{subfigure}{.32\linewidth}
    \includegraphics[width=\linewidth]{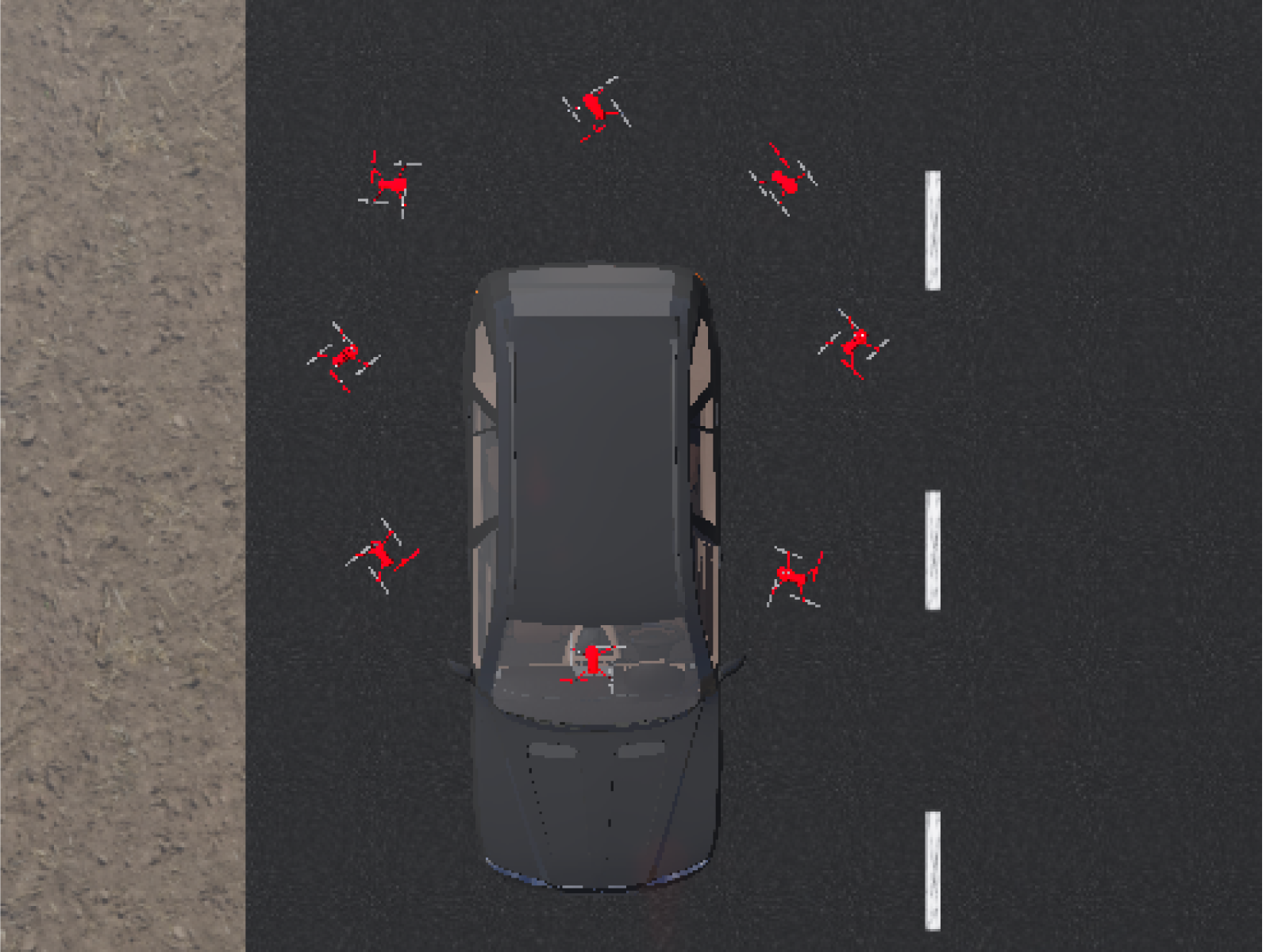}
\caption{Corrected formation on car 1}
\label{fig:app_case1_2g}
\end{subfigure}
\hfill
\begin{subfigure}{.32\linewidth}
    \includegraphics[width=\linewidth]{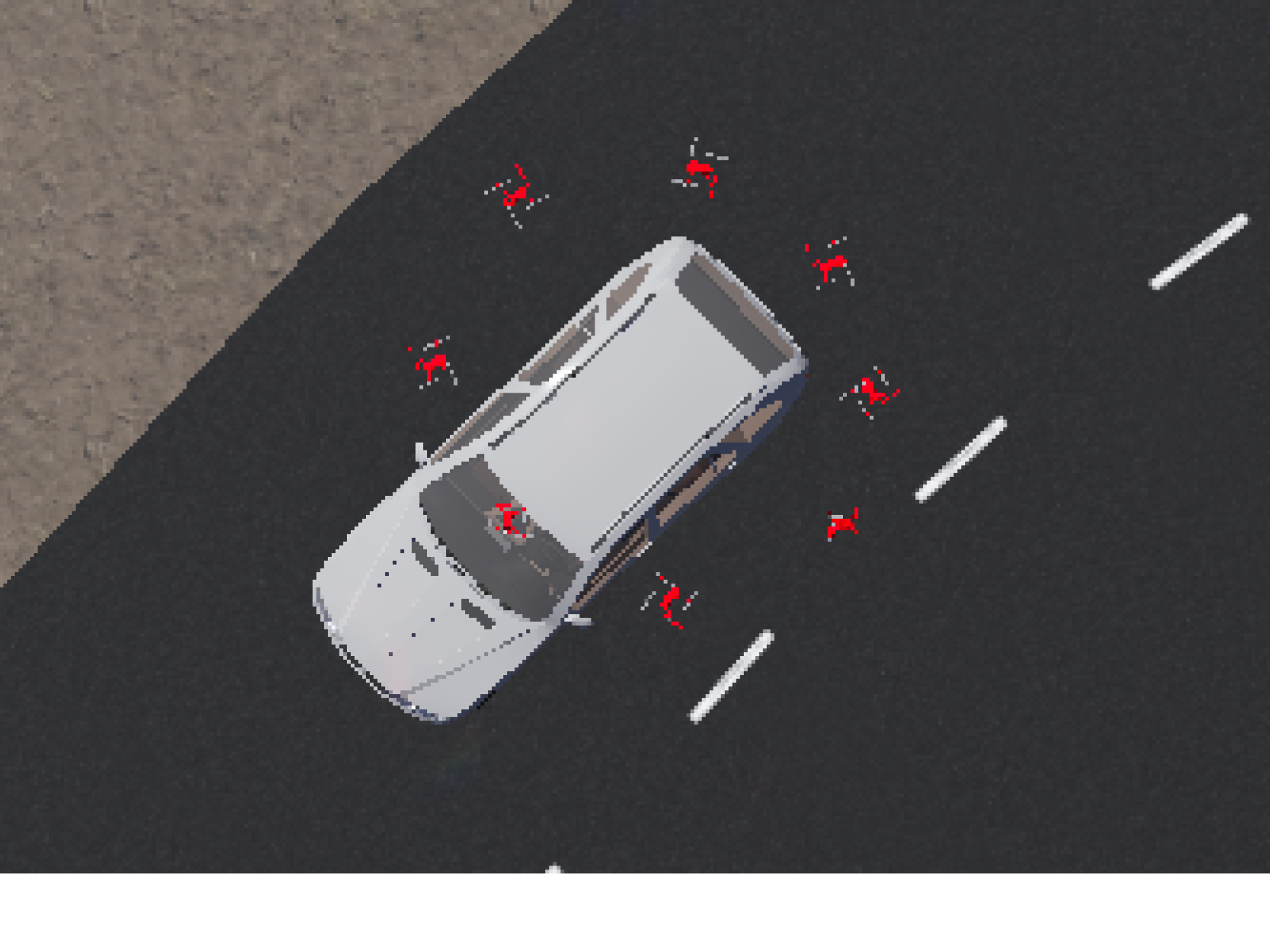}
\caption{Corrected formation on car 2}
\label{fig:app_case1_2h}
\end{subfigure}
\hfill
\begin{subfigure}{.32\linewidth}

\includegraphics[width=\linewidth]{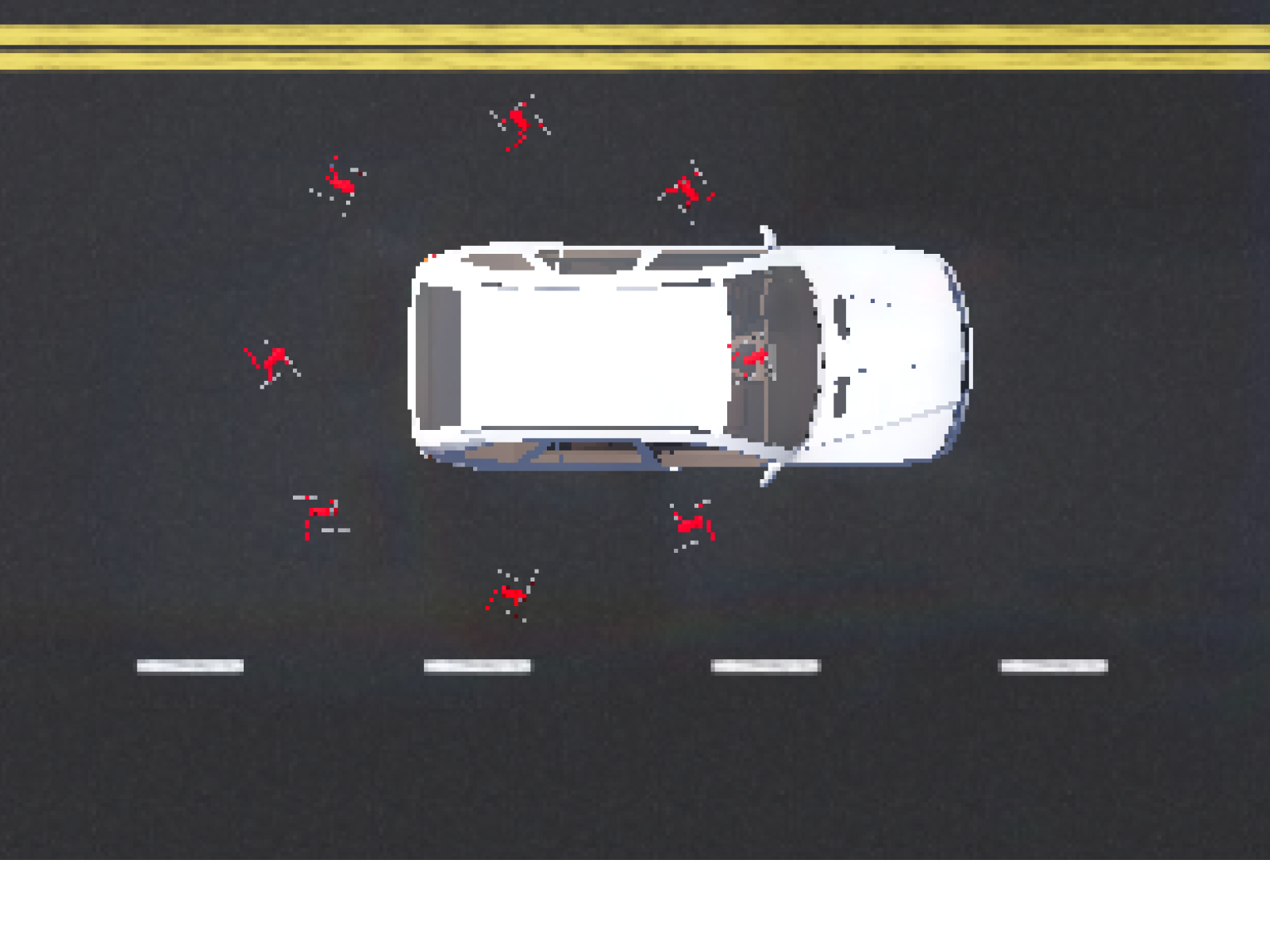}
\caption{Corrected formation on car 3}
\label{fig:app_case1_2i}
\end{subfigure}

\begin{subfigure}{.32\linewidth}
    \includegraphics[width=\linewidth]{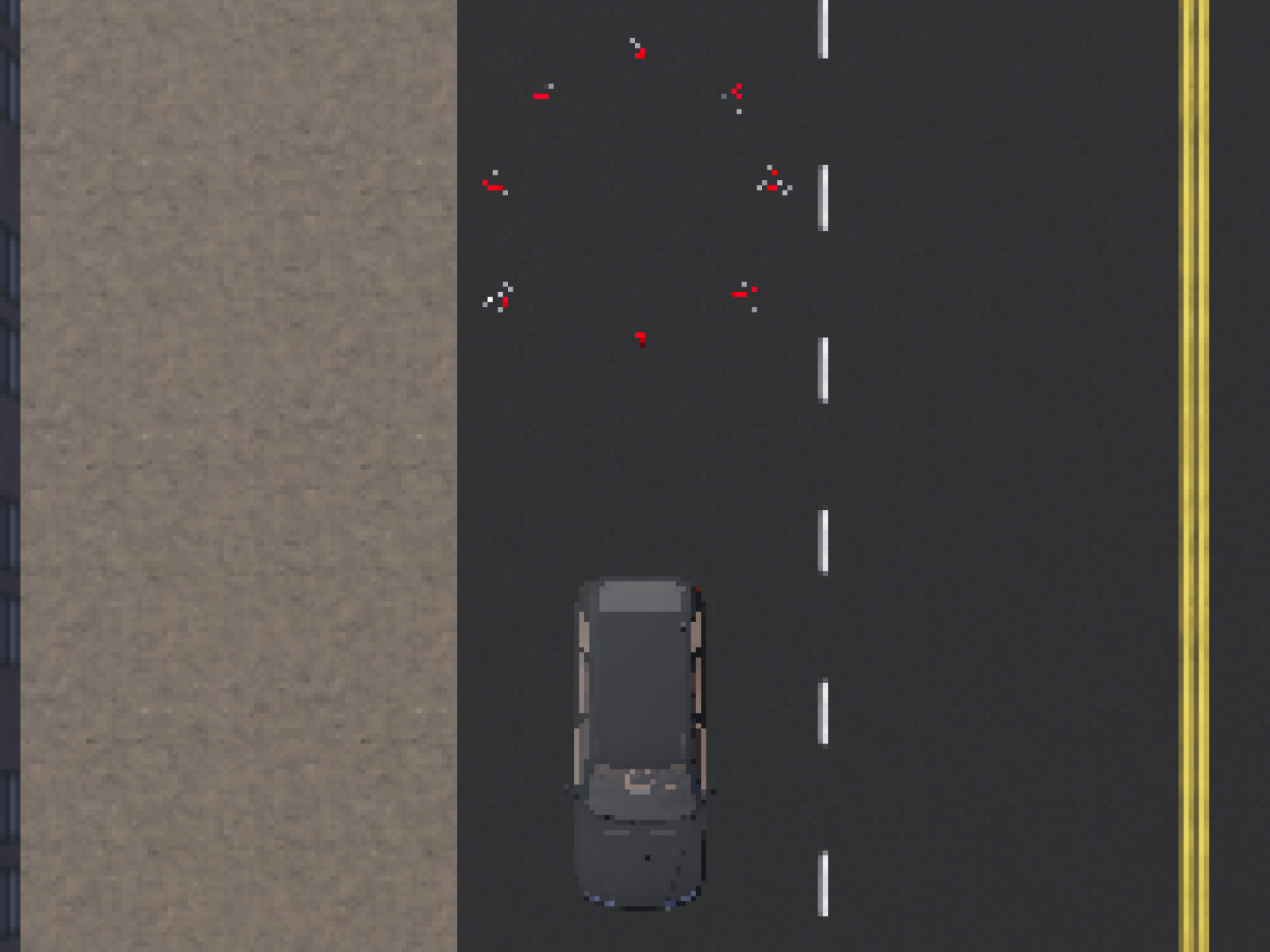}
\caption{Reformation on car 1}
\label{fig:app_case1_2j}
\end{subfigure}
\hfill
\begin{subfigure}{.32\linewidth}
    \includegraphics[width=\linewidth]{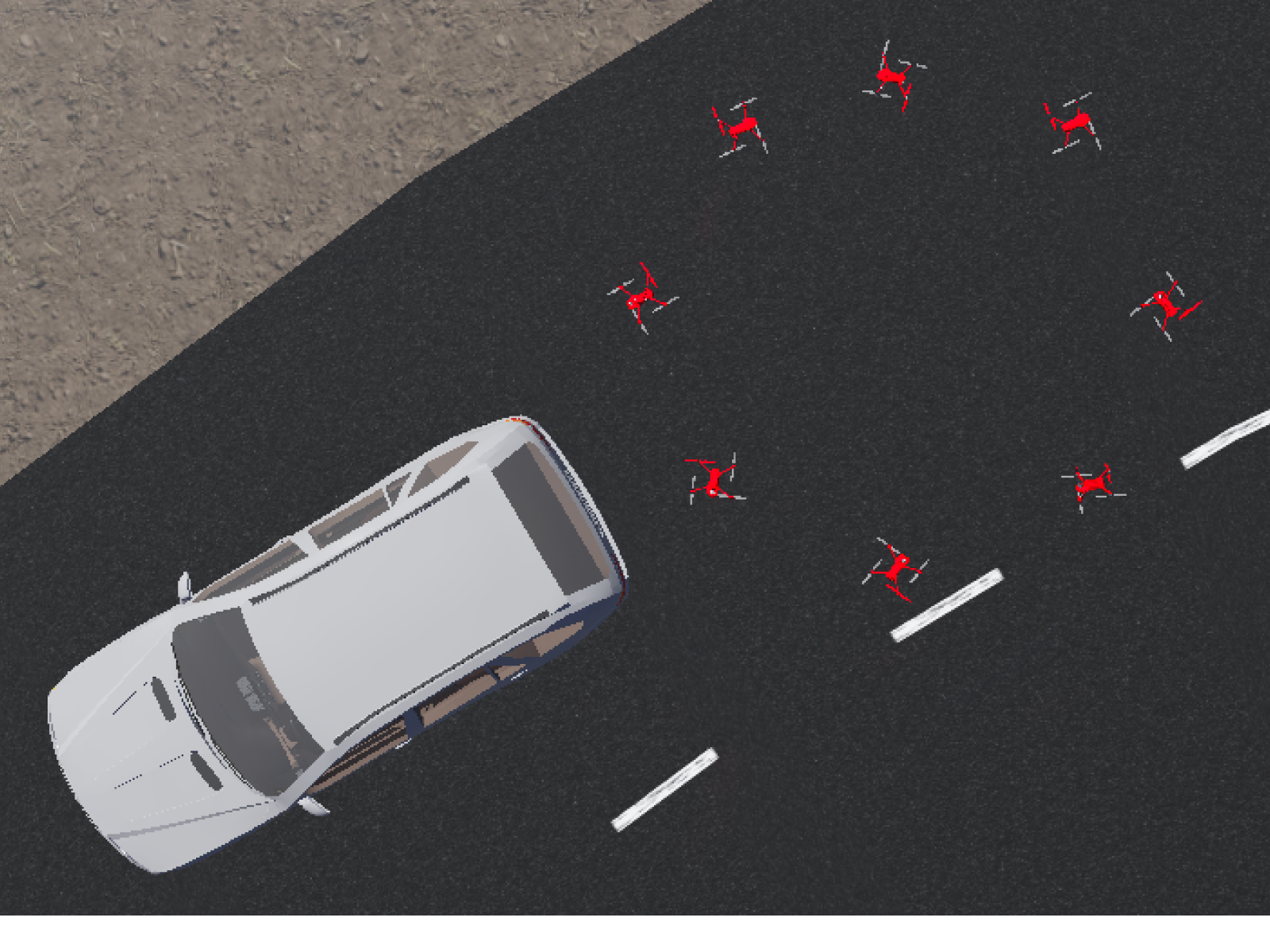}
\caption{Reformation on car 2}
\label{fig:app_case1_2k}
\end{subfigure}
\hfill
\begin{subfigure}{.32\linewidth}

\includegraphics[width=\linewidth]{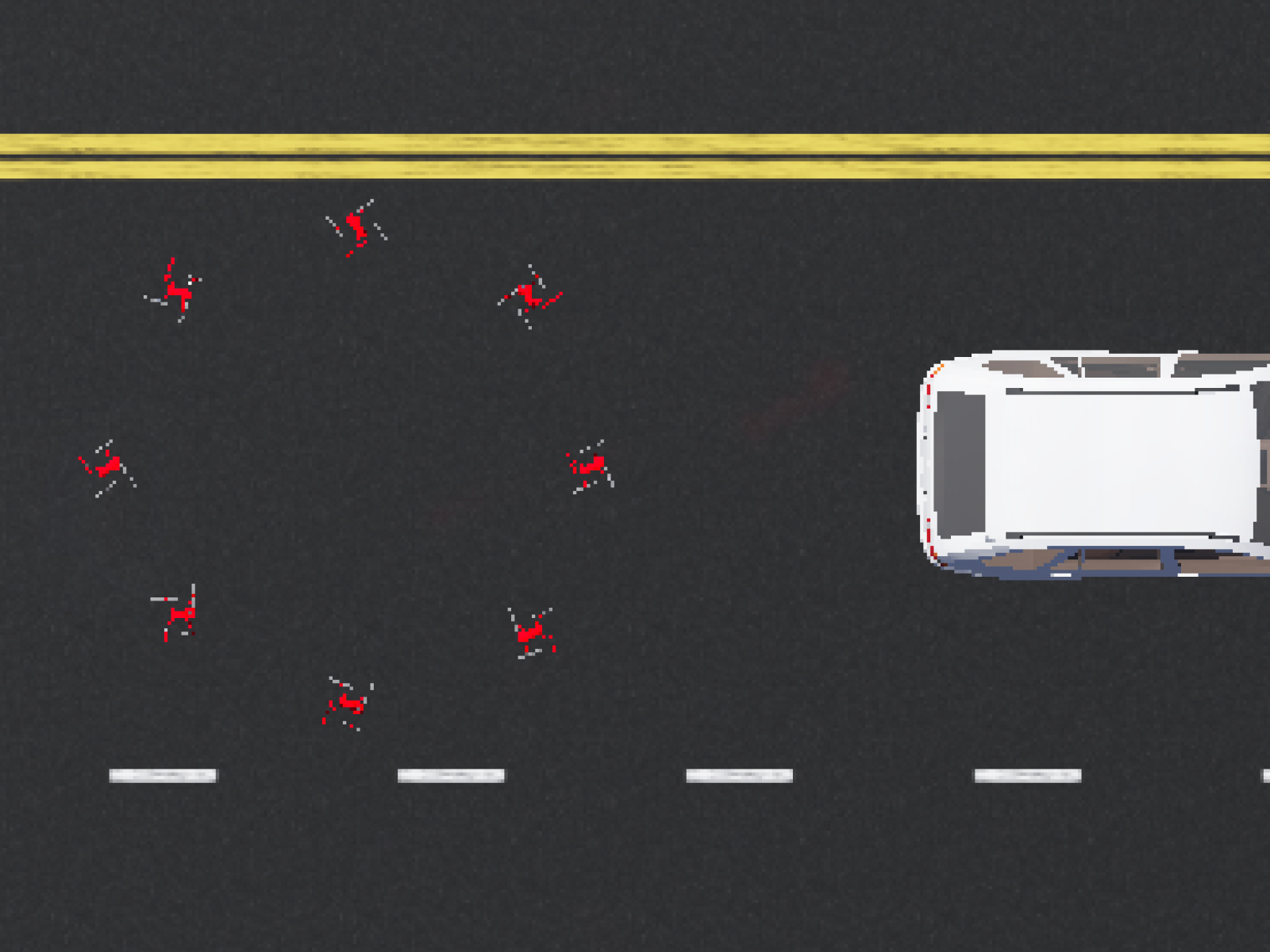}
\caption{Reformation on car 3}
\label{fig:app_case1_2l}
\end{subfigure}
\caption{In \Cref{fig:app_case1_2a}, we have 24 drones scattered around the target cars. The user then commands the drones to form a circle around the cars with LLM, as shown in \Cref{fig:app_case1_2b}. As the cars move toward the road intersection, they diverge to different directions. The inner-layer control algorithm has the group of drones diverted to track different cars, as shown in \Cref{fig:app_case1_2c}. However, the low-level control heuristic does not organize the car to maintain the formation exactly during the splitting process. Lastly, the LLM auto-verification recenters each drone on the target and rebalances the group for each target, as shown in \Cref{fig:app_case1_2d}, \Cref{fig:app_case1_2e}, and \Cref{fig:app_case1_2f}. The user then issues additional commands to not tracking, as shown in \Cref{fig:app_case1_2j}, \Cref{fig:app_case1_2k}, and \Cref{fig:app_case1_2l}.}
\label{fig:app_figure2}
\end{figure*}

\textbf{Police–chasing scenario 3.}
In~\Cref{fig:app_figure3}, the operator issues a command such as \emph{``Three suspects are attempting to escape. Form a circle like a coordinated police dragnet and track all three targets.''} 
The suspect vehicles initially travel together along the road and then diverge in different directions at the intersection. As the targets separate, the system automatically adapts the swarm configuration to maintain coordinated tracking of all three vehicles. At a later time, the vehicles execute U-turns, reverse direction, and return toward the intersection, where they eventually merge again onto a common path. The 24 drones subsequently regroup, and under the LLM supervision of the mid-layer, re-form a single encirclement above the three vehicles.

\begin{figure*}
\begin{subfigure}{.32\linewidth}
    \includegraphics[width=\linewidth]{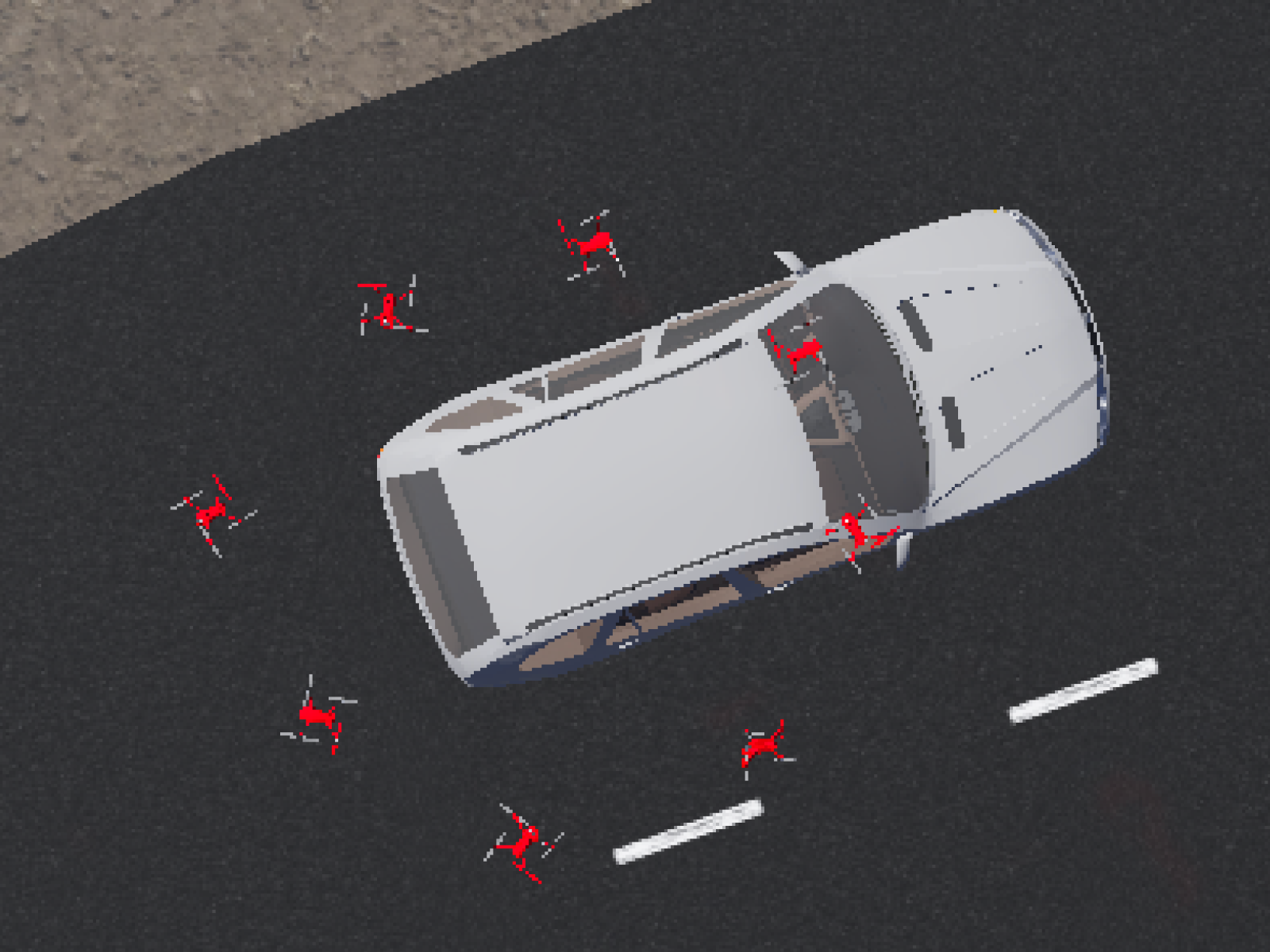}
\caption{Car 2 after U-turn}
\label{fig:app_case1_3a}
\end{subfigure}
\hfill
\begin{subfigure}{.32\linewidth}
    \includegraphics[width=\linewidth]{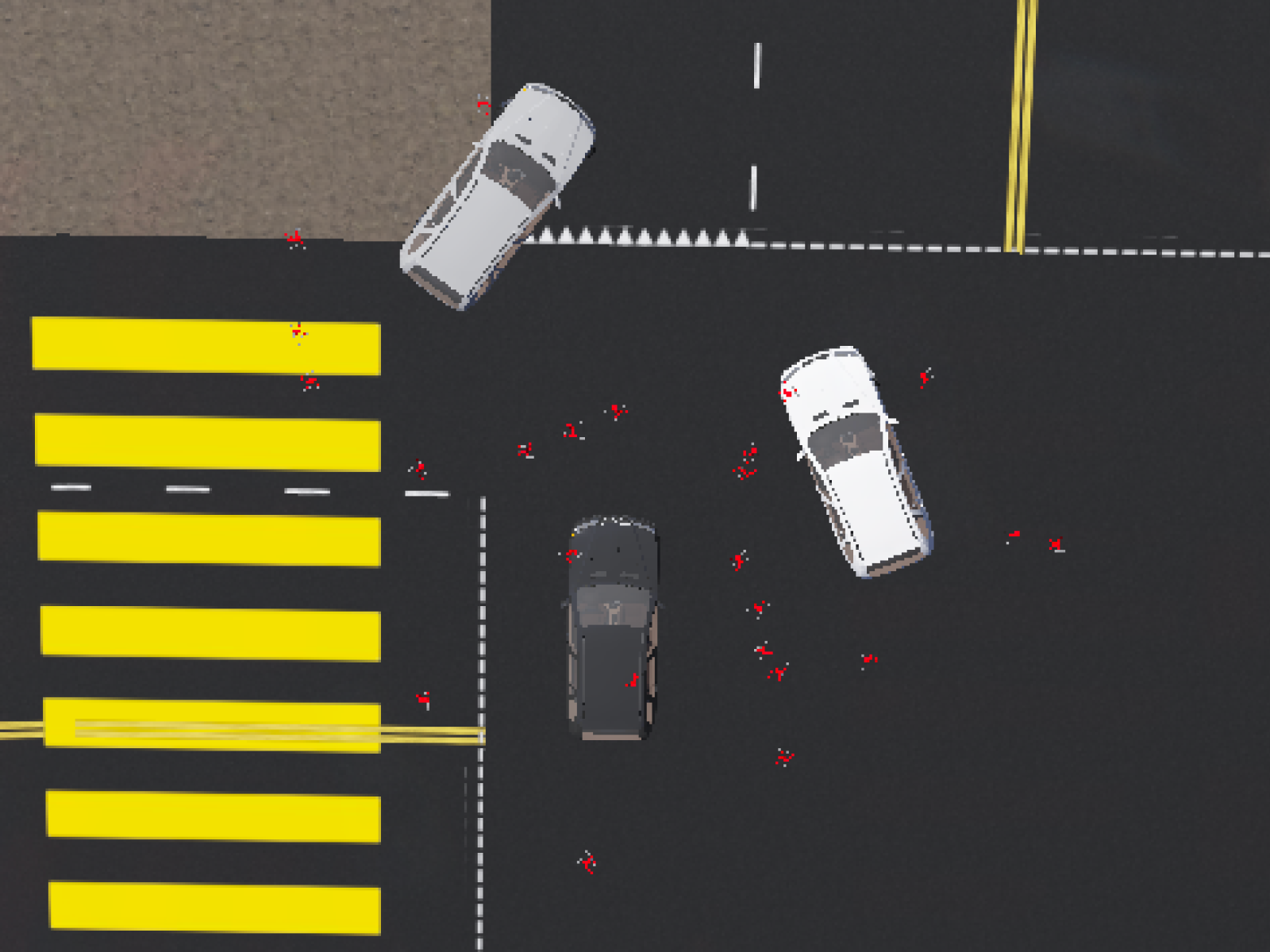}
\caption{Cars merge at the intersection}
\label{fig:app_case1_3b}
\end{subfigure}
\hfill
\begin{subfigure}{.32\linewidth}
    \includegraphics[width=\linewidth]{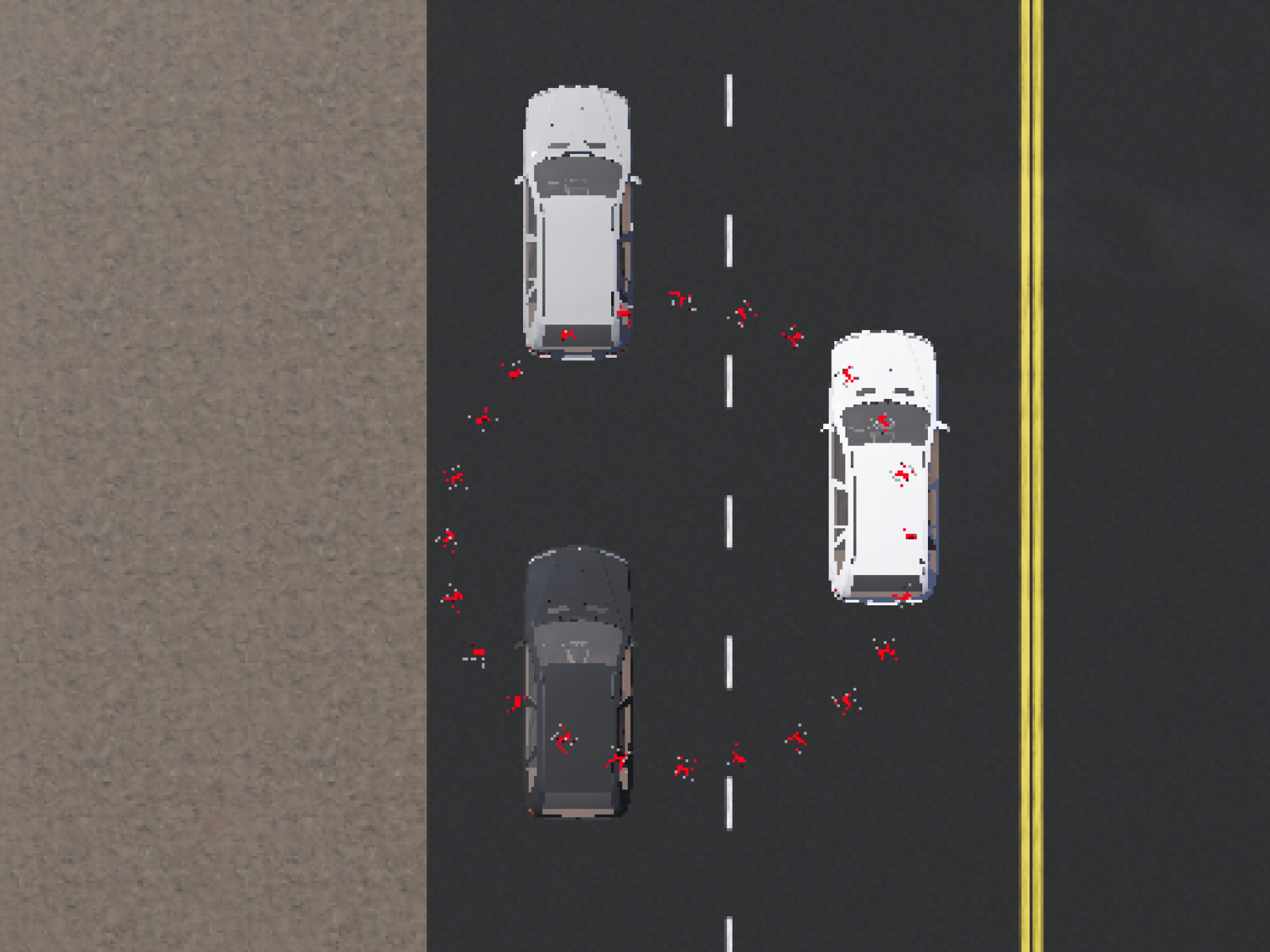}
    \caption{Drone merge after the intersection}
\label{fig:app_case1_3c}
\end{subfigure}
\caption{In~\Cref{fig:app_case1_3a}, the drones maintain a circular formation while vehicle~2 executes a U-turn. In~\Cref{fig:app_case1_3b}, all vehicles reconverge at the intersection. In~\Cref{fig:app_case1_3c}, the vehicles merge onto a common path, and the drones correspondingly regroup into a single encirclement.
}
\label{fig:app_figure3}
\end{figure*}

\textbf{Search-and-rescue scenario 1.} 
This scenario has been introduced and discussed in~\Cref{sec:sim_sar} with~\Cref{fig:case2}.

\textbf{Search-and-rescue scenario 2.} 
This scenario follows the same procedure as Scenario~1 (see \Cref{sec:sim_sar}), but places the target at a different location in the environment (see ~\Cref{fig:app_figure4}).
\Cref{fig:app_case2_2a} shows the initial deployment of the drones,
while \Cref{fig:app_case2_2b} illustrates the ongoing search process with distributed location exploration.
Upon detection, the drones transition to a three-dimensional cube formation centered at the target, providing volumetric coverage and better visibility, as shown in \Cref{fig:app_case2_2c}.

\begin{figure*}
\begin{subfigure}{.32\linewidth}
    \includegraphics[width=\linewidth]{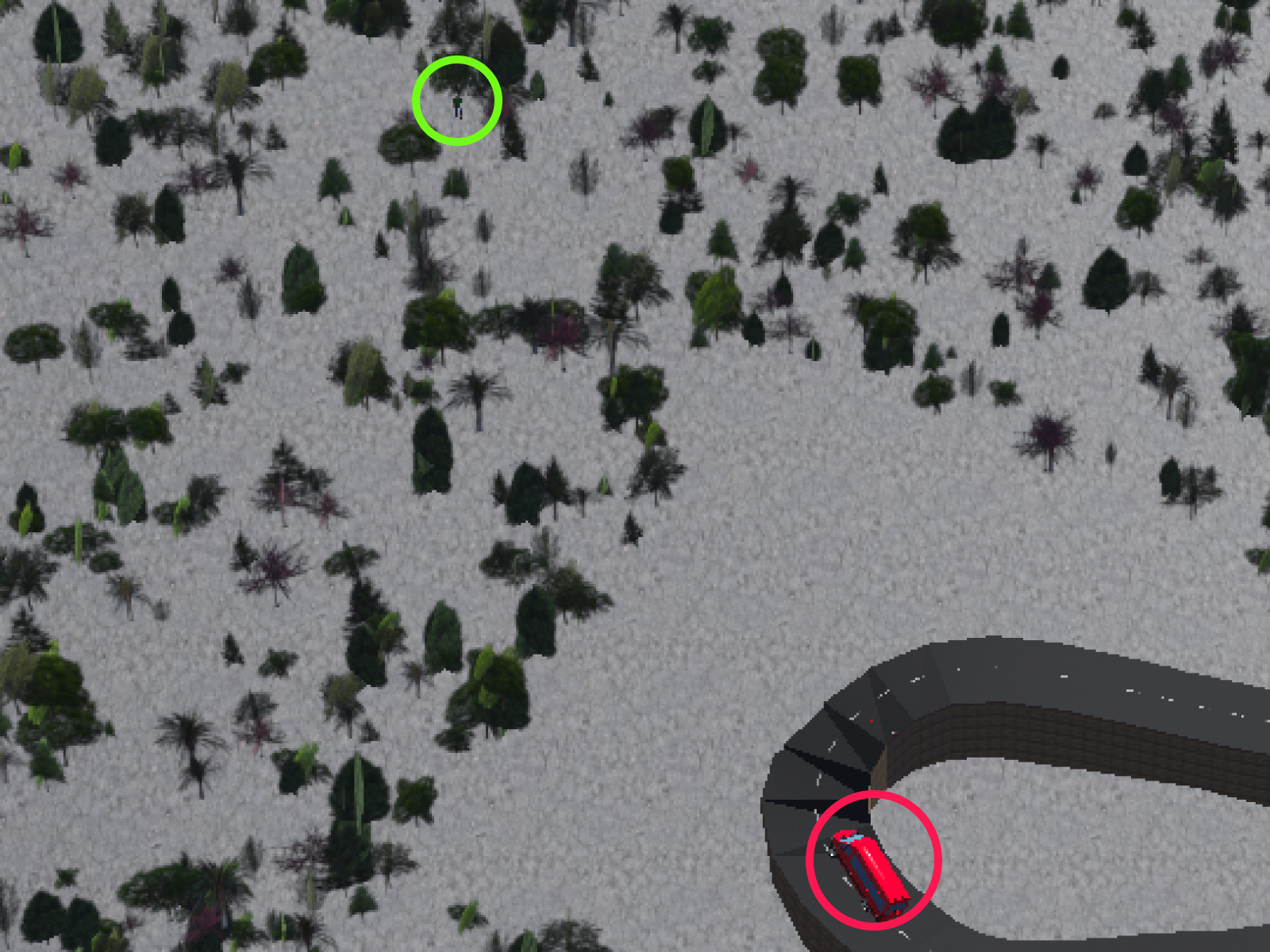}
\caption{Scenario illustration}
\label{fig:app_case2_2a}
\end{subfigure}
\hfill
\begin{subfigure}{.32\linewidth}
    \includegraphics[width=\linewidth]{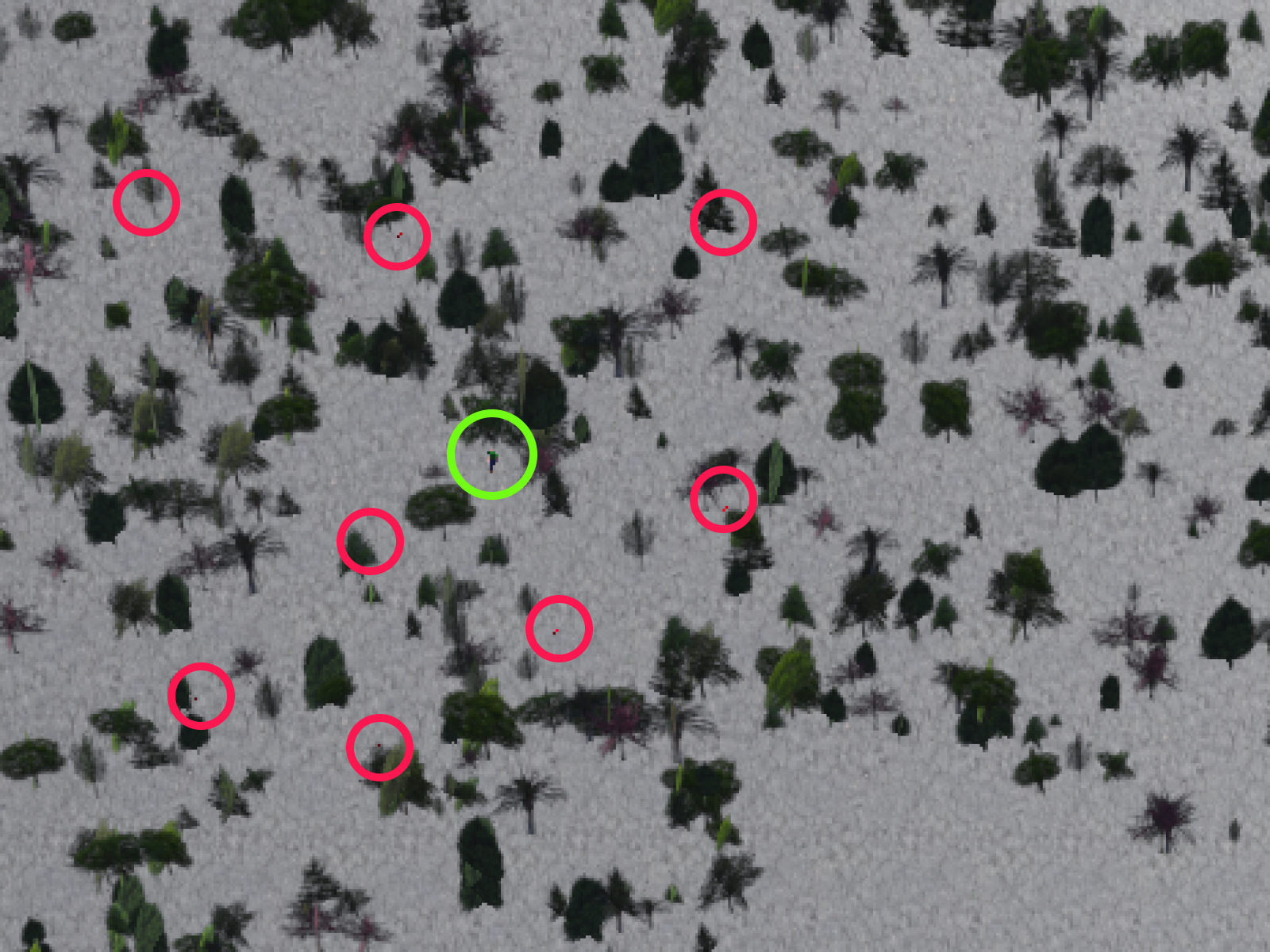}
\caption{Ongoing search process}
\label{fig:app_case2_2b}
\end{subfigure}
\hfill
\begin{subfigure}{.32\linewidth}
    \includegraphics[width=\linewidth]{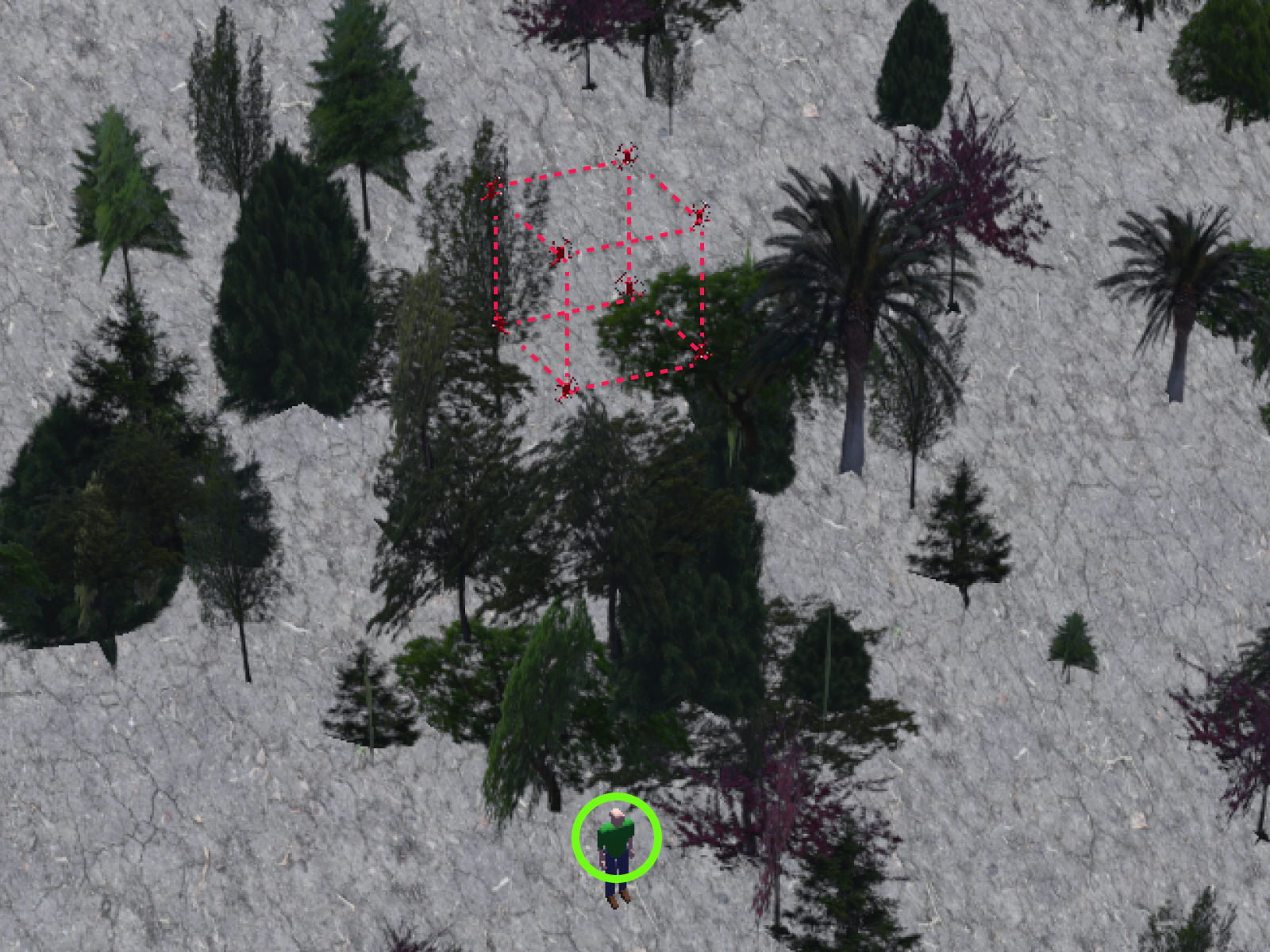}
    \caption{Three-dimensional enclosure}
\label{fig:app_case2_2c}
\end{subfigure}
\caption{Search-and-rescue Scenario~2 with the target located at a different position. In this scenario, eight drones are deployed.
(a) Initial deployment of the rescue vehicle and the target.
(b) Ongoing search process with distributed waypoint exploration.
(c) Upon detection, the drones form a three-dimensional cube formation around the target, providing volumetric coverage.
}
\label{fig:app_figure4}
\end{figure*}

\textbf{Search-and-rescue scenario 3.}
This scenario follows the same overall procedure, but incorporates automatic correction by the LLM supervisor and considers an enlarged search region (see~\Cref{fig:app_figure5}).
\Cref{fig:app_case2_3a} shows the initial deployment of the drones,
while \Cref{fig:app_case2_3b} illustrates the search process within the initial search area.
As the search region is expanded through automatic correction, the drones adapt their exploration accordingly, as shown in \Cref{fig:app_case2_3c}.

\begin{figure*}
\begin{subfigure}{.32\linewidth}
    \includegraphics[width=\linewidth]{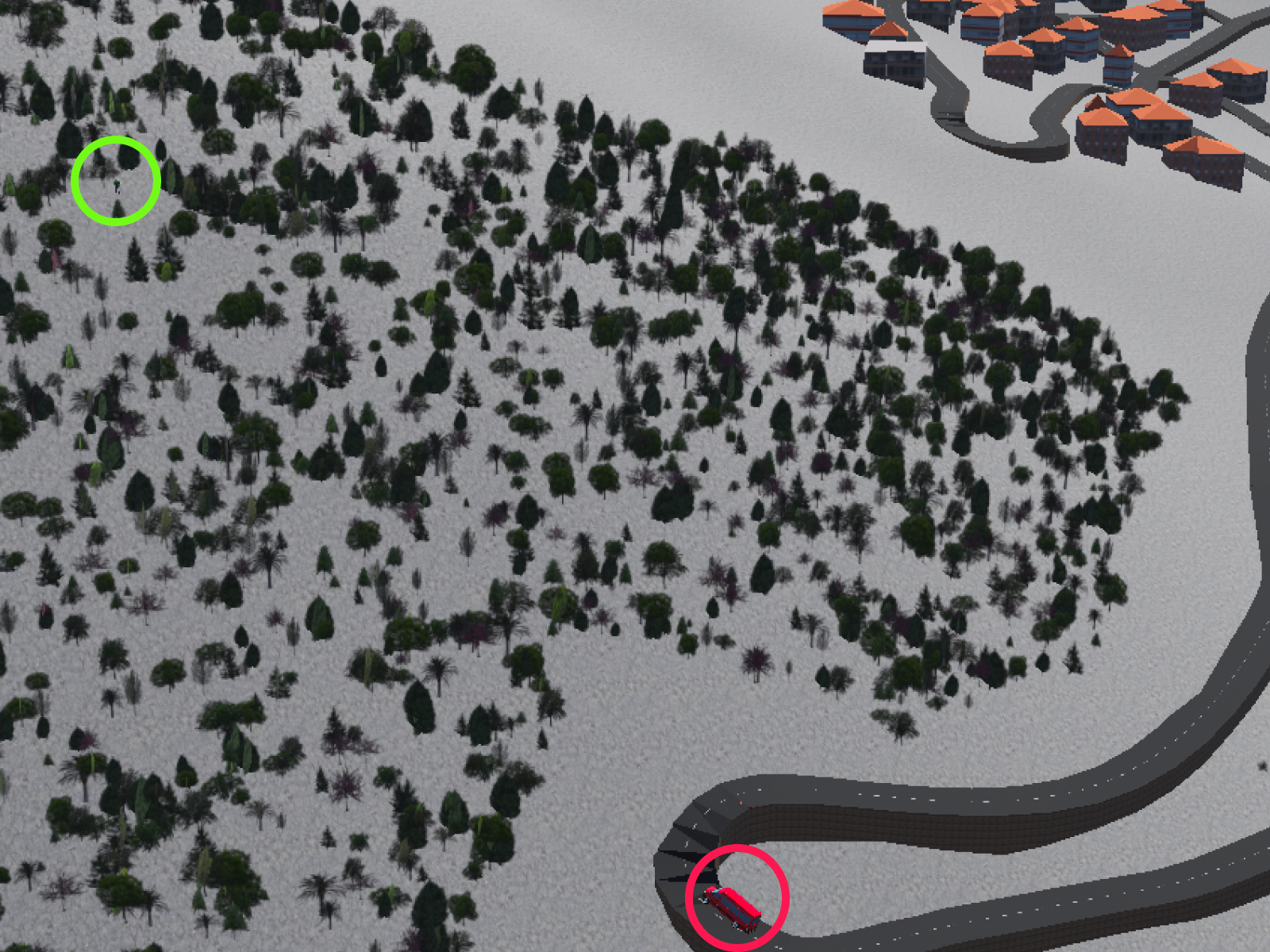}
\caption{Scenario illustration}
\label{fig:app_case2_3a}
\end{subfigure}
\hfill
\begin{subfigure}{.32\linewidth}
    \includegraphics[width=\linewidth]{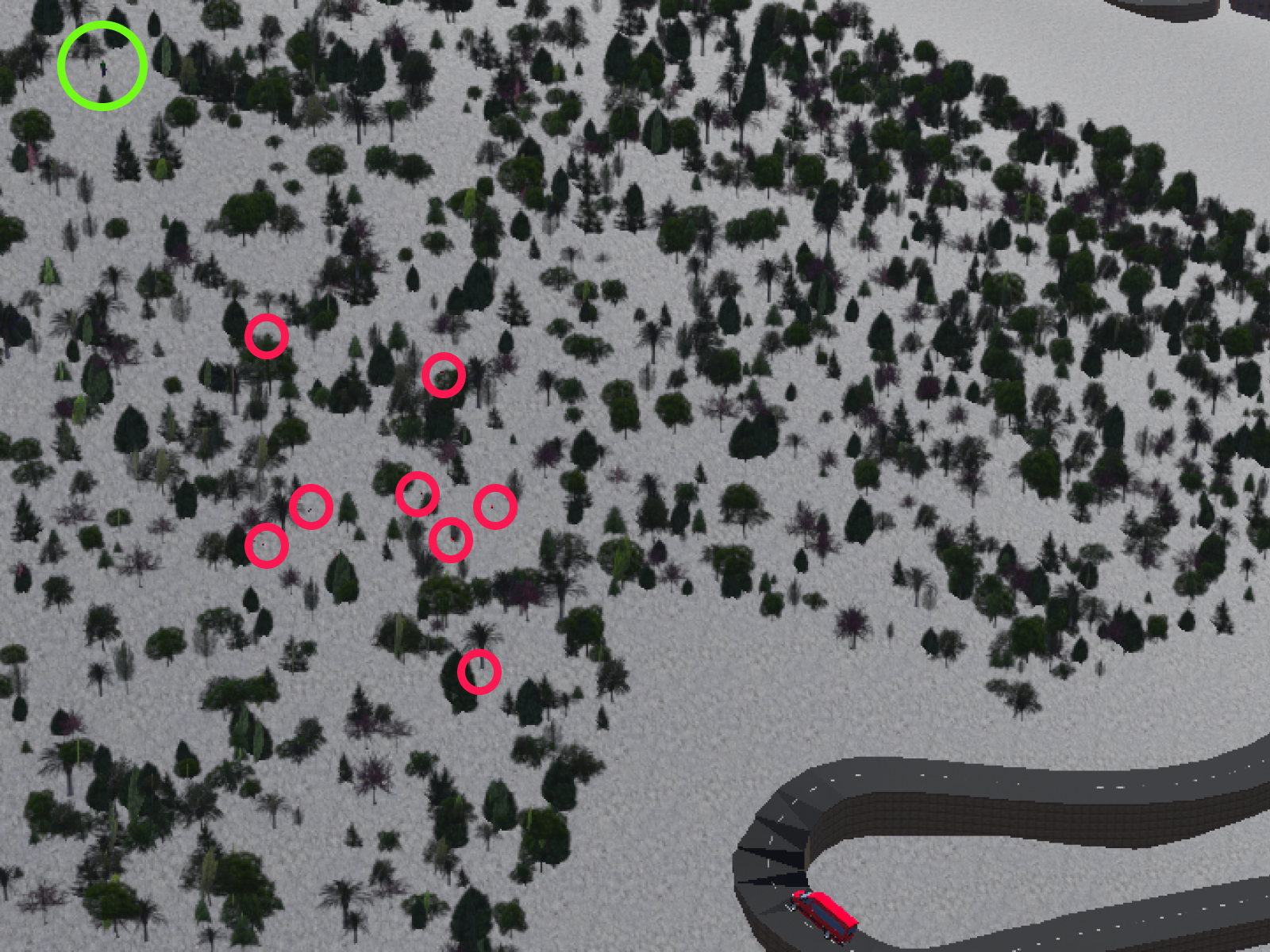}
\caption{Search with distributed waypoints}
\label{fig:app_case2_3b}
\end{subfigure}
\hfill
\begin{subfigure}{.32\linewidth}
    \includegraphics[width=\linewidth]{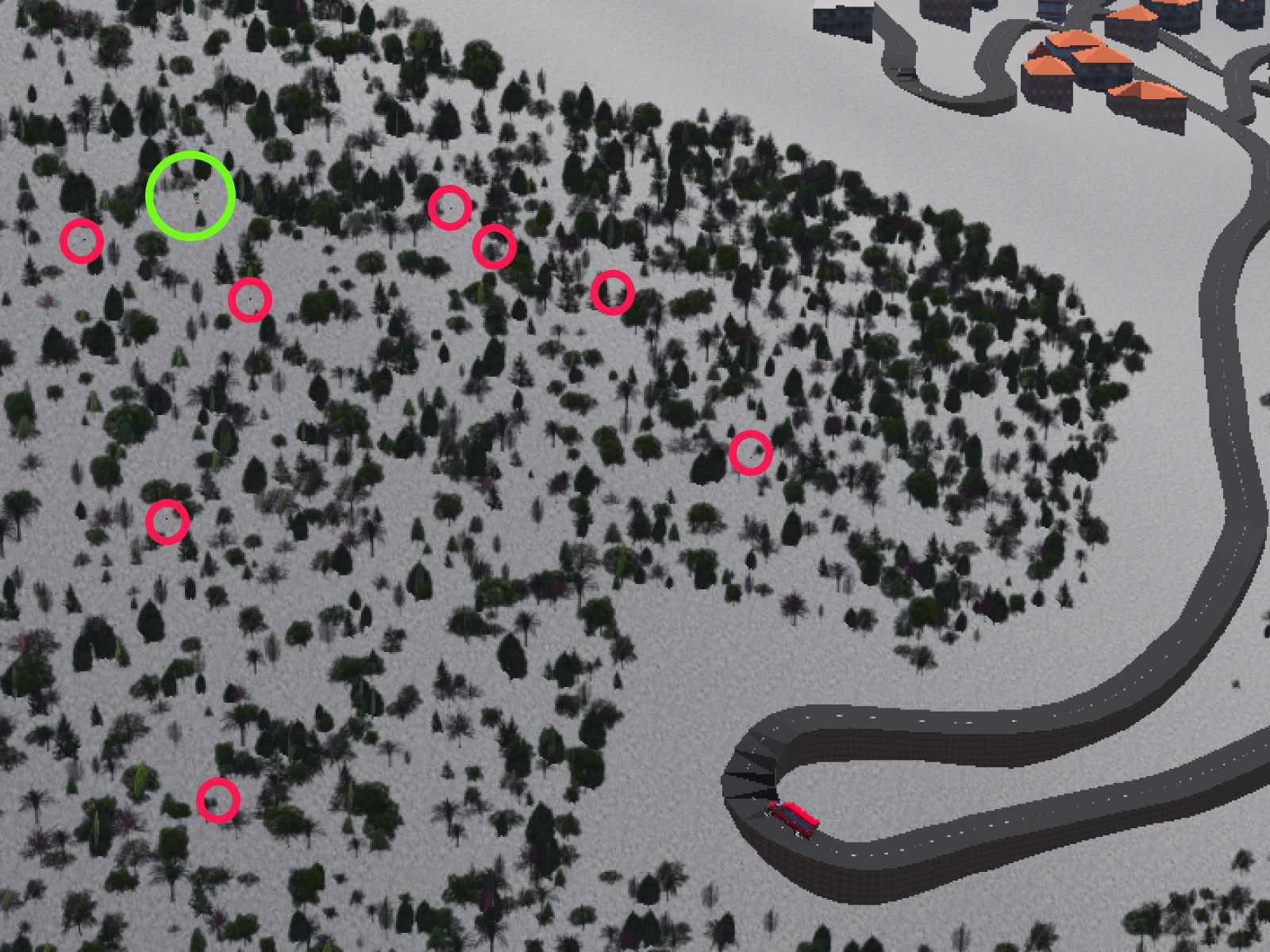}
    \caption{Expanded search region}
\label{fig:app_case2_3c}
\end{subfigure}
\caption{Search-and-rescue Scenario~3 with the target located at a farther distance.
In this scenario, eight drones are deployed.
(a) Initial deployment of the rescue vehicle and the target.
(b) Ongoing search process with distributed waypoint exploration within the initial search region.
(c) Search region expansion through automatic correction, with the drones adapting their exploration accordingly.
}
\label{fig:app_figure5}
\end{figure*}

\section{Prompts for Users and LLM-supervisor}
This section summarizes the prompts used in the scenarios described above.

\begin{tcolorbox}[
    breakable, 
    enhanced,  
    colback=gray!15!white,      
    colframe=gray!50!black,    
    title={Motion Descriptor},        
    center title,              
    rounded corners,           
]

You are an instruction parser for a multi-agent control system.

Rules:

1. "mode":
\begin{itemize}[leftmargin=0pt, itemsep=0pt, topsep=0pt, partopsep=0pt, parsep=0pt]
\item[-] "track" if the user mentions tracking, following, escorting, or maintaining formation around moving targets.
\item[-] "search" if the user mentions searching, scanning, exploring, or patrolling.
\item[-] "stationary" otherwise.
\end{itemize}

2. "tracking":
\begin{itemize}[leftmargin=0pt, itemsep=0pt, topsep=0pt, partopsep=0pt, parsep=0pt]
\item[-] true if mode is "track".
\item[-] false otherwise.
\end{itemize}

3. Each "group" corresponds to one target mentioned by the user (e.g. car1, car2).

4. "formation":
\begin{itemize}[leftmargin=0pt, itemsep=0pt, topsep=0pt, partopsep=0pt, parsep=0pt]
\item[-] Use the shape mentioned by the user.
\item[-] Use "grid" if not specified.
\end{itemize}

5. "even\_split":
\begin{itemize}[leftmargin=0pt, itemsep=0pt, topsep=0pt, partopsep=0pt, parsep=0pt]
\item[-] true only if the user explicitly says "evenly", "equally", or "balanced".
\item[-] false otherwise.
\end{itemize}

6. "spacing":
\begin{itemize}[leftmargin=0pt, itemsep=0pt, topsep=0pt, partopsep=0pt, parsep=0pt]
\item[-] Use the number if provided.
\item[-] Default to 2 if not specified.
\end{itemize}

Output JSON only. No extra text.
\end{tcolorbox}

\begin{tcolorbox}[
    breakable, 
    enhanced,  
    colback=gray!15!white,      
    colframe=gray!50!black,    
    title={Formation Instruction},        
    center title,              
    rounded corners,           
]

You are a formation geometry generator.

Your task is to output 3D formation offsets (x,y,z) for N drones,
relative to the formation center at (0,0) and at a height offset z above the ground.

Input:
\begin{itemize}[leftmargin=0pt, itemsep=0pt, topsep=0pt, partopsep=0pt, parsep=0pt]
  \item[-] Formation shape (e.g., circle, square, grid, cross, spiral).
  \item[-] Number of drones N.
  \item[-] Optional spacing (meters), default is 1.0 meter.
  \item[-] Optional height offset z (meters above ground), default is 0.
\end{itemize}

Rules:
\begin{itemize}[leftmargin=0pt, itemsep=0pt, topsep=0pt, partopsep=0pt, parsep=0pt]
\item[-] Output exactly N points.
\item[-] Points are relative offsets from the formation center (0,0).
\item[-] z represents the vertical offset above the ground.
\item[-] Use simple geometry consistent with the requested shape.
\item[-] If spacing is given, neighboring drones should be approximately spacing meters apart.
\item[-] Do not place multiple drones at the same location.
\end{itemize}

Output format (CSV only, no extra text):

id,x,y,z\\
0,x0,y0,z0\\
1,x1,y1,z1\\
...

Only output the table.
\end{tcolorbox}

\begin{tcolorbox}[
    breakable, 
    enhanced,  
    colback=gray!15!white,      
    colframe=gray!50!black,    
    title={Auto-correction},        
    center title,              
    rounded corners,           
]

You are a multi-agent formation checker.

Goal:
Decide if the current drone formation meets the user’s intent or if it should be revised.

INPUT:\\
1) USER INSTRUCTION: {USER\_TEXT}
\begin{itemize}[leftmargin=0pt, itemsep=0pt, topsep=0pt, partopsep=0pt, parsep=0pt]
\item[-] Describes desired formation/behavior.
\end{itemize}

2) CURRENT FORMATION (CSV): {FEEDBACK\_CSV}

\begin{itemize}[leftmargin=0pt, itemsep=0pt, topsep=0pt, partopsep=0pt, parsep=0pt]
\item[-] One or multiple groups.
\item[-] Groups separated by lines starting with "\# group" or "---".
\item[-] Coordinates are 3D relative positions: id,x,y,z
\end{itemize}

CHECK:
Compare the user instruction with the current coordinates.
Consider:
\begin{itemize}[leftmargin=0pt, itemsep=0pt, topsep=0pt, partopsep=0pt, parsep=0pt]
\item[-] Formation shape (grid, circle, line, cluster)
\item[-] Group separation (if multiple targets implied)
\item[-] Rough symmetry/alignment
\end{itemize}
Small numeric deviations should NOT trigger revision.

OUTPUT (STRICT JSON ONLY):\\
\{\\
  "feedback": true/false,\\
  "reason": "Short explanation (1-2 sentences)"\\
\}

Examples:\\
{"feedback": false, "reason": "Grid matches user request."}\\
{"feedback": true, "reason": "One cluster but multiple groups requested."}
\end{tcolorbox}

\onecolumn
\section{Proof of \Cref{thm:ISS} and \Cref{cor:LLM_accuracy}}
\label{appendix:proof_main}
In this section, we offer the proof of \Cref{thm:ISS} and \Cref{cor:LLM_accuracy}.
\begin{proof}[Proof of \Cref{thm:ISS}]
\item
\paragraph{Step 1: Closed-loop error dynamics with exogenous targets.}
On $[t_k,t_{k+1})$, the incidence matrix $E_k$ is constant.
Differentiate the edge-error definition \eqref{eq:edge_error_revised2}:
\[
\dot e(t)=\dot z^r(t)-(E_k^\top\otimes I_d)\dot p(t).
\]
Using the stacked dynamics from \Cref{sec:inner_layer_revised2}, namely
$\dot p_a=u_a+d_a$ and $\dot p_b=v_b$, and the stacked neighbor-feedback controller
\eqref{eq:stacked_controller_revised2} (i.e., $u_a=(E_{a,k}\otimes I_d)e$), we obtain
\begin{align*}
\dot e(t)
&=\dot z^r(t)-(E_k^\top\otimes I_d)
\begin{bmatrix}
\dot p_a(t)\\ \dot p_b(t)
\end{bmatrix}
\\
&=\dot z^r(t)-(E_k^\top\otimes I_d)
\begin{bmatrix}
(E_{a,k}\otimes I_d)e(t)+d_a(t)\\ v_b(t)
\end{bmatrix}
\\
&=\dot z^r(t)-\big(E_k^\top S_a E_k\otimes I_d\big)e(t)-(E_k^\top\otimes I_d)
\begin{bmatrix}d_a(t)\\ v_b(t)\end{bmatrix}.
\end{align*}

To emphasize how exogenous target motion enters, note that by construction
$z^r(t)=(E_k^\top\otimes I_d)p^r(t)$ with $p^r(t)=[p_a^r(t);p_b(t)]$, hence
\[
\dot z^r(t)=(E_k^\top\otimes I_d)\begin{bmatrix}\dot p_a^r(t)\\ v_b(t)\end{bmatrix}.
\]
Substituting into $w(t)$ cancels the target-velocity term $v_b(t)$, yielding the simplified input
expression \eqref{eq:w_simplified}:
\[
w(t)=(E_{a,k}^\top\otimes I_d)\big(\dot p_a^r(t)-d_a(t)\big).
\]
Under \Cref{assumption:general}(A1), $\dot z^r(t)=0$ on $[t_k,t_{k+1})$, so \eqref{eq:compatibility} holds and
\[
w(t)=\dot z^r(t)-(E_k^\top\otimes I_d)d(t)=-(E_k^\top\otimes I_d)d(t).
\]
By \eqref{eq:compatibility}, this is equivalent to the simplified expression above, and in particular
$w(t)\in \mathrm{range}(E_{a,k}^\top\otimes I_d)$.

Thus, targets affect the error dynamics only through how the enforced drone reference $p_a^r$ evolves.

\item
\paragraph{Step 2: Invariance of the admissible subspace.}
From \eqref{eq:e_range}, we have
$e(t)\in \mathrm{range}(E_{a,k}^\top\otimes I_d)$ for all $t\in[t_k,t_{k+1})$.
From \eqref{eq:w_simplified}, we also have
$w(t)\in \mathrm{range}(E_{a,k}^\top\otimes I_d)$ for all $t$.
Therefore, the dynamics \eqref{eq:error_dyn} evolve on the invariant subspace
$\mathrm{range}(E_{a,k}^\top\otimes I_d)$.

\item
\paragraph{Step 3: Strict positivity of $L_{e,k}^a$ on $\mathrm{range}(E_{a,k}^\top)$.}
The matrix $L_{e,k}^a=E_{a,k}^\top E_{a,k}$ is symmetric positive semidefinite, with
\[
\ker(L_{e,k}^a)=\ker(E_{a,k}).
\]
Moreover, for any matrix $M$, $\mathrm{range}(M^\top)=(\ker(M))^\perp$.
Applying this with $M=E_{a,k}$ gives
\[
\mathrm{range}(E_{a,k}^\top)=(\ker(E_{a,k}))^\perp=(\ker(L_{e,k}^a))^\perp.
\]
Hence, the restriction of $L_{e,k}^a$ to $\mathrm{range}(E_{a,k}^\top)$ is positive definite.

Let the eigenvalues of $L_{e,k}^a$ be
\[
0=\mu_1=\cdots=\mu_q < \mu_{q+1}\le\cdots\le \mu_{m_k},
\]
where $q=\dim\ker(L_{e,k}^a)$.
Then, by definition, the smallest strictly positive eigenvalue is $\mu_{q+1}=\lambda_{\min}^+(L_{e,k}^a)=:\lambda_k$.
Consequently, for all $\xi\in \mathrm{range}(E_{a,k}^\top)$,
\begin{equation}
\xi^\top L_{e,k}^a\,\xi \;\ge\; \lambda_k \|\xi\|^2. \label{eq:rayleigh_k}
\end{equation}
Finally, since $E_{a,k}^\top E_{a,k}$ and $E_{a,k}E_{a,k}^\top$ share identical nonzero eigenvalues (squares of the nonzero
singular values of $E_{a,k}$), \eqref{eq:lambda_k_def} holds.

\item
\paragraph{Step 4: Semigroup contraction and variation-of-constants.}
Because $L_{e,k}^a\otimes I_d$ is symmetric, it is diagonalizable with nonnegative real eigenvalues.
On the invariant subspace $\mathrm{range}(E_{a,k}^\top\otimes I_d)$, the Rayleigh bound \eqref{eq:rayleigh_k} implies that
the semigroup $e^{-(L_{e,k}^a\otimes I_d)t}$ contracts at rate at least $\lambda_k$:
\begin{equation}
\|e^{-(L_{e,k}^a\otimes I_d)t}x\|\le e^{-\lambda_k t}\|x\|,
\qquad \forall x\in\mathrm{range}(E_{a,k}^\top\otimes I_d),\ \forall t\ge 0. \label{eq:semigroup_bound_k}
\end{equation}
Applying the variation-of-constants formula to \eqref{eq:error_dyn} over $[t_k,t]$ yields
\[
e(t)=e^{-(L_{e,k}^a\otimes I_d)(t-t_k)}e(t_k^+)
+\int_{t_k}^{t} e^{-(L_{e,k}^a\otimes I_d)(t-s)}w(s)\,ds.
\]
Taking norms and using \eqref{eq:semigroup_bound_k} gives \eqref{eq:ISS_bound}.
If $\|w(s)\|\le \bar w$ on $[t_k,t_{k+1})$, then
\[
\int_{t_k}^{t}e^{-\lambda_k(t-s)}\|w(s)\|\,ds
\le \bar w\int_{t_k}^{t}e^{-\lambda_k(t-s)}ds
=\frac{\bar w}{\lambda_k}\big(1-e^{-\lambda_k(t-t_k)}\big),
\]
which implies \eqref{eq:ISS_uniform}.

\item
\paragraph{Step 5: Jump at supervision times.}
If the radius-induced graph does not change across $t_k$, then $E_k$ is unchanged and $p(t)$ is continuous while $z^r$ may jump.
From \eqref{eq:edge_error_revised2},
\[
e(t_k^+)-e(t_k^-)=\big(z^r(t_k^+)-z(t_k)\big)-\big(z^r(t_k^-)-z(t_k)\big)=z^r(t_k^+)-z^r(t_k^-)=\Delta z_k^r,
\]
which yields \eqref{eq:jump}. If the radius-induced graph changes at $t_k$, then $E_k$ changes and the post-update error
is reinitialized in the new edge coordinates as described earlier; this concludes the proof.
\end{proof}

We then offer the proof for \Cref{cor:LLM_accuracy}.

\begin{proof}
For any $t\in[t_k,t_{k+1})$, by triangle inequality,
\[
\|z(t)-z^\star(t)\|\le \|z(t)-z^r(t)\|+\|z^r(t)-z^\star(t)\|=\|e(t)\|+\|z^r(t)-z^\star(t)\|.
\]
By the assumed grounding bound, $\|z^r(t)-z^\star(t)\|\le \varepsilon_z$.
By \eqref{eq:ISS_uniform} with $\|w(t)\|\le \bar w$,
\[
\|e(t_{k+1}^-)\|\le e^{-\lambda_k\Delta}\|e(t_k^+)\|+\frac{\bar w}{\lambda_k}(1-e^{-\lambda_k\Delta}).
\]
Thus
\[
\|z(t_{k+1}^-)-z^\star(t_{k+1}^-)\|
\le e^{-\lambda_k\Delta}\|e(t_k^+)\|+\frac{\bar w}{\lambda_k}(1-e^{-\lambda_k\Delta})+\varepsilon_z.
\]
Requiring the right-hand side to be $\le\delta_z$ and rearranging yields exactly
$\|e(t_k^+)\|\le \eta^{(z)}_{\max}(\delta_z,\Delta,\varepsilon_z)$.
\end{proof}
\section{Horizon-wide error bounds under stochastic LLM verification}
\label{sec:horizon_bounds}

Recall that $z^\star(t)$ denotes the \emph{ground-truth} edge-relative reference consistent with the human intent.
Define the ground-truth edge error
\[
e_{\mathrm{gt}}(t):=z(t)-z^\star(t).
\]
Then, for all $t$,
\begin{equation}
\label{eq:tri_decomp_edge}
\|e_{\mathrm{gt}}(t)\|
\le \|e(t)\|+\|z^r(t)-z^\star(t)\|.
\end{equation}

We extend \Cref{assumption:general} to a long horizon. 
\begin{assumption}
    \label{assumption:long_horizon}
    \begin{enumerate}
    \item[(B1)] \textbf{Relative reference held between checks.}
    \Cref{assumption:general}(A1) holds on all intervals $[t_k,t_{k+1})$.
    \item[(B2)] \textbf{Uniform contraction.} The neighbor graph does not change across any $t_k$ and
    \[
    \|z^r(t_k^+)-z^r(t_k^-)\|\le \bar J_z,\qquad \forall k,
    \]
    for some $\bar J_z<\infty$. Since the incidence matrix $E_k$ and edge Laplacian $L_{e,k}^a$ will stay constant, we can denote $\underline\lambda = \lambda_k$ for all $k$.
    \item[(B3)] \textbf{Bounded specification gap in relative coordinates.}
    There exist $0\le \epsilon_{\mathsf C}\le \epsilon_{\mathsf W}<\infty$ such that, for all $k$,
    \[
    \sup_{t\in[t_k,t_{k+1})}\|z^r(t)-z^\star(t)\|
    \le
    \begin{cases}
    \epsilon_{\mathsf C}, & S_k=\mathsf C,\\
    \epsilon_{\mathsf W}, & S_k=\mathsf W.
    \end{cases}
    \]
\end{enumerate}
\end{assumption}
The above assumption holds when the user's intention does not change in the time horizon, and the environment does not change rapidly. In this case, the formation of the drones will stay relatively consistent throughout the entire time horizon, and the neighbor relationship will also remain constant. When the user's intention changes or the environment changes rapidly, the formation of the drones will be significantly impacted, making the total error difficult to bound. For example, if the original user intention is to form a grid around a target, and then the user decides to form a circle around another target very far away, the error incurred in this process will significantly impact the final regret bound. \Cref{assumption:long_horizon} is to prevent this from happening. 

\begin{theorem}[Horizon-wide \emph{relative} tracking bound under stochastic LLM verification]
\label{thm:horizon_wide_bound_relative}
Under \Cref{assumption:long_horizon}, define $\eta_k:=\|e(t_k^+)\|$ and $\alpha:=e^{-\underline\lambda\Delta}$.
Let $\bar w := \sup_{t}\| (E_k^\top\otimes I_d)d(t)\|$ on each interval; under (B1)--(B2) one may take
$\bar w \le \|E_k^\top\otimes I_d\|\bar d$ to be a constant.
Then:

\smallskip
\noindent\textbf{(i) Uniform bound on edge tracking to the enforced relative reference.}
For all $k\ge 0$,
\begin{equation}
\label{eq:eta_recursion_edge}
\eta_{k+1}
\le
\alpha\,\eta_k
+\frac{\bar w}{\underline\lambda}(1-\alpha)
+\bar J_z.
\end{equation}
Hence
\begin{equation}
\label{eq:eta_closed_form_edge}
\eta_k
\le
\alpha^k\eta_0
+\Bigl(1-\alpha^k\Bigr)\Bigl(\frac{\bar w}{\underline\lambda}+\frac{\bar J_z}{1-\alpha}\Bigr),
\end{equation}
and for any $t\in[t_k,t_{k+1})$,
\begin{equation}
\label{eq:e_interval_bound}
\|e(t)\|
\le
e^{-\underline\lambda(t-t_k)}\eta_k
+\frac{\bar w}{\underline\lambda}\Bigl(1-e^{-\underline\lambda(t-t_k)}\Bigr).
\end{equation}

\smallskip
\noindent\textbf{(ii) Expected ground-truth relative error over a finite horizon.}
For $K\ge 1$ and horizon $T_f=K\Delta_t$,
\begin{equation}
\label{eq:finite_horizon_bound_relative}
\frac{1}{T_f}\int_0^{T_f}\mathbb E\|e_{\mathrm{gt}}(t)\|\,dt
\;\le\;
\underbrace{\frac{\bar w}{\underline\lambda}+\frac{\bar J_z}{\underline\lambda\Delta_t}}_{\text{tracking-to-LLM term}}
\;+\;
\underbrace{\epsilon_{\mathsf C}+(\epsilon_{\mathsf W}-\epsilon_{\mathsf C})\cdot
\frac{1}{K}\sum_{k=0}^{K-1}p_k}_{\text{specification-gap term}}
\;+\;
\frac{\max\{0,\eta_0-\eta_\infty\}}{K\,\underline\lambda\,\Delta_t},
\end{equation}
where $\eta_\infty:=\frac{\bar w}{\underline\lambda}+\frac{\bar J_z}{1-e^{-\underline\lambda\Delta_t}}$ and
$p_k=\mathbb P(S_k=\mathsf W)$.

Moreover, for the two-state chain in this section,
\[
p_k=\pi_{\mathsf W}+(p_0-\pi_{\mathsf W})(1-a-b)^k,
\qquad
\frac{1}{K}\sum_{k=0}^{K-1}p_k
=
\pi_{\mathsf W}+(p_0-\pi_{\mathsf W})\frac{1-(1-a-b)^K}{K(a+b)}.
\]

\smallskip
\noindent\textbf{(iii) Long-run average bound.}
As $T_f\to\infty$,
\begin{equation}
\label{eq:asymptotic_bound_relative}
\limsup_{T_f\to\infty}\frac{1}{T_f}\int_0^{T_f}\mathbb E\|e_{\mathrm{gt}}(t)\|\,dt
\;\le\;
\frac{\bar w}{\underline\lambda}+\frac{\bar J_z}{\underline\lambda\Delta_t}
+\epsilon_{\mathsf C}
+\pi_{\mathsf W}\,(\epsilon_{\mathsf W}-\epsilon_{\mathsf C}).
\end{equation}
\end{theorem}

\begin{proof}
\item
\paragraph{Step 1: Bounding $w(t)$ under a relative hold.}
Under \Cref{assumption:long_horizon} (B1), $\dot z^r(t)=0$ on $[t_k,t_{k+1})$, hence by \eqref{eq:w_def}, $w(t)=-(E_k^\top\otimes I_d)d(t)$ and thus $\|w(t)\|\le \bar w$ on the interval.

\item
\paragraph{Step 2: One-interval ISS bound and the discrete recursion.}
Given \eqref{eq:ISS_uniform} and \Cref{assumption:long_horizon} (B2), we have
\[
\|e(t)\|
\le
e^{-\underline\lambda(t-t_k)}\|e(t_k^+)\|
+\frac{\bar w}{\underline\lambda}\bigl(1-e^{-\underline\lambda(t-t_k)}\bigr),
\]
which is \eqref{eq:e_interval_bound}. Evaluating at $t=t_{k+1}^-$ gives
\[
\|e(t_{k+1}^-)\|
\le
\alpha\,\eta_k + \frac{\bar w}{\underline\lambda}(1-\alpha).
\]
\Cref{assumption:long_horizon} (B2) and \eqref{eq:jump} implies
$e(t_{k+1}^+)=e(t_{k+1}^-)+\Delta z^r_{k+1}$ with $\Delta z^r_{k+1}:=z^r(t_{k+1}^+)-z^r(t_{k+1}^-)$, hence
\[
\eta_{k+1}=\|e(t_{k+1}^+)\|
\le
\|e(t_{k+1}^-)\|+\|\Delta z^r_{k+1}\|
\le
\alpha\,\eta_k+\frac{\bar w}{\underline\lambda}(1-\alpha)+\bar J_z,
\]
which is \eqref{eq:eta_recursion_edge}. Unrolling \eqref{eq:eta_recursion_edge} yields
\eqref{eq:eta_closed_form_edge} exactly as for a scalar affine contraction.

\item
\paragraph{Step 3: Finite-horizon average bound.}
From \eqref{eq:tri_decomp_edge} and Tonelli's theorem,
\[
\frac{1}{T_f}\int_0^{T_f}\mathbb E\|e_{\mathrm{gt}}(t)\|\,dt
\le
\underbrace{\frac{1}{T_f}\int_0^{T_f}\mathbb E\|e(t)\|\,dt}_{(\mathrm I)}
+
\underbrace{\frac{1}{T_f}\int_0^{T_f}\mathbb E\|z^r(t)-z^\star(t)\|\,dt}_{(\mathrm{II})}.
\]

\emph{Term (II).} By \Cref{assumption:long_horizon} (B3), for all $t\in[t_k,t_{k+1})$,
$\|z^r(t)-z^\star(t)\|\le \epsilon_{\mathsf C}\mathbf 1_{\{S_k=\mathsf C\}}+\epsilon_{\mathsf W}\mathbf 1_{\{S_k=\mathsf W\}}$.
Taking expectation gives
$\mathbb E\|z^r(t)-z^\star(t)\|\le \epsilon_{\mathsf C}+(\epsilon_{\mathsf W}-\epsilon_{\mathsf C})p_k$,
and integrating over $[0,T_f]$ yields the specification-gap term in \eqref{eq:finite_horizon_bound_relative}.

\emph{Term (I).} The bound \eqref{eq:e_interval_bound} implies, for $t\in[t_k,t_{k+1})$,
\[
\|e(t)\|\le \frac{\bar w}{\underline\lambda}
+ e^{-\underline\lambda(t-t_k)}\Bigl(\eta_k-\frac{\bar w}{\underline\lambda}\Bigr)_+.
\]
\emph{Term (I).}
Fix $k\in\{0,\dots,K-1\}$ and consider $t\in[t_k,t_{k+1})$.
Starting from \eqref{eq:e_interval_bound} and rearranging,
\begin{align}
\|e(t)\|
&\le e^{-\underline\lambda(t-t_k)}\eta_k
+\frac{\bar w}{\underline\lambda}\Bigl(1-e^{-\underline\lambda(t-t_k)}\Bigr)\nonumber\\
&= \frac{\bar w}{\underline\lambda}
+e^{-\underline\lambda(t-t_k)}\Bigl(\eta_k-\frac{\bar w}{\underline\lambda}\Bigr)\nonumber\\
&\le \frac{\bar w}{\underline\lambda}
+e^{-\underline\lambda(t-t_k)}\Bigl(\eta_k-\frac{\bar w}{\underline\lambda}\Bigr)_+.
\label{eq:e_bound_pospart}
\end{align}
Integrating \eqref{eq:e_bound_pospart} over $[t_k,t_{k+1})$ and using
$\int_{t_k}^{t_{k+1}} e^{-\underline\lambda(t-t_k)}dt=\int_0^\Delta e^{-\underline\lambda s}ds
=\frac{1-e^{-\underline\lambda\Delta_t}}{\underline\lambda}$ yields the \emph{pathwise} bound
\begin{equation}
\label{eq:interval_integral_bound}
\int_{t_k}^{t_{k+1}}\|e(t)\|dt \le
\Delta\frac{\bar w}{\underline\lambda}
+\frac{1-\alpha}{\underline\lambda}\Bigl(\eta_k-\frac{\bar w}{\underline\lambda}\Bigr)_+,
\qquad \alpha:=e^{-\underline\lambda\Delta}.
\end{equation}
Summing \eqref{eq:interval_integral_bound} over $k=0,\dots,K-1$, dividing by $T_f=K\Delta$,
and taking expectations gives
\begin{equation}
\label{eq:I_preavg}
(\mathrm I)
=\frac{1}{T_f}\int_0^{T_f}\mathbb E\|e(t)\|\,dt
\le
\frac{\bar w}{\underline\lambda}
+\frac{1-\alpha}{\underline\lambda\Delta}\cdot
\frac{1}{K}\sum_{k=0}^{K-1}\mathbb E\Bigl[\Bigl(\eta_k-\frac{\bar w}{\underline\lambda}\Bigr)_+\Bigr].
\end{equation}

It remains to upper bound the average of the positive parts.
Define the shifted nonnegative sequence
\[
\xi_k := \Bigl(\eta_k-\frac{\bar w}{\underline\lambda}\Bigr)_+.
\]
From the recursion \eqref{eq:eta_recursion_edge},
\[
\eta_{k+1}
\le
\alpha\eta_k+\frac{\bar w}{\underline\lambda}(1-\alpha)+\bar J_z,
\]
subtracting $\bar w/\underline\lambda$ from both sides gives
\begin{equation}
\label{eq:shifted_eta}
\eta_{k+1}-\frac{\bar w}{\underline\lambda}
\le
\alpha\Bigl(\eta_k-\frac{\bar w}{\underline\lambda}\Bigr)+\bar J_z.
\end{equation}
Taking positive parts on both sides and using $(\alpha x+\bar J_z)_+\le \alpha x_+ +\bar J_z$
for $\alpha\in[0,1]$ and $\bar J_z\ge 0$ yields the \emph{nonnegative affine recursion}
\begin{equation}
\label{eq:xi_recursion}
\xi_{k+1}\le \alpha \xi_k+\bar J_z.
\end{equation}
Unrolling \eqref{eq:xi_recursion} gives, for all $k\ge 0$,
\begin{equation}
\label{eq:xi_closed}
\xi_k
\le
\alpha^k\xi_0+\bar J_z\sum_{i=0}^{k-1}\alpha^i
=
\alpha^k\xi_0+\frac{\bar J_z}{1-\alpha}(1-\alpha^k).
\end{equation}
Averaging \eqref{eq:xi_closed} over $k=0,\dots,K-1$ and using
$\frac{1}{K}\sum_{k=0}^{K-1}\alpha^k=\frac{1-\alpha^K}{K(1-\alpha)}$ yields
\begin{align}
\frac{1}{K}\sum_{k=0}^{K-1}\xi_k
&\le
\frac{\bar J_z}{1-\alpha}
+\Bigl(\xi_0-\frac{\bar J_z}{1-\alpha}\Bigr)\frac{1-\alpha^K}{K(1-\alpha)}\nonumber\\
&\le
\frac{\bar J_z}{1-\alpha}
+\Bigl(\xi_0-\frac{\bar J_z}{1-\alpha}\Bigr)_+\frac{1-\alpha^K}{K(1-\alpha)}
\;\le\;
\frac{\bar J_z}{1-\alpha}
+\frac{1}{K(1-\alpha)}\Bigl(\xi_0-\frac{\bar J_z}{1-\alpha}\Bigr)_+.
\label{eq:xi_avg_bound}
\end{align}
Taking expectations in \eqref{eq:xi_avg_bound} preserves the inequality. Substituting
\eqref{eq:xi_avg_bound} into \eqref{eq:I_preavg} gives
\begin{equation}
\label{eq:I_final}
(\mathrm I)
\le
\frac{\bar w}{\underline\lambda}
+\frac{1-\alpha}{\underline\lambda\Delta_t}
\left(
\frac{\bar J_z}{1-\alpha}
+\frac{1}{K(1-\alpha)}\Bigl(\xi_0-\frac{\bar J_z}{1-\alpha}\Bigr)_+
\right)
=
\frac{\bar w}{\underline\lambda}
+\frac{\bar J_z}{\underline\lambda\Delta_t}
+\frac{1}{K\underline\lambda\Delta_t}\Bigl(\xi_0-\frac{\bar J_z}{1-\alpha}\Bigr)_+.
\end{equation}
Finally, since $\xi_0=(\eta_0-\bar w/\underline\lambda)_+$ and $\bar J_z/(1-\alpha)\ge 0$,
one has the identity
\[
\Bigl(\xi_0-\frac{\bar J_z}{1-\alpha}\Bigr)_+
=
\Bigl(\eta_0-\frac{\bar w}{\underline\lambda}-\frac{\bar J_z}{1-\alpha}\Bigr)_+.
\]
With $\eta_\infty:=\frac{\bar w}{\underline\lambda}+\frac{\bar J_z}{1-\alpha}
=\frac{\bar w}{\underline\lambda}+\frac{\bar J_z}{1-e^{-\underline\lambda\Delta}}$,
this becomes $(\eta_0-\eta_\infty)_+$, and \eqref{eq:I_final} matches the tracking term used in
\eqref{eq:finite_horizon_bound_relative}.

Combining (I) and (II) yields \eqref{eq:finite_horizon_bound_relative}.

Finally, the closed form for $p_k$ and $\frac1K\sum_{k=0}^{K-1}p_k$ is follows from the scalar recursion $p_{k+1}=a+(1-a-b)p_k$.

\item
\paragraph{Step 4: Long-run bound.}
Since $|1-a-b|<1$, we have $\frac1K\sum_{k=0}^{K-1}p_k\to \pi_{\mathsf W}$ and
$\frac{(\eta_0-\eta_\infty)_+}{K\underline\lambda\Delta_t}\to 0$ as $K\to\infty$.
Taking $\limsup$ in \eqref{eq:finite_horizon_bound_relative} yields \eqref{eq:asymptotic_bound_relative}.
\end{proof}

\end{document}